\documentclass{article}

    \usepackage[preprint,nonatbib]{neurips_2019}

\usepackage[utf8]{inputenc}
\usepackage[T1]{fontenc}    
\usepackage{hyperref}       
\usepackage{url}            
\usepackage{booktabs}     
\usepackage{amsfonts}       
\usepackage{nicefrac}       
\usepackage{microtype}      
\usepackage{amsfonts}       
\usepackage{nicefrac}       
\usepackage{microtype}      
\usepackage{times}
\usepackage{adjustbox}
\usepackage{tabularx}
\usepackage{amsmath}
\usepackage{amssymb}
\usepackage{amsthm}
\usepackage{blindtext}
\usepackage{amsthm}
\usepackage{bbm}
\usepackage{blindtext}
\usepackage{enumitem}
\usepackage{multirow}
\usepackage{graphicx}
\usepackage{subfigure}
\usepackage{algorithm}
\usepackage{algpseudocode}%
\theoremstyle{definition}
\newtheorem{assumption}{Assumption}
\newtheorem{theorem}{Theorem}
\newtheorem{definition}{Definition}
\newtheorem{corollary}{Corollary}
\newtheorem{lemma}{Lemma}
\newtheorem{remark}{Remark}

\usepackage{color}
\newcommand{\LD}{L_\mathcal{D} (w^{t+0.75})}
\newcommand{\LT}{L_\mathcal{D} (w^{t})}
\newcommand{\wT}{w^t}
\newcommand{\wD}{w^{t+0.75}}
\newcommand{\hatwD}{\hat{w}^{t+0.75}}

\newcommand{\lge}{\langle}
\newcommand{\rge}{\rangle}
\def\lc{\left\lceil}   
\def\rc{\right\rceil}
\newcommand{\nm}{\nonumber}
\newcommand{\sunion }{{S^{t+1}\bigcup S^t} }
\newcommand{\sdiffa}{ {S^{t+1}\backslash S^t} }
\newcommand{\sdiffb}{{S^{t}\backslash S^{t+1}} }
\title{High Dimensional Differentially Private Stochastic Optimization with Heavy-tailed Data}

\author{%
  Lijie Hu\\
  CEMSE\\ 
KAUST
  \And
  Shuo Ni\thanks{The first two authors contributed equally. The work was done when Shuo Ni was a research intern at KAUST.} \\
  Department of ECE\\ 
USC \\
\And 
Hanshen Xiao\\
CSAIL\\
MIT
  \And 
  Di Wang \\
    CEMSE\\ 
KAUST
}

\begin{document}

     \maketitle
\begin{abstract}
As one of the most fundamental problems in machine learning, statistics and differential privacy, Differentially Private Stochastic Convex Optimization (DP-SCO) has been extensively studied in recent years. However, most of the previous work can only handle either regular data distributions or irregular data in the low dimensional space case. To better understand the challenges arising from irregular data distributions, in this paper we provide the first study on the problem of DP-SCO with heavy-tailed data in the high dimensional space. In the first part we focus on the problem over some polytope constraint (such as the $\ell_1$-norm ball). We show that if the loss function is smooth and its gradient has bounded second order moment, it is possible to get a (high probability) error bound (excess population risk) of $\tilde{O}(\frac{\log d}{(n\epsilon)^\frac{1}{3}})$ in the $\epsilon$-DP model, where $n$ is the sample size and $d$ is the dimension of the underlying space. Next, for LASSO,  if the data  distribution has bounded fourth-order moments, we improve the bound to $\tilde{O}(\frac{\log d}{(n\epsilon)^\frac{2}{5}})$ in the $(\epsilon, \delta)$-DP model.  In the second part of the paper, we study sparse learning with heavy-tailed data. We first revisit the sparse linear model and propose a truncated DP-IHT method whose output could achieve an error of $\tilde{O}(\frac{s^{*2}\log^2 d}{n\epsilon})$, where $s^*$ is the sparsity of the underlying parameter. Then we study a more general problem over the sparsity ({\em i.e.,} $\ell_0$-norm) constraint, and show that it is possible to achieve an error of $\tilde{O}(\frac{s^{*\frac{3}{2}}\log d}{n\epsilon})$, which is also near optimal up to a factor of $\tilde{O}{(\sqrt{s^*})}$,  if the loss function is smooth and strongly convex. Experiments on both of the synthetic and real world data also support our theoretical analysis.
\end{abstract}

\section{Introduction}

Privacy-preservation has become an important consideration and now is a challenging task for machine learning algorithms with sensitive data. To address the privacy issue, Differential Privacy (DP) has received a great deal of attentions and now has established itself as a de facto notation of privacy for data analysis. Methods to guarantee differential privacy have been widely studied, and recently adopted in industry \cite{apple,ding2017collecting}.

Stochastic Convex Optimization (SCO)  \cite{vapnik2013nature}  and its empirical form, Empirical Risk Minimization (ERM), are the most fundamental problems in machine learning and statistics, which include several basic models, such as linear regression and logistic regression. They find numerous applications in many areas such as medicine, finance, genomics  and social science. Due to their importance, the problem of designing DP algorithms for SCO or ERM ({\em i.e.,} DP-SCO and DP-ERM)
have been
extensively studied for nearly a decade 
starting from
\cite{chaudhuri2009privacy,chaudhuri2011differentially}.  Later on, 
a long list of works have attacked the problems
from different perspectives: \cite{bassily2014private, iyengar2019towards,bassily2020,bassily2020stability,feldman2020private,zhou2020bypassing,song2020characterizing} studied the problems in the low dimensional case and the central model, \cite{kasiviswanathan2016efficient,kifer2012private,talwar2015nearly,wang2020knowledge,cai2020cost} considered the problems in the high dimensional sparse case and the central model, \cite{smith2017interaction,duchi2013local,JMLR:v21:19-253,duchi2018minimax} focused on the problems in the local model.

However, most of those previous work can only handle regular data, {\em i.e.,} they need to assume either the underlying data distribution is bounded or sub-Gaussian,  or the loss function is $O(1)$-Lipschitz for all the data. This is particularly true for those output perturbation based \cite{chaudhuri2011differentially} and objective or gradient perturbation based 
\cite{bassily2014private} DP methods.
However, such assumptions may not always hold when dealing with real-world datasets, especially those from biomedicine and finance, which are often heavy-tailed \cite{woolson2011statistical,biswas2007statistical,ibragimov2015heavy}, implying that existing algorithms may fail to guarantee the DP property. Compared with bounded data, 
heavy-tailed data could lead to unbounded gradient and thus violate the Lipschitz condition. 
For example, consider the linear squared loss $\ell(w, (x, y))=(w^Tx-y)^2$. When $x$ is heavy-tailed, the gradient of $\ell(w, (x,y))$ becomes unbounded. To address the  issue,  one potential approach is to truncating or trimming the gradient, such as in \cite{abadi2016deep}. However, there is no existing convergence result based on  their algorithm. Thus, new private and robust estimation methods for heavy-tailed data are needed. 

 Recently, there are several work studied private mean estimation or DP-SCO with heavy-tailed data \cite{wang2020differentially,kamath2020private,liu2021robust,barber2014privacy} (see Section \ref{sec:related} for details). However, the estimation errors of these results all are dependent on polynomial in the dimension of the underlying space, which impedes them to be implemented to the high dimensional setting, where the dimension is far greater than the sample size. In contrast, as we mentioned earlier, high dimensional DP-SCO with regular data has been studied quite well. Thus, our question is, what are the theoretical behaviors of DP-SCO with heavy-tailed data in the high dimensional space? In this paper, we provide a comprehensive and the first study on the problem under different settings by providing several new methods. Our contributions can be summarized as the following, 
\begin{enumerate}
    \item We first study DP-SCO over some polytope constraint, which has been studied in \cite{talwar2015nearly,asi2021private} for regular data.  We first show that if the loss function is smooth and its gradient has bounded second order moment, it is possible to get an excess population risk (error bound)  of $\tilde{O}(\frac{\log d}{(n\epsilon)^\frac{1}{3}})$ with high probability in the $\epsilon$-DP model, where $n$ is the sample size and $d$ is the dimensionality of the underlying space.   Next, for LASSO, if the data  distribution has bounded fourth-order moments, we improve the bound to $\tilde{O}(\frac{\log d}{(n\epsilon)^\frac{2}{5}})$ in the $(\epsilon, \delta)$-DP model. 
    \item  We then study DP-SCO for sparse learning with heavy-tailed data in the $(\epsilon, \delta)$-DP model, which has been studied in \cite{Wang019a,wang2019differentially12,cai2019cost} in the regular data case. We first revisit the sparse linear regression problem  and propose a new method whose output could achieve an error bound of $\tilde{O}(\frac{s^{*2}\log^2 d}{n\epsilon})$, where $s^*$ is the sparsity of the underlying parameter. Then we study a general DP-SCO problem under the sparsity constraint, and show that it is possible to achieve an error of $\tilde{O}(\frac{s^{*\frac{3}{2}}\log d}{n\epsilon})$, if the loss function is smooth and strongly convex. We also show this bound is near optimal up to a factor of $O(\sqrt{s^*}\log^2 n )$. To get these results, we provide several new methods and hard instances which may be used to in other machine learning problems.

\end{enumerate}
Due to space limit, all the proofs and lemmas 
are included in the Appendix. 

\section{Related Work}\label{sec:related}
As mentioned earlier, there is a long list of results on DP-SCO and DP-ERM. However, most of them consider the case where the underlying data distribution is sub-Gaussian and cannot be extended to heavy-tailed case. 
On the other side, in the non-private case, recently a number of works have studied the SCO and ERM problems with heavy-tailed data, such as  
\cite{brownlees2015empirical,minsker2015geometric,hsu2016loss,lecue2018robust,holland2019better,lugosi2019risk,prasad2018robust}.
It is not clear whether they can be adapted to private versions and in the high dimensional setting.

For DP-SCO or private estimation for heavy-tailed distribution,  \cite{barber2014privacy} provides the first study on private mean estimation for distributions with bounded moment and proposes the minimax private rates. Their methods are based on truncating the data to make each data record has a bounded $\ell_2$-norm. 
However, as \cite{kamath2020private} mentioned, they need a stronger assumption on the bounded moment, {\em e.g., for the mean estimation problem they need to assume $\mathbb{E}[\|x\|_2^2]\leq 1$ while we only assume $\mathbb{E}[x_j^2]\leq 1$ for each coordinate $j\in [d]$}. Moreover, their method cannot be extended to the high dimensional sparse setting directly, and  their error bound is in the expectation form, while in the robust statistics it is preferable to get high probability results (see Definition \ref{def:4} for details).  Later, \cite{kamath2020private} also studies the heavy-tailed mean estimation, which is also studied by \cite{liu2021robust} recently. However, their results for general $d$ dimensional space are still not the high probability form (they can only show their results hold with probability at least 0.7). Thus, their methods cannot be used to DP-SCO directly. 
Moreover, it is unknown whether their methods could be extended to the high dimensional or the sparse setting. \cite{brunel2020propose} recently also studies the same problem and proposes a method based on the PTR mechanism \cite{dwork2009differential}. However, their method can be only used in the $1$-dimensional space and and needs stronger assumptions. 

Meanwhile, instead of the mean estimation,  \cite{wang2020differentially} provides the first study on DP-SCO with heavy-tailed data and proposes three methods based on different assumptions. Their first method is based on the Sample-and-Aggregate framework \cite{nissim2007smooth}. However, this method needs enormous assumptions and its error bound is quite large. Their second method is still based on the smooth sensitivity \cite{bun2019average}. However, \cite{wang2020differentially} needs to assume the distribution is sub-exponential. It also provides a new private estimator motivated by the previous work in robust statistics. While some our estimators are quite similar as theirs, they are quite a lot differences (see Remark \ref{remark:1} for details).
Based on the mean estimator in \cite{kamath2020private}, \cite{kamath2021improved} recently studies DP-SCO and improves the (expected) excess population risk to $\tilde{O}((\frac{d}{\epsilon n})^\frac{1}{2}) $ and $\tilde{O}(\frac{d}{\epsilon n}) $ for convex and strongly convex loss functions respectively under the assumption that the gradient of the loss has bounded second order moment. These results match the best known result of the heavy-tailed mean estimation problem. However, all of these results are in the expectation form instead of the high probability form. Moreover, their method cannot be extended to the linear model, where the bounded second order moment of loss assumption is quite strong (see Assumption \ref{ass:3} for details). We note that all  these methods cannot be directly extended to the high dimensional case or the sparse learning problem. \footnote{We refer readers the reference \cite{kamath2020private,wang2020differentially} to see more related work on DP methods for unbounded sensitivity.}

\section{Preliminaries}\label{secprelin} 

\paragraph{Notations:} For vectors $v, v_i\in \mathbb{R}^d$, we denote $v_j$ and $v_{i, j}$ as their corresponding the $j$-th coordinate. Given a set of indices $S \subseteq [d]$, we denote the vector $v_S\in \mathbb{R}^d$ as the projection of $v$ onto $S$, {\em i.e., } $v_{S,j}=v_j$ if $j\in S$, and $v_{S, j}=0$ otherwise.  We also denote $|S|$ as the number of elements in $S$ and $\text{supp}(w)=\{j\in [d]: w_j\neq 0\}\subseteq [d]$ for $w$. For a constraint set $\mathcal{W}$, we denote $\|\mathcal{W}\|_1$ as it is $\ell_1$-norm diameter, {\em i.e.,} $\|\mathcal{W}\|_1=\max_{u, v \in \mathcal{W}}\|u-v\|_1$. 
\begin{definition}[Differential Privacy \cite{dwork2006calibrating}]\label{def:1}
	Given a data universe $\mathcal{X}$, we say that two datasets $D,D'\subseteq \mathcal{X}$ are neighbors if they differ by only one data sample, which is denoted as $D \sim D'$. A randomized algorithm $\mathcal{A}$ is $(\epsilon,\delta)$-differentially private (DP) if for all neighboring datasets $D,D'$ and for all events $S$ in the output space of $\mathcal{A}$, we have $$\text{Pr}(\mathcal{A}(D)\in S)\leq e^{\epsilon} \text{Pr}(\mathcal{A}(D')\in S)+\delta.$$
\end{definition}
In this paper, we will mainly use the Laplacian and the Exponential mechanism, and the Advanced Composition Theorem to guarantee DP property. 
\begin{definition}[Laplacian Mechanism]\label{def:2}
		Given a function $q : \mathcal{X}^n\rightarrow \mathbb{R}^d$, the Laplacian Mechanism is defined as:
		$\mathcal{M}_L(D,q,\epsilon)=q(D)+ (Y_1, Y_2, \cdots, Y_d),$
		where $Y_i$ is i.i.d. drawn from a Laplacian Distribution $\text{Lap}(\frac{\Delta_1(q)}{\epsilon}),$ where $\Delta_1(q)$ is the $\ell_1$-sensitivity of the function $q$, {\em i.e.,}
		$\Delta_1(q)=\sup_{D\sim D'}||q(D)-q(D')||_1.$ For a parameter $\lambda$, the Laplacian distribution has the density function $\text{Lap}(\lambda) (x)=\frac{1}{2\lambda}\exp(-\frac{x}{\lambda})$. 
		Laplacian Mechanism preserves $\epsilon$-DP.
\end{definition}
	\begin{definition}[Exponential Mechanism]\label{def:4}
		The Exponential Mechanism allows differentially private computation over arbitrary domains and range $\mathcal{R}$, parametrized by a score function $u(D,r)$ which maps a pair of input data set $D$ and candidate result $r\in \mathcal{R}$ to a real valued score. With the score function $u$ and privacy budget $\epsilon$, the mechanism yields an output with exponential bias in favor of high scoring outputs. Let $\mathcal{M}(D, u, \mathcal{R})$ denote the exponential mechanism, and $\Delta$ be the sensitivity of $u$ in the range $\mathcal{R}$, {\em i.e.,}
		$\Delta=\max_{r\in \mathcal{R}}\max_{D\sim D'}|u(D,r)-u(D',r)|.$
		Then if $\mathcal{M}(D, u, R)$ selects and outputs an element $r\in \mathcal{R}$ with probability proportional to $\exp(\frac{\epsilon u(D,r)}{2\Delta u})$, it  preserves $\epsilon$-DP.
	\end{definition}
	The output of exponential mechanism has the following utility. 
		\begin{lemma}[\cite{dwork2014algorithmic}]\label{lemma:exp}
		For the exponential mechanism $\mathcal{M}(D,u, \mathcal{R})$, we have 
		\begin{equation*}
		\text{Pr}\{u(\mathcal{M}(D, u, \mathcal{R}))\leq \text{OPT}_u(x)-\frac{2\Delta u}{\epsilon}(\ln|\mathcal{R}|+t)\}\leq e^{-t}.
		\end{equation*}
		where $\text{OPT}_u(x)$ is the highest score in the range $\mathcal{R}$, {\em i.e.} $\max_{r\in \mathcal{R}}u(D,r)$.
	\end{lemma}
	\begin{lemma}[Advanced Composition Theorem]\label{lemma:adv}
Given target privacy parameters $0<\epsilon <1$ and $0<\delta<1$, to ensure $(\epsilon, T\delta'+\delta)$-DP over $T$ mechanisms, it suffices that each mechanism is $(\epsilon',\delta')$-DP, where $\epsilon'=\frac{\epsilon}{2\sqrt{2T\ln(2/\delta)}}$ and $\delta'=\frac{\delta}{T}$.  
\end{lemma}
	\begin{definition}[DP-SCO \cite{bassily2014private}]\label{def:5}
		Given a dataset $D=\{z_1,\cdots,z_n\}$ from a data universe $\mathcal{Z}$ where $z_i=(x_i, y_i)$ with a feature vector $x_i$ and a label/response $y_i$ are i.i.d. samples from some unknown distribution $\mathcal{D}$,  a convex constraint set  $\mathcal{W} \subseteq \mathbb{R}^d$, and a convex loss function $\ell: \mathcal{W}\times \mathcal{Z}\mapsto \mathbb{R}$. Differentially Private Stochastic Convex Optimization (DP-SCO) is to find $w^{\text{priv}}$ so as to minimize the population risk, {\em i.e.,} $L_\mathcal{D} (w)=\mathbb{E}_{z\sim \mathcal{D}}[\ell(w, z)]$
		with the guarantee of being differentially private.\footnote{Note that in this paper we consider the improper learning case, that is $w^{\text{priv}}$ may not in $\mathcal{W}$.} 
		 The utility of the algorithm is measured by the excess population risk, that is  $$L_\mathcal{D} (w^{\text{priv}})-\min_{w\in \mathbb{\mathcal{W}}}L_\mathcal{D} (w).$$ 	 Besides the population risk, we can also measure the \textit{empirical risk} of dataset $D$: $\hat{L}(w, D)=\frac{1}{n}\sum_{i=1}^n \ell(w, z_i).$ It is notable that in the \textbf{high probability setting}, we need to get a high probability excess population risk.  That is given a failure probability $0<\zeta<1$, we want get a (polynomial) function $f(d, \log\frac{1}{\delta}, \log\frac{1}{\zeta},  \frac{1}{n}, \frac{1}{\epsilon})$ such that with probability at least $1-\zeta$ (over the randomness of the algorithm and the data distribution), $$L_\mathcal{D} (w^{\text{priv}})-\min_{w\in \mathbb{\mathcal{W}}}L_\mathcal{D} (w)\leq O(f(d, \log\frac{1}{\delta}, \log\frac{1}{\zeta}, \frac{1}{n}, \frac{1}{\epsilon})).$$ 
		 Compared with the high probability setting, there is another setting namely the expectation setting where our goal is to get a (polynomial) function $f(d, \log\frac{1}{\delta},  \frac{1}{n}, \frac{1}{\epsilon})$ such that $$\mathbb{E} L_\mathcal{D} (w^{\text{priv}})-\min_{w\in \mathbb{\mathcal{W}}}L_\mathcal{D} (w)\leq O(f(d, \log\frac{1}{\delta},  \frac{1}{n}, \frac{1}{\epsilon})),$$ where the expectation takes over the randomness of the data records and the algorithm. 
		 
	\end{definition}
	It is notable that, in the regular data case where the data distribution $\mathcal{D}$ or the gradient of the loss is bounded or sub-Gaussian, it is easy to transform an expected excess population risk to an excess population risk with high probability. However, this is not true for the heavy-tailed case. \footnote{See \cite{catoni2012challenging} for the necessity to consider the high probability setting.} Thus, all of the recent studies on robust statistics such as \cite{brownlees2015empirical,minsker2015geometric,hsu2016loss,lecue2018robust,holland2019better,lugosi2019risk,prasad2018robust} focused on the high probability setting. In the paper, we will study the problem in the high probability setting. Moreover, throughout the paper we focus on the high dimensional case where $d$ could be far greater than $n$. Thus we wish the error bounds (excess population risk) be logarithmic of $d$. 
	
	The following two definitions on loss functions are commonly used in machine learning, optimization and statistics. 
	\begin{definition}\label{def:6}
	    A function $f$ is $L$-Lipschitz w.r.t the norm $\|\cdot\|$ if  for all $w, w'\in\mathcal{W}$, $|f(w)-f(w')|\leq L\|w-w'\|$.
	\end{definition}
	\begin{definition}\label{def:7}
	    	    A function $f$ is $\alpha$-smooth on $\mathcal{W}$ if for all $w, w'\in \mathcal{W}$, $f(w')\leq f(w)+\langle \nabla f(w), w'-w \rangle+\frac{\alpha}{2}\|w'-w\|_2^2.$
	\end{definition}

\section{High Dimensional DP-SCO over Polytope Domain}
In this section we will study DP-SCO over polytope domain, {\em i.e.,} the underlying constraint set $\mathcal{W}$ is some polytope and thus could be written as the convex hull of a finite set $V$. This contains numerous of  learning models that address high dimensional data, such as LASSO and minimization over  probability simplex.

\cite{talwar2015nearly} first studied the problem of DP-ERM over polytope domain in the regular data setting ({\em i.e.,} the gradient of loss function has bounded norm).  Specifically, they showed that when the loss function is Lipschitz w.r.t $\ell_1$-norm, there is an $(\epsilon, \delta)$-DP algorithm (DP Frank-Wolfe) whose  output  could achieve an error of $O(\frac{\log (|V|n)}{(n\epsilon)^\frac{2}{3}})$. However, to generalize to the heavy-tailed data setting, the main difficulty  is that  the assumption of $\ell_1$-norm Lipschitz does not hold anymore. 
To address the problem, one possible approach may be truncating the gradient to make it has bounded $\ell_\infty$-norm (since $\ell_1$-norm Lipschitz is equivalent to its gradient has bounded $\ell_\infty$-norm). However, as mentioned in \cite{wang2020differentially}, it could introduce enormous amount of error and it is difficult to select the best threshold parameter. In the following we will propose a new method to overcome this challenge.  We will focus on the case where the gradient of the loss is heavy-tailed. Specifically, following from the previous work on robust statistics such as \cite{prasad2018robust,holland2019better}, here we propose the following assumption on the gradient of the loss function. 

\begin{assumption}\label{ass:1}
We assume $L_D(\cdot)$ is $\alpha$-smooth,  and there exists a $\tau>0$ such that for any $w\in \mathcal{W}$ and each coordinate $j\in [d]$, we have $\mathbb{E}[(\nabla_j \ell(w, x))^2]\leq \tau$. 
\end{assumption}
First, it is notable that the smoothness condition in Assumption \ref{ass:1} is necessary for the high dimensional setting. As shown by \cite{asi2021private},  when the loss function is non-smooth and  $\ell_1$-norm Lipschitz, even in the regular data setting the excess population risk is lower bounded by $\Omega(\sqrt{\frac{\log d}{n}}+\frac{\sqrt{d}}{n\epsilon})$, which depends on $\Omega(\sqrt{d})$. Secondly, in some other work on studying private estimation for distributions with bounded second-order moment (such as \cite{kamath2020private}), they assume that for each unit vector $u\in \mathbb{R}^d$, $\mathbb{E}[\langle u, \nabla \ell(w, x)\rangle^2]\leq \tau=O(1)$. Thus, our assumption on the moment is reasonable. Thirdly, we note that $\tau$ may be not a constant, it could depend on the structure of the loss function, data distribution and the underlying structure of $\mathcal{W}$ \cite{DBLP:journals/corr/abs-2010-13520}.  Throughout the whole paper we assume $\tau$ is known, which is commonly used in other related work in robust statistics such as  \cite{kamath2020private,bubeck2013bandits}. 

Our approach, namely Heavy-tailed DP-FW, could be seen as a generalization of the DP Frank-Wolfe method in \cite{talwar2015nearly}. The approach is motivated by a robust mean estimator for heavy-tailed distribution given by \cite{catoni2017dimension} which was extended by \cite{holland2019better}. For simplicity, we first consider a $1$-dimensional random variable $x$ and assume that $x_1, x_2, \cdots, x_n$ are i.i.d. sampled from $x$. The robust mean estimator consists of three steps:

\noindent \textbf{Scaling and Truncation} For each sample $x_i$, we first re-scale it by dividing $s$ (which will be specified later). Then, the re-scaled one was passed through a  soft truncation function $\phi$. Finally, we put the truncated mean back to the original scale. That is, 
\begin{equation}\label{eq:1}
    \frac{s}{n}\sum_{i=1}^n \phi(\frac{x_i}{s})\approx \mathbb{E}x.
\end{equation}
Here, we use the function given in \cite{catoni2017dimension},
\begin{equation}\label{eq:2}
    \phi(x)= \begin{cases} x-\frac{x^3}{6}, & -\sqrt{2}\leq x\leq \sqrt{2} \\
    \frac{2\sqrt{2}}{3}, & x>\sqrt{2} \\
    -\frac{2\sqrt{2}}{3}, & x<-\sqrt{2}.   
    \end{cases}
\end{equation}
A key property for $\phi$ is that $\phi$ is bounded, that is,  $|\phi(x)|\leq \frac{2\sqrt{2}}{3}$.

\noindent \textbf{Noise Multiplication} Let $\eta_1, \eta_2, \cdots, \eta_n$ be random noise generated from a common distribution $\eta\sim \chi$ with $\mathbb{E}\eta =0$. We multiply each data $x_i$ by a factor of $1+\eta_i$, and then perform the scaling and truncation step on the term $x_i(1+\eta_i)$.  
That is, 
\begin{equation}\label{eq:3}
    \tilde{x}(\eta) =\frac{s}{n}\sum_{i=1}^n \phi(\frac{x_i+\eta_i x_i}{s}).
\end{equation}
\noindent \textbf{Noise Smoothing} In this final step, we smooth the multiplicative noise by taking the expectation w.r.t. the distributions. In total the robust mean estimator $\hat{x}(s, \beta)$ could be written as,
\begin{equation}\label{eq:4}
    \hat{x}(s, \beta)=\mathbb{E}  \tilde{x}(\eta, s, \beta) = \frac{s}{n}\sum_{i=1}^n \int \phi(\frac{x_i+\eta_i x_i}{s})d \chi(\eta_i).
\end{equation}
Computing the explicit form of each integral in (\ref{eq:4}) depends on the function $\phi(\cdot)$ and the distribution $\chi$. Fortunately, \cite{catoni2017dimension} showed that when $\phi$ is in (\ref{eq:2}) and $\chi\sim \mathcal{N}(0, \frac{1}{\beta})$ (where $\beta$ will be specified later), we have for any $a$ and  $b>0$
\begin{equation}\label{eq:5}
    \mathbb{E}_{\eta}\phi(a+b\sqrt{\beta}\eta)=a(1-\frac{b^2}{2})-\frac{a^3}{6}+\hat{C}(a,b),
\end{equation}
where $\hat{C}(a, b)$ is a correction form which is easy to implement and its explicit form will be given in Appendix.

The key idea of our method is that, by the definition of $ \hat{x}(s, \beta)$ in (\ref{eq:4}) and the function $\phi$ is in (\ref{eq:2}), we can see that the value of $\hat{x}(s, \beta)$ will be changed at most $\frac{4\sqrt{2}s}{3n}$ if we change one sample in the data, {\em i.e.,} the sensitivity of $\hat{x}(s, \beta)$ is bounded by $\frac{4\sqrt{2}s}{3n}$. That is, given a fixed vector $w$ and $n$ gradients $\{\nabla \ell(w, z_i)\}_{i=1}^n$, we can use the above estimator to the entrywise of these gradients to get an estimator (we denote it as $\tilde{g}(w, D)$) of $\mathbb{E}[\ell(w, z)]$. Moreover, we can see the $\ell_\infty$-norm sensitivity of $\tilde{g}(w,  D)$ is bounded $\frac{4\sqrt{2}s}{3n}$, {\em i.e.,} $\|\tilde{g}(w, D)-\tilde{g}(w, D')\|_\infty\leq \frac{4\sqrt{2}s}{3n}$, where $D$ and $D'$ are neighboring datasets. Combing this result with DP Frank Wolfe method, we propose our algorithm. See Algorithm \ref{alg:1} for details. 

\begin{algorithm}
	\caption{Heavy-tailed DP-FW \label{alg:1}}
	\begin{algorithmic}[1]
		\State {\bfseries Input:} $n$-size dataset $D$, loss function $\ell(\cdot, \cdot)$, initial parameter $w^0$, parameters $s, T, \beta, \{\eta_t\}_{t}$ (will be specified later), privacy parameter $\epsilon$, failure probability $\zeta$.  $\mathcal{W}$ is the convex hull of a finite set $V$.
        \State Split the data $D$ into $T$ parts $\{D_{t}\}_{t=1}^{T}$ with $|D_t|=m=\frac{n}{T}$.
        \For{$t=1, \cdots, T$}
           \State For each $j\in [d]$, calculate the robust gradient by (\ref{eq:2})-(\ref{eq:5}), that is 
    \begin{align*}
        &g_j^{t-1}(w^{t-1}, D_t)\\
        &= \frac{1}{m}\sum_{x\in D_{t}} \left(\nabla_j \ell(w^{t-1}, x)\big(1-\frac{\nabla^2_j \ell(w^{t-1}, x)}{2s^2\beta}\big)- \frac{\nabla^3_j \ell(w^{t-1}, x)}{6s^2}\right)\\&+\frac{s}{m}\sum_{x\in D_{t}}\hat{C}\left(\frac{\nabla_j \ell(w^{t-1}, x)}{s}, \frac{|\nabla_j \ell(w^{t-1}, x)|}{s\sqrt{\beta}}\right).
    \end{align*}
      \State Let vector $\tilde{g}(w^{t-1}, D_t)\in \mathbb{R}^d$ as 
     $\tilde{g}(w^{t-1}, D_t)=(g_1^{t-1}(w^{t-1}, D_t), g_2^{t-1}(w^{t-1}, D_t), \cdots, g_d^{t-1}(w^{t-1}, D_t))$. 
        \State Denote the score function $u(D_{t}, \cdot): V\mapsto \mathbb{R}$ such that for each $v\in V$ let $u(D_{t}, v)= -\langle v, \tilde{g}(w^{t-1}, D_t)\rangle$. Run the exponential mechanism with the range $R=V$, sensitivity $\Delta=\frac{4\|\mathcal{W}\|_1\sqrt{2}s}{3m}$ and  the privacy budget $\epsilon$.  Denote the output as $\tilde{w}^{t-1}\in \mathcal{W}$. 
        \State Let $w^t=(1-\eta_{t-1}) w^{t-1}+\eta_{t-1}\tilde{w}^{t-1}.$
        \EndFor \\
        \Return $w^{T}$.
	\end{algorithmic}
\end{algorithm}
\begin{theorem}\label{thm:1}
For any $\epsilon>0$, Algorithm \ref{alg:1} is $\epsilon$-DP. 
\end{theorem}

\begin{theorem}\label{thm:2}
Under Assumption \ref{ass:1} and if $\mathcal{W}$ is a  convex hull of a finite compact set $V$. Then for any given probability of failure $0<\zeta<1$, if we set $T=\tilde{O}\big( (\frac{n\epsilon \alpha^2}{\tau\log \frac{|V|d}{\zeta}})^\frac{1}{3}  \big)$, $\beta=O(1), s=O(\sqrt{\frac{n\epsilon\tau }{T \log \frac{|V|dT}{\zeta}}})$ and $\eta_{t-1}=\frac{2}{t+2}$  in Algorithm \ref{alg:1}, with probability at least $1-\zeta$, 
 \begin{equation}\label{eq:7}
          L_\mathcal{D}(w^T)-\min_{w\in \mathcal{W}} L_\mathcal{D}(w)\leq O \big( \frac{\|\mathcal{W}\|_1 (\alpha\tau\log \frac{n|V|d}{\zeta} )^\frac{1}{3} }{(n\epsilon)^\frac{1}{3}}\big).
   \end{equation}
\end{theorem}
\begin{remark}\label{remark:1}
From Theorem \ref{thm:2} we can see that when $|V|=\text{poly}(d)$ and $\tau=O(1)$, the excess population risk will be upper bounded by $\tilde{O}(\frac{1}{(n\epsilon)^\frac{1}{3}})$. Compared with the previous results in private heavy-tailed estimation \cite{kamath2020private,brunel2020propose,wang2020differentially}, we improve the error bound from $O(d)$ to $O(\log d)$. It is also notable that \cite{wang2020differentially} also used a similar robust estimator as ours. However, there are several differences: First, \cite{wang2020differentially} first performs the robust estimator to each coordinate of the gradients and then add Gaussian noise to the whole vector to ensure DP. Thus, all the errors in   \cite{wang2020differentially} depend on $\text{poly}(d)$ and their method cannot be extended to high dimensional space directly. Secondly, \cite{wang2020differentially} sets $s=O(\sqrt{n})$ while our $s$ depends on both $n, \epsilon$ and $T$. We provide a finer analysis on the  trade-off between the bias and variance of the robust estimator,  and the noise we added in each iteration (see the proof of Theorem \ref{thm:2} for details). Thus, our error is much lower than theirs and our method could be used in \cite{wang2020differentially} and improve their bounds. 
\end{remark}
\begin{corollary}\label{cor:1}
Consider the LASSO problem where $L_D(w)=\mathbb{E}(\langle x, w\rangle -y)^2$ and $\mathcal{W}=\{w\in \mathbb{R}^d: \|w\|_1\leq 1\}$. We know that  the population risk function is $\lambda_{\max}(\mathbb{E}(xx^T))$-smooth, where $\lambda_{\max}(M)$ is the maximal eigenvalue of the matrix $M$. If we further assume each coordinate of the gradient has bound second moment {\em i.e.,} for each $w\in \mathcal{W}$ and $j\in [d] $, $\mathbb{E}[(x_j(\langle x, w\rangle -y))^2]\leq O(1)$ (for example $x_j$ and $y$ are $O(1)$-sub-Gaussian). Then the output of Algorithm \ref{alg:1} satisfies the following with probability at least $1-\zeta$:
 \begin{equation}\label{eq:8}
          L_\mathcal{D}(w^T)-\min_{w\in \mathcal{W}} L_\mathcal{D}(w)\leq O \big( \frac{ (\lambda_{\max}(\mathbb{E}(xx^T)\log \frac{d}{\zeta}\log n )^\frac{1}{3} }{(n\epsilon)^\frac{1}{3}}\big).
   \end{equation}
\end{corollary}
In the previous theorem, we need to assume the loss function is convex. However, we can also show that Algorithm \ref{alg:1} could be used to some specific non-convex loss functions. Below we will study the Robust Regression and provide an upper bound of $\tilde{O}(\frac{1}{(n\epsilon)^\frac{1}{4}})$.  For $\mathcal{W}=\{w\in\mathbb{R}^d| \|w\|_1\leq 1\}$, 
 and  a non-convex positive loss function $\psi$, the loss of robust regression  is defined as $\ell(w, (x,y))=\psi(\langle x, w \rangle-y)$. We make the following assumptions on $\psi$, which includes the biweight loss function \footnote{For a fixed parameter $c>0$, the biweight loss is defined as $\psi(s)=\frac{c^2}{6}\cdot\begin{cases}
1-(1-(\frac{s}{c})^2)^3, |t|\leq c\\
1, |t|\geq c.
\end{cases}$ }
\cite{loh2013regularized}.
\begin{assumption}\label{ass:2}
We assume that 
\begin{enumerate}
    \item There is a constant $ C_\psi \geq 1$, s.t. $\max\{\psi'(s), \psi''(s)\}\leq C_\psi=O(1), \text{ for }$ for all $s$.
    \item $\psi'(\cdot)$ is odd with $\psi'(s)>0, \text{ for } \forall s>0$; and $h(s):=\mathbb{E}_\xi[\psi'(s+\xi)] $ satisfies $h'(0)>c_\psi$, where $c_\psi=O(1)>0$.
    \item There is $w^*\in \mathcal{W}$  such that $y=\langle w^*, x\rangle +\xi$, where $\xi$ is symmetric noise with a zero-mean given $x$. Also we assume that for each coordinate $j\in [d]$, $x_j$ has bounded second order moment, that is $\mathbb{E} x_j^2 \leq O(1)$. 
\end{enumerate}
\end{assumption}
\begin{theorem}\label{thm:3}
Under Assumption \ref{ass:2}, for any given probability of failure $0<\zeta<1$, if we set 	$\beta=O(1)$, $s=O(\frac{\sqrt{n\epsilon}}{ \sqrt{T\log \frac{dT}{\zeta}}})$, $\eta=\frac{1}{\sqrt{T}}$, and $T=\tilde{O}(\sqrt{\frac{n\epsilon}{\log \frac{d}{\zeta}}})$ in Algorithm \ref{alg:1}. Then with probability at least $1-\zeta$ (we omit the $C_\psi$ and $c_\psi$ term), 
     \begin{equation}\label{eq:9}
          L_\mathcal{D}(w^T)-\min_{w\in \mathcal{W}} L_\mathcal{D}(w)\leq O \big( \frac{\lambda_{\max}(\mathbb{E}(xx^T))\log^\frac{1}{4}\frac{dn}{\zeta} }{(n\epsilon)^\frac{1}{4}}.  
          \big).
   \end{equation}

\end{theorem}

For LASSO, there are enormous differences between our results and the results in \cite{talwar2015nearly}. First, \cite{talwar2015nearly} needs to assume that each $|x_{ij}|\leq O(1)$ and $|y_i|\leq O(1)$ to guarantee the loss function be $\ell_1$-norm Lipschitz, while here we just need a bounded second order moment condition.  Secondly,  \cite{talwar2015nearly} only considers the empirical risk  function while here we consider the population risk. It is notable that their method cannot be extended to population risk directly based on their theoretical analysis. Thus, our result of $\tilde{O}(\frac{1}{(n\epsilon)^\frac{1}{3}})$ cannot be compared with theirs directly. Recently \cite{asi2021private} considers  DP-SCO with $\ell_1$-norm Lipschitz loss functions and it provides an upper bound of $\tilde{O}(\frac{1}{  \sqrt{\epsilon n}})$ and $\tilde{O}(\sqrt{\frac{1}{n}}+\frac{1}{(n\epsilon)^\frac{2}{3}})$ for $\epsilon$ and $(\epsilon, \delta)$-DP model respectively. Compared with this, we can see, due to the heavy-tailed distribution, the upper bound now decreases to $\tilde{O}(\frac{1}{(n\epsilon)^\frac{1}{3}})$ for $\epsilon$-DP. 
Thirdly, the DP Frank Wolfe algorithm given by \cite{talwar2015nearly} could guarantee both $\epsilon$ and $(\epsilon, \delta)$-DP with error upper bounds   of $\tilde{O}({\frac{1}{\sqrt{n\epsilon}}})$ and $\tilde{O}(\frac{1}{(n\epsilon)^\frac{2}{3}})$ respectively.\footnote{We can adopt the idea in  \cite{talwar2015nearly} and get the result for $\epsilon$-DP} However, our method can only guarantee $\epsilon$-DP and cannot get improved bounds in the $(\epsilon, \delta)$-DP model.  The mainly reason is that, \cite{talwar2015nearly}  performs the exponential mechanism on the whole data to achieve $O(\frac{\epsilon}{\sqrt{T\log\frac{1}{\delta}}})$-DP in each iteration,  then the whole algorithm will be $(\epsilon, \delta)$-DP due to the advanced composition theorem. However here we cannot adopt this technique directly. The major difficulty is that if we use whole dataset in each iteration then $w^{t-1}$ will depend on the whole dataset. And this cause us in the proof to analyze an upper bound of $\sup_{v\in V}\sup_{w\in \mathcal{W}}\langle v, \tilde{g}(w, D)-\mathbb{E}[\nabla \ell(w; z)]\rangle $, which is difficult to analyze due to the complex form of our estimator $\tilde{g}(w, D)$ in step 5. Thus, we need to get avoid of the dependency. Our strategy is splitting the whole dataset into several parts and in each iteration we use the exponential mechanism on one subset. 
That is why here we  only consider the  $\epsilon$-DP model. 
It is an open problem that whether we can get an improved $(\epsilon, \delta)$-DP method in general. Below we will show that for LASSO it is possible to improve the upper bound from $\tilde{O}(\frac{1}{(n\epsilon)^\frac{1}{3}})$ in Corollary \ref{cor:1} to $\tilde{O}(\frac{1}{(n\epsilon)^\frac{2}{5}})$ in $(\epsilon, \delta)$-DP model if the data distribution has bounded fourth-order moments.

The algorithm consists of two parts. In the first part, motivated by \cite{fan2016shrinkage}, 
we shrunk each entry of each sample by a threshold $K$, which will be determined later. That is, for each $i\in [n]$ and $j\in [d]$, we let $\tilde{x}_{i, j}=\text{sign}(x_{i,j})\min\{|x_{i,j}|, K\}$ and $\tilde{y}_i=\text{sign}(y_i)\min\{|y_i|, K\}$.
Note that since now each entry is bounded, the loss function will be $\ell_1$-norm Lipschitz with $O(K^2)$. Thus, in the second part,  we perform the DP-FW in \cite{talwar2015nearly} on the shrunken data. See Algorithm \ref{alg:2} for details.

\begin{algorithm}
	\caption{Heavy-tailed Private LASSO \label{alg:2}}
	\begin{algorithmic}[1]
		\State {\bfseries Input:} $n$-size dataset $D=\{(x_i, y_i)\}_{i=1}^n$, loss function $\ell(w, (x,y))=(\langle w, x \rangle-y)^2$, initial parameter $w^0$, parameters $K, T, \{\eta_t\}$ (will be specified later), privacy parameter $\epsilon, \delta$, failure probability $\zeta$.  $\mathcal{W}$ is the $\ell_1$-norm ball with set of vertices $V$. 
\State For each $i\in [n]$, we denote a truncated sample $\tilde{x}_i\in \mathbb{R}^d$ where 
for $j\in [d]$ $\tilde{x}_{i, j}=\text{sign}(x_{i,j})\min\{|x_{i,j}|, K\}$, and $\tilde{y}_i=\text{sign}(y_i)\min\{|y_i|, K\}$. Denote the truncated dataset as $\tilde{D}=\{(\tilde{x}_i, \tilde{y}_i)\}_{i=1}^n$.
        \For{$t=1, \cdots, T$}. 
 
        \State Denote the score function $u(\tilde{D}, \cdot): V\mapsto \mathbb{R}$ such that for each $v\in V$ let $u(\tilde{D}, v)= -\langle v, \tilde{g}(w^{t-1}, \tilde{D})\rangle$, where 
     $\tilde{g}(w^{t-1}, \tilde{D})= \frac{2}{n}\sum_{i=1}^n \tilde{x}_i(\langle \tilde{x}_i,  w^{t-1} \rangle-\tilde{y}_i).$
        Run the exponential mechanism with the range $R=V$, sensitivity $\Delta=\frac{8\|\mathcal{W}\|_1 K^2}{n}$ and  the privacy budget $\frac{\epsilon}{2\sqrt{2T\log \frac{1}{\delta}}}$.  Denote the output as $\tilde{w}^{t-1}\in V$. 
        \State Let $w^t=(1-\eta_{t-1}) w^{t-1}+\eta_{t-1}\tilde{w}^{t-1}.$
        \EndFor \\
        \Return $w^{T}$.
	\end{algorithmic}
\end{algorithm}
\begin{theorem}\label{thm:4}
For any $0<\epsilon, \delta<1$, Algorithm \ref{alg:1} is $(\epsilon, \delta)$-DP. 
\end{theorem}
\begin{assumption}\label{ass:3}
We assume that $x$ and $y$ have  bounded forth order moment, {\em i.e.,} for each $j_1, j_2\in [d]$, $\mathbb{E}(x_{j_1}x_{j_2})^2\leq M$, and $\mathbb{E}[y^4]\leq M$, where $M=O(1)$ is a constant.

\end{assumption}
\begin{remark}
We note that Assumption \ref{ass:1} implies $\mathbb{E}(x_{j}x_{k})^2\leq O(\tau)$ for any $j, k\in [d]$.  Since in Assumption \ref{ass:1} we can get $\mathbb{E}[(x_j y)^2]\leq \tau$ if we take $w=0$. And we have $\mathbb{E}[(x_j(x_k-y))^2]\leq \tau$ when we take $w=e_k$ (the $k$-th basis vector), thus $\mathbb{E}[(x_j x_k)^2]\leq O(\tau)$. From this view, Assumption \ref{ass:3} is weaker than Assumption \ref{ass:1}. Moreover, in Assumption \ref{ass:1} we need to assume that the term $\mathbb{E}[x_i^2 (\langle w, x\rangle-y)^2]$ is bounded for each $\|w\|_1\leq 1$, which  is hard to be verified and is unnatural for the linear model compared with the previous work on linear regression with heavy-tailed data \cite{zhang2018ell_1,hsu2016loss}. 
\end{remark}

\begin{theorem}\label{thm:5}
Under Assumption \ref{ass:3}, for any given probability of failure $0<\zeta<1$, if we set $K=\frac{{(n\epsilon)^{\frac{1}{4}}}}{T^\frac{1}{8}}$,  $T=\tilde{O}((\frac{\sqrt{n\epsilon}\lambda_{\max}(\mathbb{E}(xx^T))}{\sqrt{\log \frac{1}{\delta}}  \log \frac{dT}{\zeta}})^\frac{4}{5})$ and $\eta_{t-1}=\frac{2}{t+2}$  in Algorithm \ref{alg:2}, then with probability at least $1-\zeta$, 
\begin{equation}\label{aeq:32}
      L_\mathcal{D}(w^T)-\min_{w\in \mathcal{W}}L_\mathcal{D} (w)\leq O(\frac{\lambda_{\max}^\frac{1}{5}(\mathbb{E}(xx^T)){(\sqrt{\log \frac{1}{\delta}}  \log \frac{dn}{\zeta})^\frac{4}{5}}}{(n\epsilon)^\frac{2}{5}}). 
\end{equation}
\end{theorem}
Truncating or shrunking the data to let them has bounded norm (or bounded sensitivity) is a commonly used technique in previous study on DP machine learning such as \cite{barber2014privacy,cai2019cost,cai2020cost}. However, all of these methods need to assume the data distribution is sub-Gaussian so that truncation may not lose too much information about the original record. Here we generalized to a heavy-tailed case, which may could be used to other problems. Moreover, the thresholds in the truncation step for  sub-Gaussian and heavy-tailed cases are also quite different. In the sub-Gaussian case, the threshold always depends on the sub-Gaussian parameter and $\log n, \log d$, while in Algorithm \ref{alg:2} we set the threshold as a function of $n, \epsilon$ and $T$. 

\section{Heavy-tailed DP-SCO for Sparse Learning}
\subsection{Private Heavy-tailed Sparse Linear Regression }\label{linearregression}
In the previous section, we studied DP-SCO over polytope constraint. However, in the high dimensional statistics we always assume the underlying parameter has additional structure of sparsity. Directly solving DP-SCO over $\ell_1$-norm ball constraint may not provide efficient estimation to the sparse parameters.  In this section, we will focus on  sparse learning with heavy-tailed data. Specifically, we will consider two canonical models, one is the sparse linear model, the other one is the DP-SCO over sparsity constraint, which includes sparse regularized logistic regression and spase mean estimation. First we consider the sparse linear regression, where  for each pair $(x, y)$ we have a linear model, $$y=\langle w^*, x\rangle +\iota, $$
here $\iota$ is some randomized noise and $\|w^*\|_2\leq C$ (for simplicity we assume $C=1$) and $w^*$ is $s^*$-sparse.

Similar to the previous section, here we assume Assumption \ref{ass:3} holds. Instead of using DP variants of the Frank-Wolfe method, here we will adopt a private variant of the iterative hard thresholding (IHT) method. Specifically, first we will shrunk the original heavy-tailed data, which is similar to Algorithm \ref{alg:2}. After that we will perform the DP-IHT procedure. That is, in each iteration, we fist calculate the gradient on the shrunken data, and update our vector via the gradient descent. Next, we perform a DP-thresholding step, provided by \cite{cai2019cost}  (Algorithm \ref{alg:4}). That is, we will privately select the indices with largest $s$ magnitude of the vector, keep the entries of vectors among these indices and let the remain entries be $0$. See Algorithm \ref{alg:3} and \ref{alg:4} for details. 

\begin{algorithm}
	\caption{Heavy-tailed Private Sparse Linear Regression \label{alg:3}}
	\begin{algorithmic}[1]
		\State {\bfseries Input:} $n$-size dataset $D=\{(x_i, y_i)\}_{i=1}^n$, loss function $\ell(w, (x,y))=(\langle w, x \rangle-y)^2$, initial vector $w^1$ satisfies $\|w^1\|_2\leq 1$ and is $s$-sparse, parameters $K, T, \eta_0, s$ (will be specified later), privacy parameter $\epsilon, \delta$, failure probability $\zeta$.  $\mathcal{W}$ is the unit $\ell_2$-norm ball. 
\State For each $i\in [n]$, we denote a truncated sample $\tilde{x}_i\in \mathbb{R}^d$ where 
for $j\in [d]$ $\tilde{x}_{i, j}=\text{sign}(x_{i,j})\min\{|x_{i,j}|, K\}$, and $\tilde{y}_i=\text{sign}(y_i)\min\{|y_i|, K\}$. Denote the truncated dataset as $\tilde{D}=\{(\tilde{x}_i, \tilde{y}_i)\}_{i=1}^n$. 
\State Split the data $\tilde{D}$ into $T$ parts $\{\tilde{D}_t\}_{t=1}^T$, each with $m=\frac{n}{T}$ samples. 
        \For{$t=1, \cdots, T$}. 
 
        \State Denote 
        $w^{t+0.5}=w^{t}-\frac{\eta_0 }{m}\sum_{x\in \tilde{D}_t} \tilde{x}(\langle \tilde{x},  w^{t} \rangle-\tilde{y}) $
        \State Let ${w}^{t+0.75}=\text{Peeling}(w^{t+0.5}, D_t, s, \epsilon, \delta, \frac{2K^2\eta_0(\sqrt{s}+1)}{m}).$
        \State Let $w^{t+1}=\Pi_{\mathcal{W}}({w}^{t+0.75})$
        \EndFor \\
        \Return $w^{T+1}$.
	\end{algorithmic}
\end{algorithm}

\begin{theorem}\label{thm:6}
For any $0<\epsilon, \delta<1$, Algorithm \ref{alg:3} is $(\epsilon, \delta)$-DP.
\end{theorem}

\begin{theorem}\label{thm:7}
Under Assumption \ref{ass:3},  if $\|w^*\|_2\leq \frac{1}{2}$, the initial vector $w^1$ satisfies $\|w^1-w^*\|\leq O(\frac{\gamma}{\mu})$ and $n$ is sufficiently large such that 
$n\geq \tilde{O}( \frac{s^2 M\log^2 \frac{d}{\zeta}\log\frac{1}{\delta}}{\gamma \mu^4 \epsilon}).$
Then if we set $T=\tilde{O}(\frac{\gamma}{\mu}\log n)$, $K=\frac{(n\epsilon)^\frac{1}{4}}{(sT)^\frac{1}{4}}$, $s\geq 72(\frac{\gamma}{\mu})^2 s^*$ and $\eta=\frac{2}{3\gamma}$ in Algorithm \ref{alg:3}, then with probability at least $1-\zeta$
\begin{align}
   & L_\mathcal{D}(w^{T+1})- L_\mathcal{D}(w^*)\leq O(\frac{M \gamma^4 s^{*2}\log n \log^2 \frac{d}{\zeta}\log\frac{1}{\delta}}{\mu^7 n\epsilon}),  \nm
\end{align}
where $\gamma=\lambda_{\max}(\mathbb{E}(xx^T))$ and $\mu=\lambda_{\min}(\mathbb{E}(xx^T))$ and the Big-$O$ notation omits other $\log$ terms. 
\end{theorem}

\begin{algorithm}
	\caption{Peeling \cite{cai2019cost} \label{alg:4}}
	\begin{algorithmic}[1]
		\State {\bfseries Input:} Vector $v= v(D)\in \mathbb{R}^d$ which depends on the data $D$, sparsity $s$, privacy parameter $\epsilon, \delta$, and noise scale $\lambda$. 
		\State Initialize $S=\emptyset$. 
		\For{$i=1\cdots s$}
		\State Generate $w_i\in \mathbb{R}^d$ with $ w_{i,1}, \cdots, w_{i,d}\sim \text{Lap}(\frac{2\lambda\sqrt{3s\log\frac{1}{\delta}}}{\epsilon}). $
		\State Append $j^*=\arg\max_{j\in [d]\backslash S}|v_j|+w_{i, j}$ to $S$.
		\EndFor
        		\State Generate $\tilde{w}\in \mathbb{R}^d$ with $ \tilde{w}_{1}, \cdots, \tilde{w}_{d}\sim \text{Lap}(\frac{2\lambda\sqrt{3s\log\frac{1}{\delta}}}{\epsilon}). $ \\ 
        \Return $v_S+\tilde{w}_S$. 
	\end{algorithmic}
\end{algorithm}
\begin{remark}
In Theorem \ref{thm:7}, we need to assume that $\|w^*\|_2\leq \frac{1}{2}$ and the initial vector be close to $w^*$. These two conditions guarantee $\|w^{t+0.75}\|_2\leq 1$ in each iteration, which simplify our theoretical analysis. 
For sub-Gaussian data, with some other additional assumptions, \cite{cai2019cost,wang2019differentially12} showed that the optimal rate is $\tilde{O}(\frac{s^*\log d}{n}+\frac{(s^*\log d)^2}{(n\epsilon)^2})$. Thus, due to the data irregularity, the error now increases to $\tilde{O}(\frac{s^2\log^2 d}{n\epsilon})$.  Moreover, we can see although both Algorithm \ref{alg:3} and \ref{alg:2} shrunk the data in the first step, the threshold value $K$ are quite different, where $K=\frac{{(n\epsilon)^{\frac{1}{4}}}}{T^\frac{1}{8}}$ in LASSO and $K=\frac{(n\epsilon)^\frac{1}{4}}{(sT)^\frac{1}{4}}$ in the sparse linear model. This is due to different trade-offs between the bias, variance in the estimation error and the noises we added. 
\end{remark}
\subsection{Extending to Sparse Learning}
In this section, we extend our previous ideas and methods to the problem of DP-SCO over sparsity constraints. 
That is, $\mathcal{W}$ is defined as $\mathcal{W}=\{w: \|w\|_0\leq s^*\}$. We note that such a formulation encapsulates several important problems such as the $\ell_0$-constrained linear/logistic regression \cite{bahmani2013greedy}.  DP-SCO over sparsity constraints has been studied previously \cite{wang2019differentially12,wang2020knowledge,Wang019a,WangX21}. However, all of the previous methods need either the loss function is Lipschitz, or the data follows some sub-Gaussian distribution \cite{cai2019cost,cai2020cost}.  In the following we extend to the heavy-tailed case. We first introduce some assumptions to the loss functions, which are commonly used in previous research on sparse learning. 

\begin{definition}[Restricted Strong Convexity, RSC]
A differentiable function $f(x)$ is restricted $\rho_r$-strongly convex  with parameter $r$ if there exists a constant $\mu_r>0$ such that for any $x, x'$ with $\|x-x'\|_0\leq r$, we have $f(x)-f(x')-\langle \nabla f(x'),x-x'\rangle \geq \frac{\mu_r}{2}\|x-x'\|_2^2.$
\end{definition}
\begin{definition}[Restricted Strong Smoothness, RSS]
A differentiable function $f(x)$ is restricted $\mu_s$-strong smooth  with parameter $r$ if there exists a constant $\gamma_r>0$ such that for any $x, x'$ with $\|x-x'\|_0\leq r$, we have $f(x)-f(x')-\langle \nabla f(x'),x-x'\rangle \leq  \frac{\gamma_r}{2}\|x-x'\|_2^2.$
\end{definition}

\begin{assumption}\label{ass:4}
 We assume that the objective function $L_\mathcal{D}(\cdot)$ is $\mu_r$-RSC and $\ell(w, z)$ is $\gamma_r$-RSS with parameter $r=2s+s^*$, where $s=O((\frac{\gamma_r}{\mu_r})^2s^*)$. We also assume for any $w\in \mathcal{W}'$ and each coordinate $j\in [d]$, we have $\mathbb{E}[(\nabla_j \ell(w, x))^2]\leq \tau=O(1)$, where $\tau$ is some known constant and $\mathcal{W}'=\{w|\|w\|_0\leq s\}$. 
\end{assumption}
Many problems satisfy Assumption \ref{ass:4}, e.g.,  mean estimation and  $\ell_2$-norm regularized generalized linear loss where  $L_\mathcal{D}(w)=\mathbb{E}[\ell(y\lge w, x\rge)]+\frac{\lambda}{2}\|w\|_2^2$. If $|\ell'(\cdot)|\leq O(1)$, $|\ell''(\cdot)|\leq O(1)$ (such as the logistic loss) and $x_j$ has bounded second-order moment, then we can see it satisfies Assumption \ref{ass:4}.

Since now the loss function becomes non-linear, the approach of shrunking the data in Algorithm \ref{alg:3} may introduce tremendous error. However, since in Assumption \ref{ass:4} we have stronger assumptions on the loss function, we may use the private estimator in Algorithm \ref{alg:1}. Thus, our idea is that, we first perform the robust one-dimensional mean estimator in (\ref{eq:2})-(\ref{eq:5})  to each coordinate of the gradient, then we use the private selection algorithm to select top $s$ indices, which is the same as in Algorithm \ref{alg:3}. Note that \cite{wang2020differentially} also provides a similar method in the low dimensional space. However, the main difference is that here we do not add noise directly to the vector $\tilde{g}(w^{t-1}, D_t)$. Instead, we first privately select the top $s$ indices and then add noises to the corresponding sub-vector. See Algorithm \ref{alg:5} for details. 

\begin{algorithm}
	\caption{Heavy-tailed Private Sparse Optimization \label{alg:5}}
	\begin{algorithmic}[1]
		\State {\bfseries Input:} $n$-size dataset $D=\{(x_i, y_i)\}_{i=1}^n$, initial parameter $w^1$ is $s$-sparse, parameters $s, \beta, k, T, \eta$ (will be specified later), privacy parameter $\epsilon, \delta$, failure probability $\zeta$.  
\State Split the data $D$ into $T$ parts $\{{D}_t\}_{t=1}^T$, each with $m=\frac{n}{T}$ samples. 
        \For{$t=1, \cdots, T$}. 
        \State For each $j\in [d]$, calculate the robust gradient by (\ref{eq:2})-(\ref{eq:5}), that is 
            \begin{align*} 
        &g_j^{t-1}(w^{t-1}, D_t)\\
        &= \frac{1}{m}\sum_{x\in D_{t}} \left(\nabla_j \ell(w^{t-1}, x)\big(1-\frac{\nabla^2_j \ell(w^{t-1}, x)}{2k^2\beta}\big)- \frac{\nabla^3_j \ell(w^{t-1}, x)}{6k^2}\right)\\&+\frac{k}{m}\sum_{x\in D_{t}}\hat{C}\left(\frac{\nabla_j \ell(w^{t-1}, x)}{k}, \frac{|\nabla_j \ell(w^{t-1}, x)|}{k\sqrt{\beta}}\right).
    \end{align*}
      \State Let vector $\tilde{g}(w^{t-1}, D_t)\in \mathbb{R}^d$ as 
      $\tilde{g}(w^{t-1}, D_t)=(g_1^{t-1}(w^{t-1}, D_t), g_2^{t-1}(w^{t-1}, D_t), \cdots, g_d^{t-1}(w^{t-1}, D_t))$. 
        \State Denote 
        $w^{t+0.5}=w^{t}-\eta\tilde{g}(w^{t-1}, D_t) $
        \State Let ${w}^{t+1}=\text{Peeling}(w^{t+0.5}, D_t, s, \epsilon, \delta, \frac{4k\sqrt{2}\eta}{m}).$
        \EndFor \\
        \Return $w^{T+1}$.
	\end{algorithmic}
\end{algorithm}
\begin{theorem}\label{thm:8}
For any $0<\epsilon, \delta<1$, Algorithm \ref{alg:5} is $(\epsilon, \delta)$-DP. Moreover, under Assumption \ref{ass:4}, if we set $T=\tilde{O}(\frac{\gamma_r}{\mu_r}\log n)$, $s=O((\frac{\gamma_r}{\mu_r})^2s^*)$, $\beta=O(1)$,  $\eta=\frac{2}{3\gamma_r}$  and $k=\tilde{O}(\sqrt{n\epsilon\tau})$, then with probability at least $1-\zeta$, 
$$L_\mathcal{D}(w^{T+1})-L_\mathcal{D}(w^*)\leq O(\frac{\tau \gamma_r^4 s^{*\frac{3}{2}}\log n\log \frac{d}{\zeta}\sqrt{\log\frac{1}{\delta}}}{\mu^5_r n\epsilon}),$$  
where the Big-$O$ notation omits other $\log$ terms.
\end{theorem}
\begin{remark}\label{remark:4}
 Compared with Theorem \ref{thm:7}, we can see here we do not need the assumptions on $\|w^*\|_2$ and the initial vector. This is due to that we have stronger assumptions on the loss function. Compared with the bound $\tilde{O}(\frac{s^{*2}}{n\epsilon})$ in Theorem \ref{thm:7}, it seems like here our bound is lower. However, we note that they are incomparable due to different assumptions. For example, there is a $\tau$ in the bound of Theorem \ref{thm:8}, which could also depend on the sparsity $s^*$ \cite{DBLP:journals/corr/abs-2010-13520}. \cite{Wang019a} also studies DP-SCO over sparsity constraint, it provides an upper bound of $\tilde{O}(\frac{s^*}{n^2\epsilon^2})$ under the assumption that the loss function is Lipschitz. Moreover, for high dimensional sparse mean estimation and Generalized Linear Model (GLM) with the Lipschitz loss and sub-Gaussian data, \cite{cai2020cost,cai2019cost} provided optimal rates of  $\tilde{O}(\frac{s^*\log d}{n}+\frac{(s^*\log d)^2}{(n\epsilon)^2})$. We can see that compared with these results, the error bound now becomes to $\tilde{O}(\frac{\tau s^{*\frac{3}{2}}}{ n\epsilon})$ due to data irregularity. Moreover, we can see that in the regular data case, the optimal rates of linear regression and GLM are the same, while in the heavy-tailed data case, there is a gap of $\tilde{O}(\sqrt{s^*})$ in the upper bounds. We conjecture this gap is necessary and will leave it as future research.
\end{remark}
In the following we will focus on the lower bound  of the loss functions in Theorem \ref{thm:8}.  Since our lower bound will be in the form of private minimax risk,  we first introduce the classical statistical minimax risk before discussing its $(\epsilon, \delta)$-private version. More details can be found in  \cite{barber2014privacy}.
 
 Let  $\mathcal{P}$ be  a class of distributions  over a data universe $\mathcal{X}$. For each distribution $p\in \mathcal{P}$, there is a deterministic function $\theta(p)\in \Theta$, where $\Theta$ is the parameter space. Let $\rho: \Theta \times \Theta :\mapsto \mathbb{R}_+ $ be  a semi-metric function on the space $\Theta$ and $\Phi: \mathbb{R}_+\mapsto \mathbb{R}_+$ be a non-decreasing function with $\Phi(0)=0$ (in this paper, we assume that  $\rho(x,y)=|x-y|$ and $\Phi(x)=x^2$ unless specified otherwise). We further assume that  $D=\{X_i\}_{i=1}^{n}$ are  $n$ i.i.d observations drawn according to some distribution $p\in \mathcal{P}$, and   $\hat{\theta}:\mathcal{X}^n\mapsto \Theta$ be some estimator. Then the minimax risk in metric $\Phi\circ \rho$ is defined by the following saddle point problem:
\begin{equation*}
\mathcal{M}_n(\theta(\mathcal{P}), \Phi\circ \rho):=\inf_{\hat{\theta}}\sup_{p\in \mathcal{P}}\mathbb{E}_p[\Phi(\rho(\hat{\theta}(D), \theta(p))],
\end{equation*}
where the supremum is taken over distributions $p\in \mathcal{P}$ and the infimum over all estimators $\hat{\theta}$.

In the $(\epsilon, \delta)$-DP model, the estimator $\hat{\theta}$ is obtained via some $(\epsilon, \delta)$-DP mechanism $Q$. Thus, we can also define the $(\epsilon, \delta)$-private minimax risk:   

\begin{equation*}
\mathcal{M}_n(\theta(\mathcal{P}), Q, \Phi\circ \rho):=\inf_{Q\in \mathcal{Q}}\inf_{\hat{\theta}}\sup_{p\in \mathcal{P}}\mathbb{E}_{p, Q}[\Phi(\rho(\hat{\theta}(D), \theta(p))],
\end{equation*}
where $\mathcal{Q}$ is the set of all the $(\epsilon, \delta)$-DP mechanisms. 

To proof the lower bound, we consider the sparse mean estimation problem, {\em i.e.,} $L_\mathcal{D}(w)= \mathbb{E}_{x\sim \mathcal{D}}[\|x-w\|_2^2 ]$, where the mean of $x$, $\mu(\mathcal{D})$, is $s^*$-sparse. Thus, we can see that the population risk function  satisfies Assumption \ref{ass:4} if we assume $\mathbb{E}x_j^2\leq \tau$ for each $j\in [d]$. Moreover, we have $\min_{w\in \mathcal{W}}L_\mathcal{D}(w)=0$ which indicates that the excess population risk of $w$ is equal to $\mathbb{E}\|w-\mu(\mathcal{D})\|_2^2$. That is, the lower bound of Theorem \ref{thm:8} reduced to the sparse mean estimation problem. Therefore, it is sufficient for us to consider the $(\epsilon, \delta)$-private minimax rate for the sparse mean estimation problem with $\mathbb{E}x_j^2\leq \tau$ for each $j\in [d]$.

In the non-private case, a standard approach to prove the lower bound of the minimax risk is reducing the original problem to a testing problem. Specifically, our goal is to identify a parameter $\theta\in \Theta$ from a finite collection of well-separated points. Given an index set $\mathcal{V}$ with finite cardinality, the indexed family of distributions   $\{P_v, v\in \mathcal{V}\}\subset \mathcal{P}$ is said to be a $2\gamma$-packing if $\rho(\theta(P_v), \theta(P_{v'}))\geq 2\gamma$ for all $v\neq v'\in \mathcal{V}$. In the standard hypothesis testing problem, nature chooses $V \in \mathcal{V}$ uniformly at random, then draws samples $X_1, \cdots X_n$ i.i.d. from the distribution $P_V$. The problem is to identify the index $V$. It has been shown that given a $2\gamma$-packing $\{P_v, v\in \mathcal{V}\}\subset \mathcal{P}$, 

\begin{equation*}
    \mathcal{M}_n(\theta(\mathcal{P}), \Phi\circ \rho)\geq \Phi(\gamma)\inf_{\psi}\mathbb{P}(\psi(D)\neq V), 
\end{equation*}
where $\mathbb{P}$ denotes the probability under the joint distribution of both $V$ and the samples $D$.

Similar to the non-private case, for the private minimax risk we have 
\begin{equation*}
    \mathcal{M}_n(\theta(\mathcal{P}), Q, \Phi\circ \rho)\geq \inf_{Q\in \mathcal{Q}}\Phi(\gamma)\inf_{\psi}\mathbb{P}_Q(\psi(\hat{\theta}(D))\neq V), 
\end{equation*}
where $\hat{\theta}(D)$ is the private estimator via some $(\epsilon, \delta)$-DP algorithm $Q$, where $\mathbb{P}_Q$ denotes the probability under the joint distribution of both $V$, the samples $D$ and $\hat{\theta}(D)$.

In the following we will consider a special indexed family of distributions  $\{P_v\}_{v\in \mathcal{V}} \subset \mathcal{P}$, which will be used in our main proof. We assume there exists a distribution $P_0$ such that for some fixed $p\in [0, 1]$ we have $(1-p)P_0+p P_v\in \mathcal{P}$ for all $v\in \mathcal{V}$. For simplicity for each $v\in \mathcal{V}$ we define the following parameter 
\begin{equation*}
    \theta_v:=\theta((1-p)P_0+pP_v).
\end{equation*}
We then define the separation of the set $\{\theta\}_v$ by 
\begin{equation*}
    \rho^*(\mathcal{V}):=\min\{\rho(\theta_v, \theta_{v'})| v, v'\in \mathcal{V}, v\neq v'\}. 
\end{equation*}
We have the following lower bound of $(\epsilon, \delta)$-private minimax risk based on the the family of distributions $\{(1-p)P_0+pP_v\}_{v\in \mathcal{V}}$. 

\begin{lemma}[Theorem 3 in \cite{barber2014privacy}]\label{thm:lower}
Fix $p\in [0,1]$ and define ${P_{\theta_v}}= (1-p)P_0+pP_v\in \mathcal{P}$. Let $\hat{\theta}$ be an $(\epsilon, \delta)$-DP estimator. Then 
\begin{align}
     &\mathcal{M}_n(\theta(\mathcal{P}), Q, \Phi\circ \rho) \nm \\
     &\geq \Phi(  \rho^*(\mathcal{V}))\frac{1}{|\mathcal{V}|}\sum_{v\in \mathcal{V}} P_{\theta_v}(\rho(\hat{\theta}, \theta_v)\geq \rho^*(\mathcal{V}))\notag \\ & \geq \Phi(  \rho^*(\mathcal{V})) \frac{(|\mathcal{V}|-1)(\frac{1}{2}e^{-\epsilon \lc np \rc}-\delta \frac{1-e^{-\epsilon \lc n p \rc} }{1-e^{-\epsilon}})  }{1+(|\mathcal{V}|-1)e^{-\epsilon \lc n p \rc} }. 
\end{align}
\end{lemma}
By Lemma \ref{thm:lower} and a set of hard distributions, we have the following result.

\begin{theorem}\label{thm:9}
Consider the class of distributions $\mathcal{P}_{d}^{s^*}(\tau)$ as distributions $P$ in the $d$ dimensional space satisfying that $\mathbb{E}_{X\sim P}X_j^2\leq \tau$ for all $x_j\in [d]$ and the mean of $P$, $\mu(P)$ is $s^*$-sparse. Then the $(\epsilon, \delta)$-private minimax risk with $\Phi(x)=x^2$ and $\rho(x_1, x_2)=\|x_1-x_2\|_2$ satisfies that 
\begin{equation}
        \mathcal{M}_n(\theta(\mathcal{P}_{d}^{s^*}(\tau)), Q, \Phi\circ \rho) \geq \Omega (\frac{\tau \min \{s^*\log d, \log \frac{1}{\delta}\} }{n\epsilon}).
\end{equation}
Thus, for DP-SCO problem  under Assumption \ref{ass:4}. The information-theoretical lower bound of the expected population risk in the $(\epsilon, \delta)$-DP model is  $\Omega (\frac{\tau \min \{s\log d, \log \frac{1}{\delta}\} }{n\epsilon})$. 
\end{theorem}
Compared with the upper bound in Theorem \ref{thm:8} and the lower bound in Theorem \ref{thm:9}, we can see there is still a gap of $\tilde{O}(\sqrt{s^*})$. It is an open problem that whether we can further improve the upper bound. For the low dimensional case, \cite{kamath2021improved,barber2014privacy}  showed that the optimal rate of the mean estimation is $O(\frac{\tau d}{n\epsilon})$ in both $\epsilon$ and $(\epsilon, \delta)$-DP models under the assumption that the gradient of loss has bounded second order moment. Compared with this  here we extend to  the high dimensional sparse case. 
\section{Experiments}
In this section we will study the practical behaviours of our previous algorithms on synthetic and real world datas. 
\subsection{Data Generation}
We will mainly study squared loss and logistic loss, and in the following we first clarify the synthetic data generation process for different loss functions. 

For linear regression, we generate the data via the model $y=\langle w^*, x \rangle + \iota$, where $\iota$ is some (heavy-tailed) noise and $x$ is sampled from different (heavy-tailed) distributions (see the following subsections for details). Since in the previous parts we considered two different settings, we generate $w^*$ via different approaches: For the case where the domain is a polytope, we randomly generate a $w^*$ such that $\|w^*\|_1$. Thus, $\mathcal{W}$ will be the unit $\ell_1$-norm ball. For sparse linear regression setting, we first pre-fix the sparsity $s^*$, then we  sample a $w$ form the normal distribution with $mean=0$ and $scale=100$ and  set random $(d-s^{*})$ elements to $0$. After that  we project the vector to the unit $\ell_2$-norm ball and denote the vector after projection as $w^*$. 

For logistic regression,  we generate the data as the followings $y=\text{sign}(\text{sigmoid}(z)-0.5)$ where $\text{sigmoid}(x)=\frac{1}{1+e^{-x}}$ is the sigmoid function and  $z=\langle x,
w^{*}\rangle + \zeta$ where $\zeta$ is some noise and $x$ is sampled from different distributions (see the following subsections for details). The generation of the optimal vector $w^*$ is the same as in the linear regression case.  

For real world data, we will use Blog Feedback data ($n=60021, d=281$), Twitter data ($n=583,249, d=77$),  Winnipeg data ($n=325834, d=175$) and Year Prediction data ($n=515,345, d=90$), which could be found at \cite{Dua:2019}. 

\subsection{Experimental Settings}
\paragraph{Measurements:} For all the experiments we will use the excess population risk, {\em i.e., $L_\mathcal{D}(w)-L_\mathcal{D}(w^*)$}
as our evaluation measurement. However, since it is impossible to precisely evaluate the population risk function, here we will use the empirical risk  to approximate it. 
For the real world data, in the case where $\mathcal{W}$ is some polytope, we use the non-private Frank-Wolfe algorithm to get the optimal parameter $w^*=\arg\min_{w\in \mathcal{W}}L_\mathcal{D}(w)$. 
We note that all of our experiments are repeated for at least $20$ times, and we the average of the results as our final results. 
\paragraph{Initial Vector:} For the initial vector, in the polytope case, we randomly sample a $w$ in the unit $\ell_1$-norm ball. And in the sparse case, we randomly generate a $s^*$-sparse vector in the unit $\ell_2$-norm ball. 
\paragraph{Parameters in Algorithms:}
Most of the parameters are directly followed by theoretical results. In Algorithm \ref{alg:1} we set $T= \lfloor (n\epsilon)^\frac{1}{3} \rfloor$, $\eta_t=\frac{2}{t+2}$ and $s=\lfloor n\epsilon \rfloor$. In Algorithm \ref{alg:2} we set $T=(n\epsilon)^\frac{2}{5}$, $K=\frac{(n\epsilon)^\frac{1}{4}}{T^\frac{1}{8}}$ and $\eta_t=\frac{2}{t+2}$. In Algorithm \ref{alg:3} we set $s=cs^{*}, T=\lfloor \log(n) \rfloor, K=(\frac{n \epsilon}{sT})^{\frac{1}{4}}$ and $\eta=0.5$, where $c$ is some integer which is different under various settings.  In Algorithm \ref{alg:5} we set $s=2s^{*}, T=\lfloor \log(n) \rfloor, k=c_2 n \epsilon$ and $\eta=0.5$, where $c_2$ is some integer which is different under various settings. For $(\epsilon, \delta)$-DP, we set $\delta=\frac{1}{n^{1.1}}$.

\subsection{Results of Algorithm \ref{alg:1}} 
We first study Algorithm \ref{alg:1} on synthetic data. We start from the experiments on linear regression, where $x$ is sampled from the log-normal distribution $\text{Lognormal}(0, 0.6)$ and $\zeta\sim \mathcal{N}(0, 0.1)$ (Note that a log-normal distribution $\text{Lognormal}(0, \sigma^2)$ has the PDF $p(w)=\frac{1}{w\sigma \sqrt{2\pi}} \exp(-\frac{\ln^2 w}{2\sigma^2})$). Figure \ref{fig:1} shows the results: Figure \ref{fig1a} studies the excess empirical risk with different dimensions and privacy parameters $\epsilon$, under the setting where $n=10^4$. Figure \ref{fig1b} reveals the error with different dimensions and sample sizes with fixed $\epsilon=1$. Figure \ref{fig1c} studies the difference of empirical risk between private and non-private case with different sample sizes,  where $\epsilon=1$ and $d=400$. 

 \begin{figure*}[!htbp]
\centering
\subfigure[Sample size $n=10^4$ \label{fig1a}]{
\includegraphics[width=0.31\textwidth]{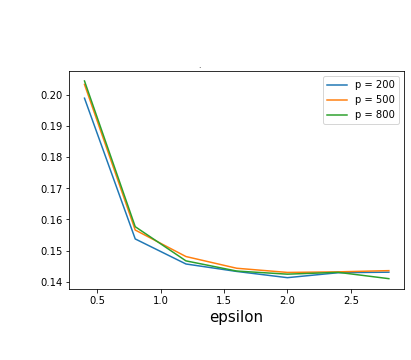}
}
\subfigure[$\epsilon=1$ \label{fig1b}]{
\includegraphics[width=0.31\textwidth]{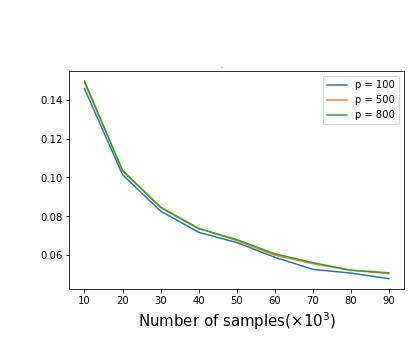}
}
\subfigure[$\epsilon=1$ and $d=400$ \label{fig1c}]{
\includegraphics[width=0.31\textwidth]{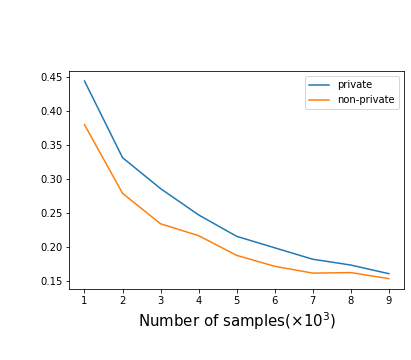}
}
\caption{ Results of Algorithm \ref{alg:1} for linear regression with $x$ sampled from log-normal distribution.}
\label{fig:1}
\end{figure*}

Besides the linear regression, we also study the logistic regression with various dimensions, sample sizes and privacy parameters $\epsilon$,  where $x$ is sampled from $\text{Lognormal}(0, 0.6)$ and there is no noise, see Figure \ref{fig:2} for details. 

 \begin{figure*}[!htbp]
\centering
\subfigure[Sample size $n=10^4$ \label{fig2a}]{
\includegraphics[width=0.31\textwidth]{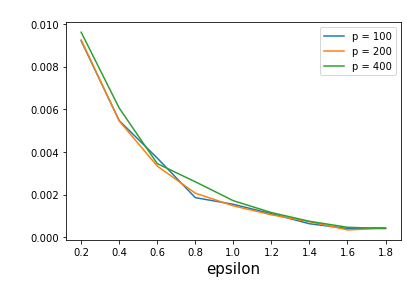}
}
\subfigure[$\epsilon=1$ \label{fig2b}]{
\includegraphics[width=0.31\textwidth]{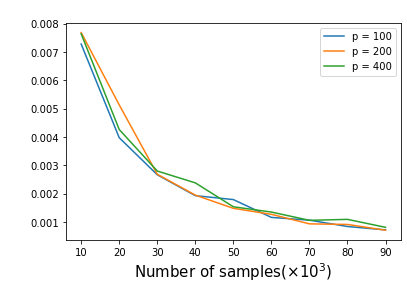}
}
\subfigure[$\epsilon=1$ and $d=400$ \label{fig2c}]{
\includegraphics[width=0.31\textwidth]{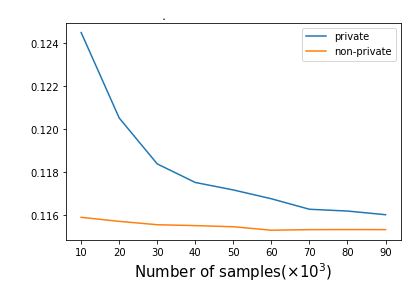}
}
\caption{ Results of Algorithm \ref{alg:1} for logistic regression with $x$ sampled from log-normal distribution.}
\label{fig:2}
\end{figure*}

From the results in Figure \ref{fig:1} and \ref{fig:2} we can see that: 1)The dimension of the data does not affect too much on the error. For example, the errors are almost the same for $p=200$ and $p=800$ under the same $\epsilon$ and $n$. This is due to that theoretically the error is just upper bounded by logarithm, instead of polynomial, of the dimension.  Moreover, when the sample size is larger, the error becomes smaller. This is due to that the error is disproportionate to sample size. However, the rate of this decrease is not fast. For example, in Figure \ref{fig1b} we can see that the error decreases from 0.14 to 0.03 when the sample size increases from $10^4$ to $9\times 10^4$. This possibly is due to that the error is not proportional to $\frac{1}{n}$, as we showed in Theorem \ref{thm:2}. We can also see that when the sample size is large enough, the private estimator is close to the optimal parameter. 

In Figure \ref{fig:25} and \ref{fig:26}, we conduct experiments on real world data for linear regression and logistic regression, respectively. We can see that in most of the cases, the error decreases when the sample size or the privacy parameter $\epsilon$ becomes larger. However, compared with the previous results on synthetic data. The results in real data is unstable. The main reason may be the real data does not satisfy our assumptions or it is difficult to estimation the second order moment of the data. 
 \begin{figure*}[!htbp]
\centering
\subfigure[Blog \label{fig25a}]{
\includegraphics[width=0.31\textwidth]{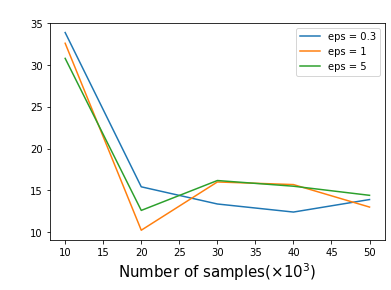}
}
\subfigure[Twitter \label{fig25b}]{
\includegraphics[width=0.31\textwidth]{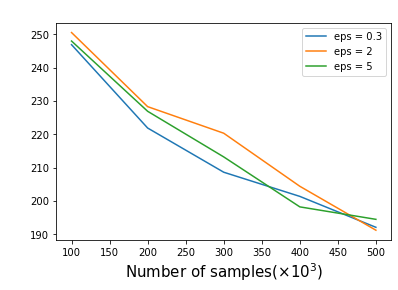}
}
\caption{ Results on real data of Algorithm \ref{alg:1} for linear regression.}
\label{fig:25}
\end{figure*}
 \begin{figure*}[!htbp]
\centering
\subfigure[Winnipeg\label{fig26a}]{
\includegraphics[width=0.31\textwidth]{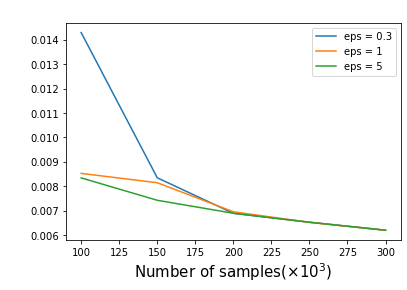}
}
\subfigure[Year Prediction \label{fig26b}]{
\includegraphics[width=0.31\textwidth]{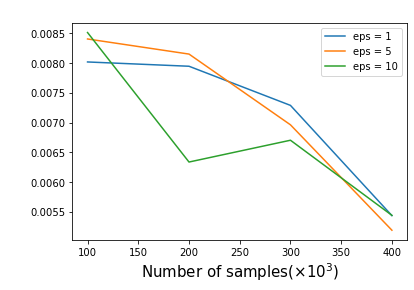}
}
\caption{ Results on real data of Algorithm \ref{alg:1} for logistic regression.}
\label{fig:26}
\end{figure*}
\subsection{Results of Algorithm \ref{alg:2}} 
Since Algorithm \ref{alg:2} is only for linear regression, in this section we only conduct experiments on squared loss. We first study the case where $x$ is sampled from  $\text{Lognormal}(0, 0.6)$ and $\zeta\sim \mathcal{N}(0, 0.1)$. Figure \ref{fig:3} shows the results: Figure \ref{fig3a} studies the excess empirical risk with different dimensions and privacy parameters $\epsilon$, under the setting where $n=10^4$. Figure \ref{fig3b} studies the excess empirical risk with different dimensions and sample size, under the setting where $\epsilon=1$. Figure \ref{fig3c} studies the difference of empirical risk between private and non-private case with different sample size,  where $\epsilon=1$ and $d=200$.  
 \begin{figure*}[!htbp]
\centering
\subfigure[Sample size $n=10^4$ \label{fig3a}]{
\includegraphics[width=0.31\textwidth]{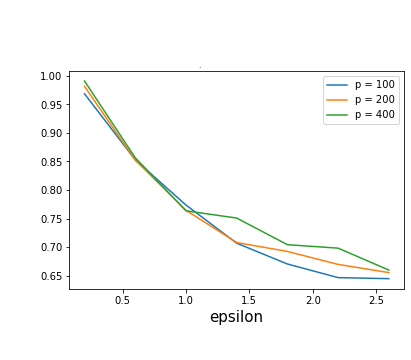}
}
\subfigure[$\epsilon=1$ \label{fig3b}]{
\includegraphics[width=0.31\textwidth]{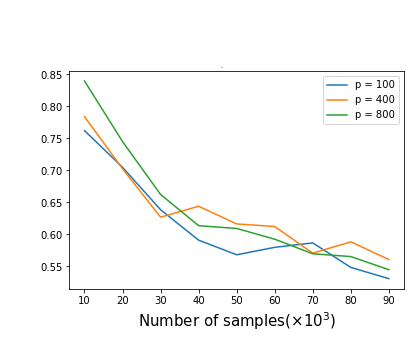}
}
\subfigure[$\epsilon=1$ and $d=200$ \label{fig3c}]{
\includegraphics[width=0.31\textwidth]{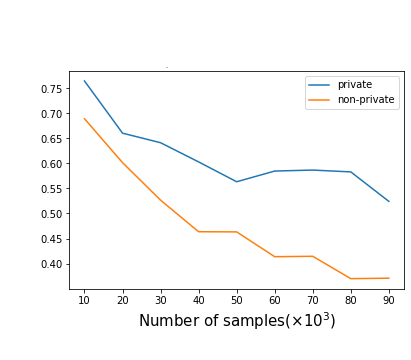}
}
\caption{ Results of Algorithm \ref{alg:2} for linear with $x$ sampled from log-normal distribution.}
\label{fig:3}
\end{figure*}

Besides the log-normal features, in Figure \ref{fig:4} we also study the case where $x$ is sampled from the Student's t-distribution with degree of freedoms $\upsilon=10$, {\em i.e., its PDF is $p(w)=\frac{\Gamma(11/2)}{\sqrt{10\pi}\Gamma(5)}(1+\frac{w^2}{10})^{-5.5}$}, and $\zeta\sim \mathcal{N}(0, 0.1)$. Figure \ref{fig4a} studies the error with various dimensions and $\epsilon$.  Figure \ref{fig4b} studies the error with various sample sizes and $\epsilon$ and Figure \ref{fig4c} provides detailed comparisons between the private and the non-private case. 

 \begin{figure*}[!htbp]
\centering
\subfigure[Sample size $n=10^5$ \label{fig4a}]{
\includegraphics[width=0.31\textwidth]{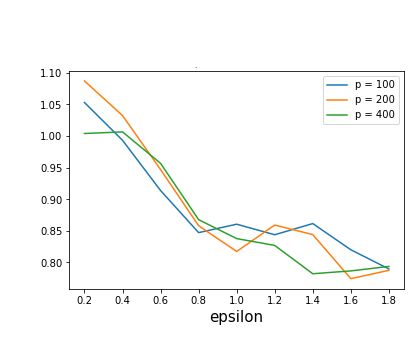}
}
\subfigure[$\epsilon=1$ \label{fig4b}]{
\includegraphics[width=0.31\textwidth]{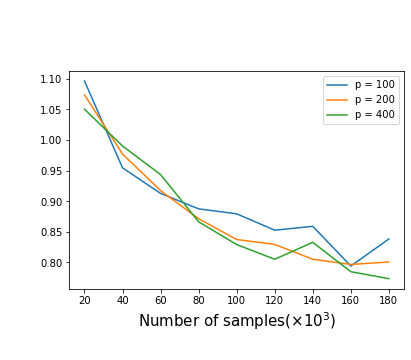}
}
\subfigure[$\epsilon=1$ and $d=200$ \label{fig4c}]{
\includegraphics[width=0.31\textwidth]{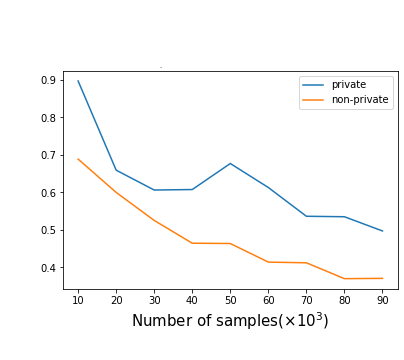}
}
\caption{ Results of Algorithm \ref{alg:2} for linear with $x$ sampled from Student's t-distribution.}
\label{fig:4}
\end{figure*}
 \begin{figure*}[!htbp]
\centering
\subfigure[Sample size $n=5\times 10^4$ and $s^*=20$ \label{fig5a}]{
\includegraphics[width=0.31\textwidth]{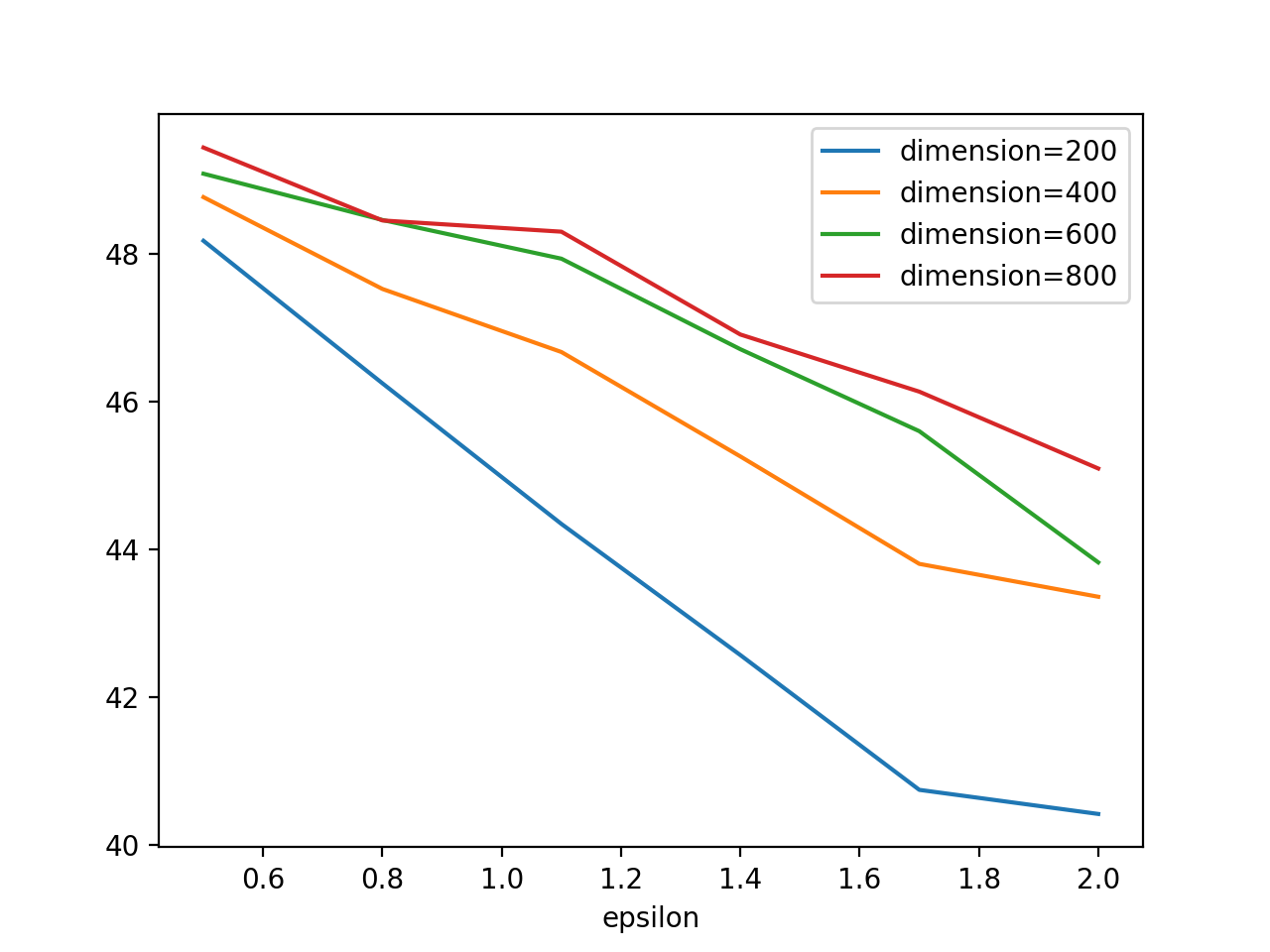}
}
\subfigure[$\epsilon=1$ and $s^*=20$ \label{fig5b}]{
\includegraphics[width=0.31\textwidth]{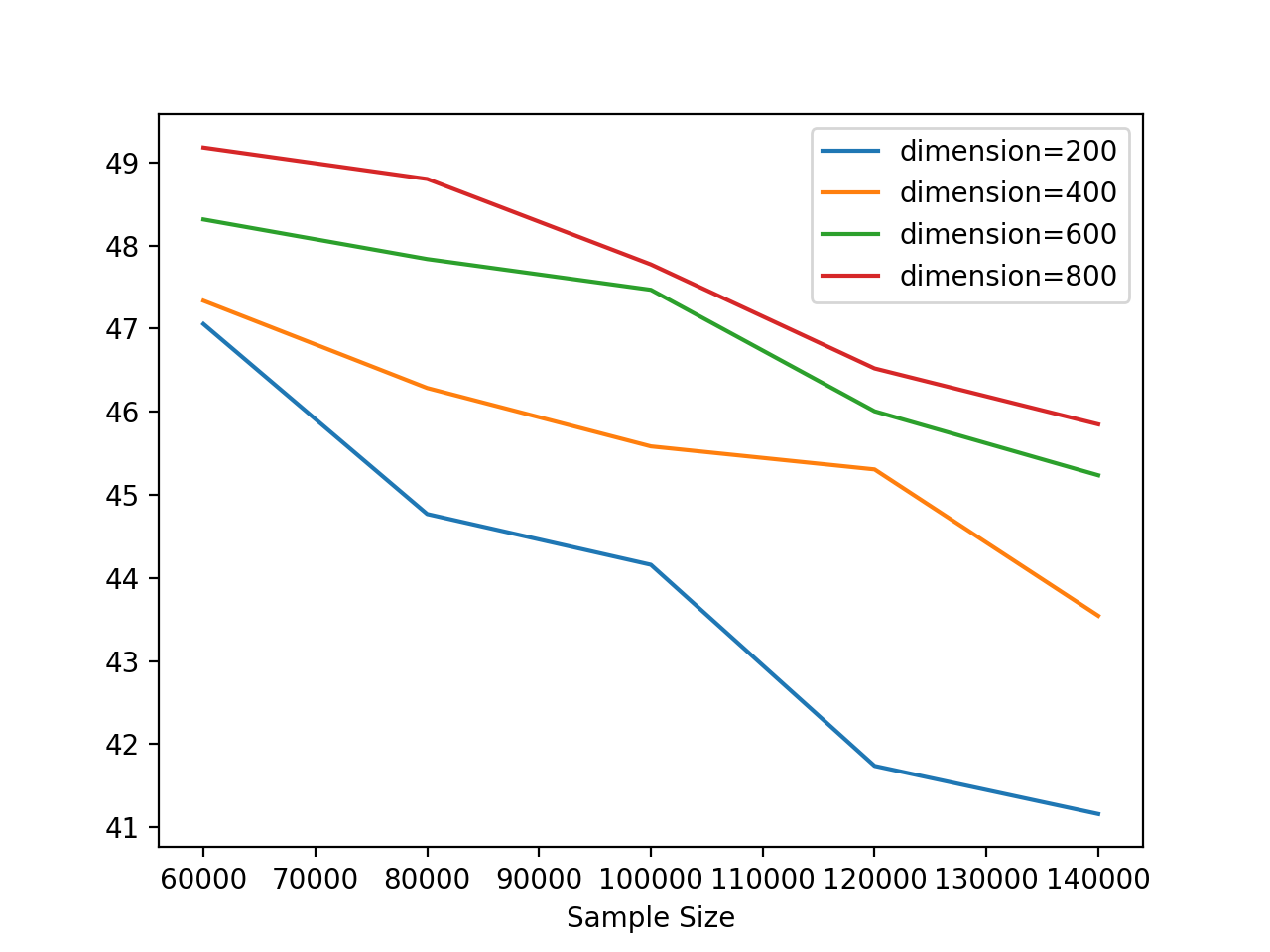} 
}
\subfigure[$\epsilon=1$ and $n=5\times 10^4$  \label{fig5c}]{
\includegraphics[width=0.31\textwidth]{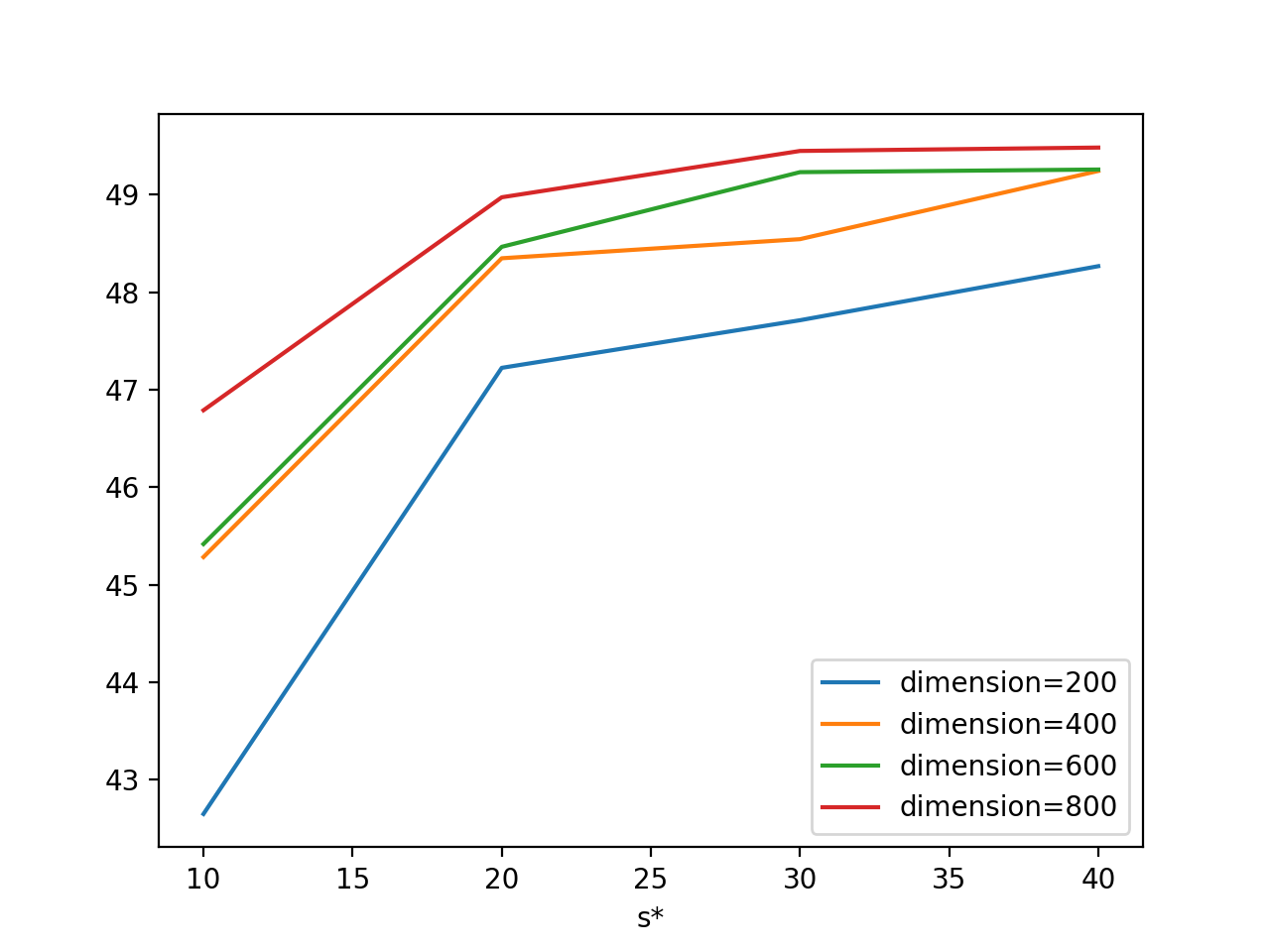}
}
\caption{ Results of Algorithm \ref{alg:3} for linear regression with $x$ sampled from Gaussian distribution and noise sampled from log-normal distribution.}
\label{fig:5}
\end{figure*}
 \begin{figure*}[!htbp]
\centering
\subfigure[Sample size $n=5\times 10^4$ and $s^*=20$ \label{fig7a}]{
\includegraphics[width=0.31\textwidth]{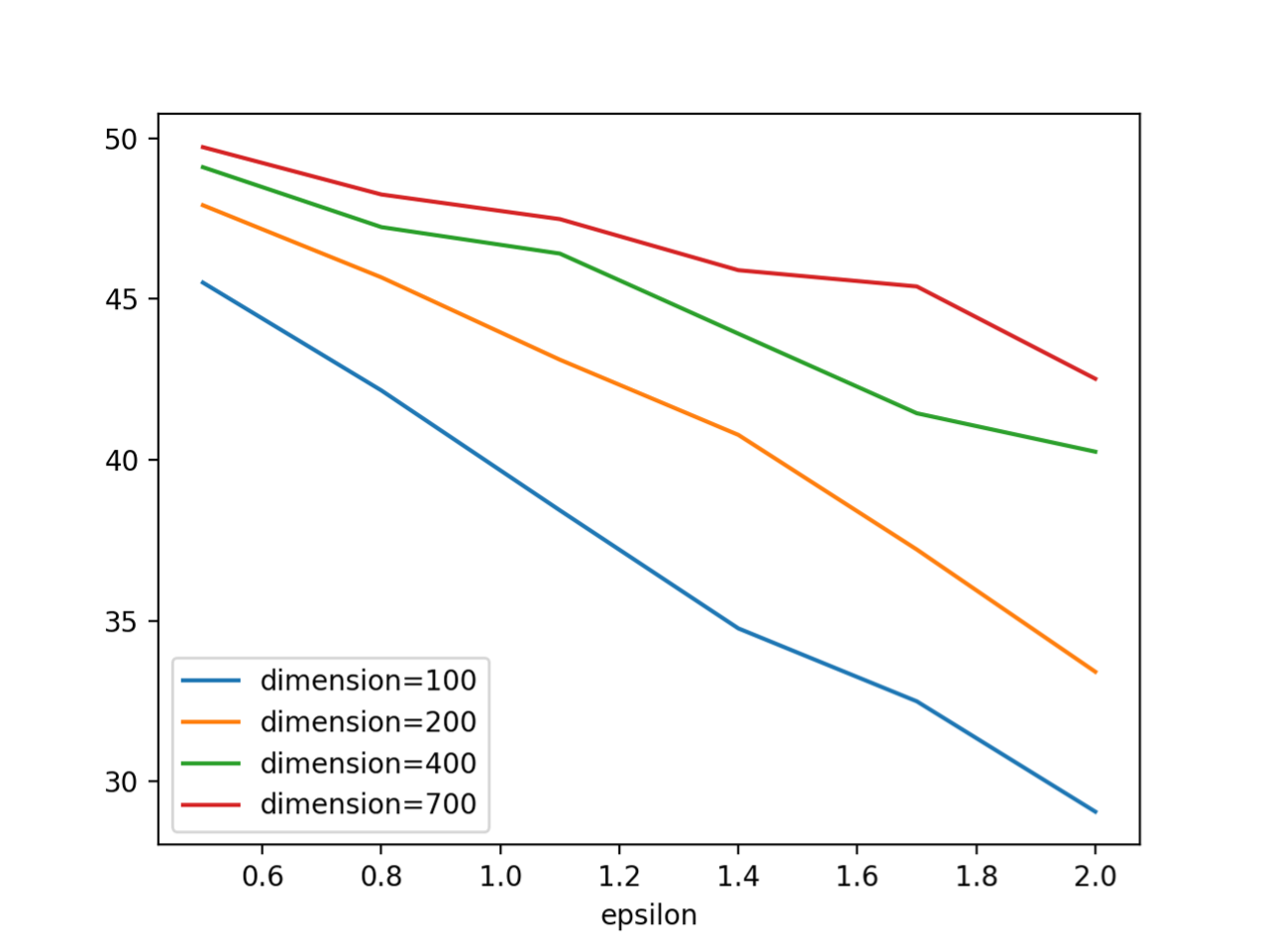} 
}
\subfigure[$\epsilon=1$ and $s^*=20$ \label{fig7b}]{
\includegraphics[width=0.31\textwidth]{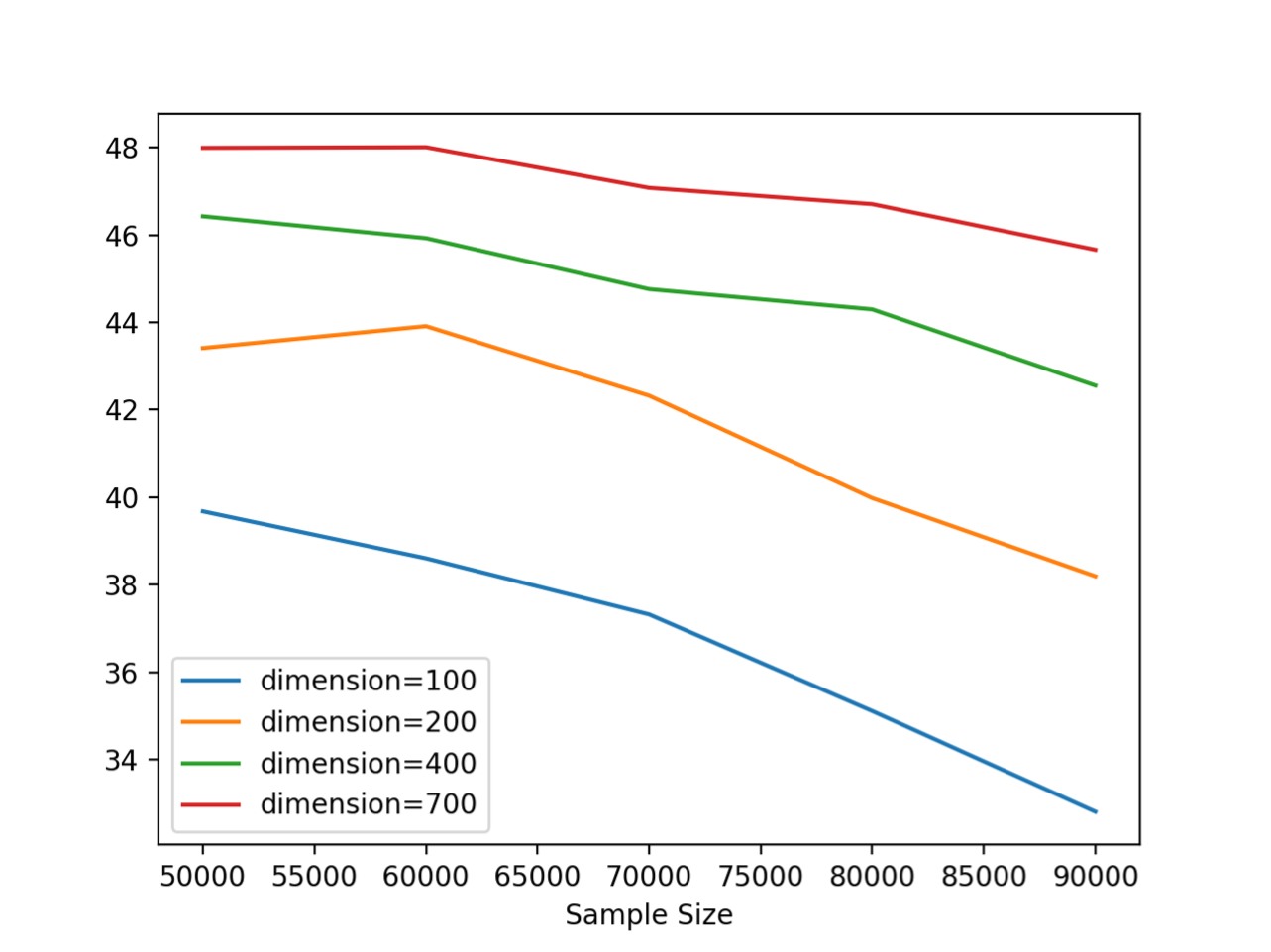} 
}
\subfigure[$\epsilon=1$ and $n=5\times 10^4$ \label{fig7c}]{
\includegraphics[width=0.31\textwidth]{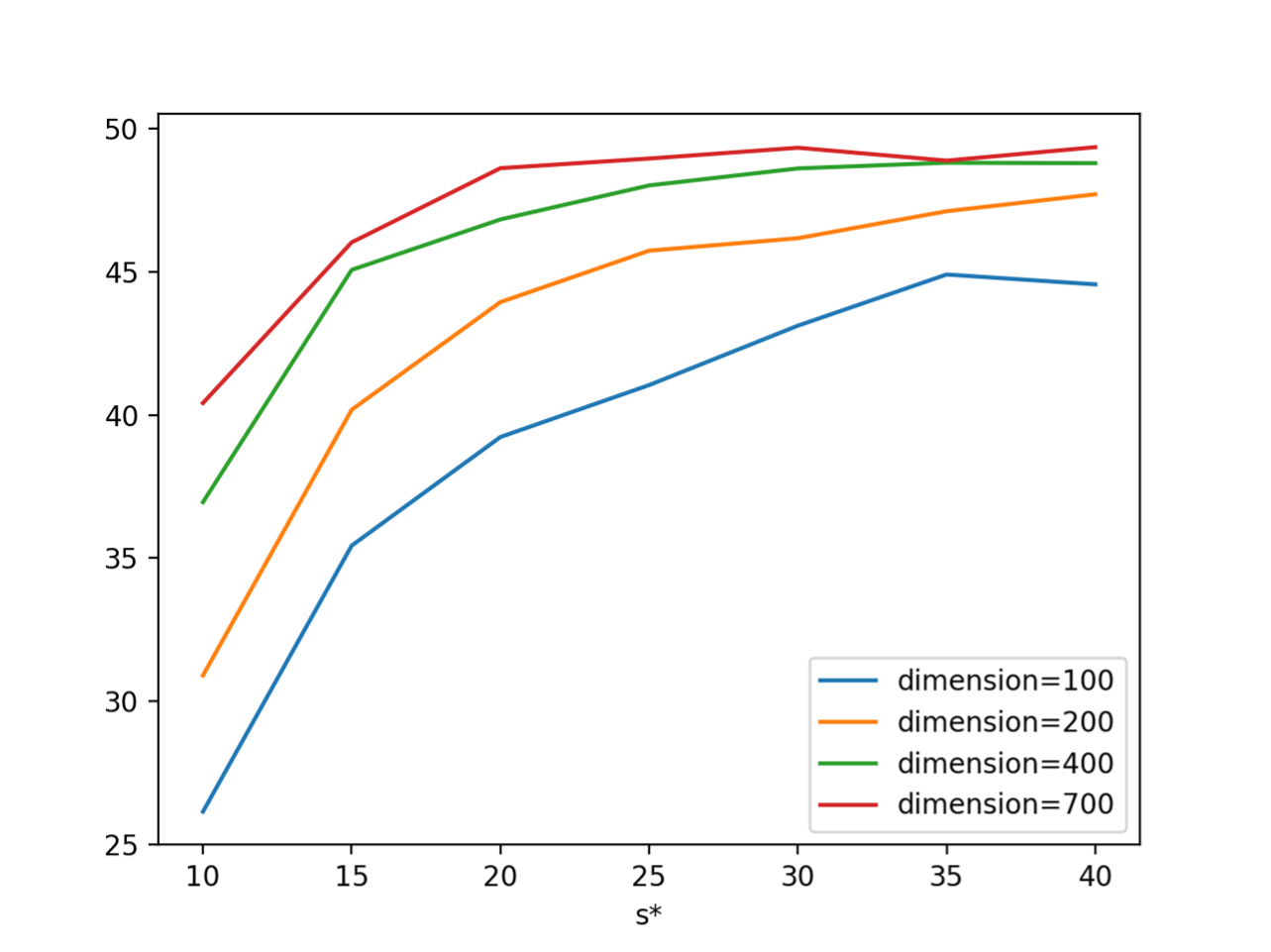} 
}
\caption{ Results of Algorithm \ref{alg:3} for linear regression with $x$ sampled from Gaussian distribution and noise sampled from log-logistic distribution.}
\label{fig:7}
\end{figure*}

 \begin{figure*}[!htbp]
\centering
\subfigure[Sample size $n=5\times 10^4$ and $s^*=20$ \label{fig8a}]{
\includegraphics[width=0.31\textwidth]{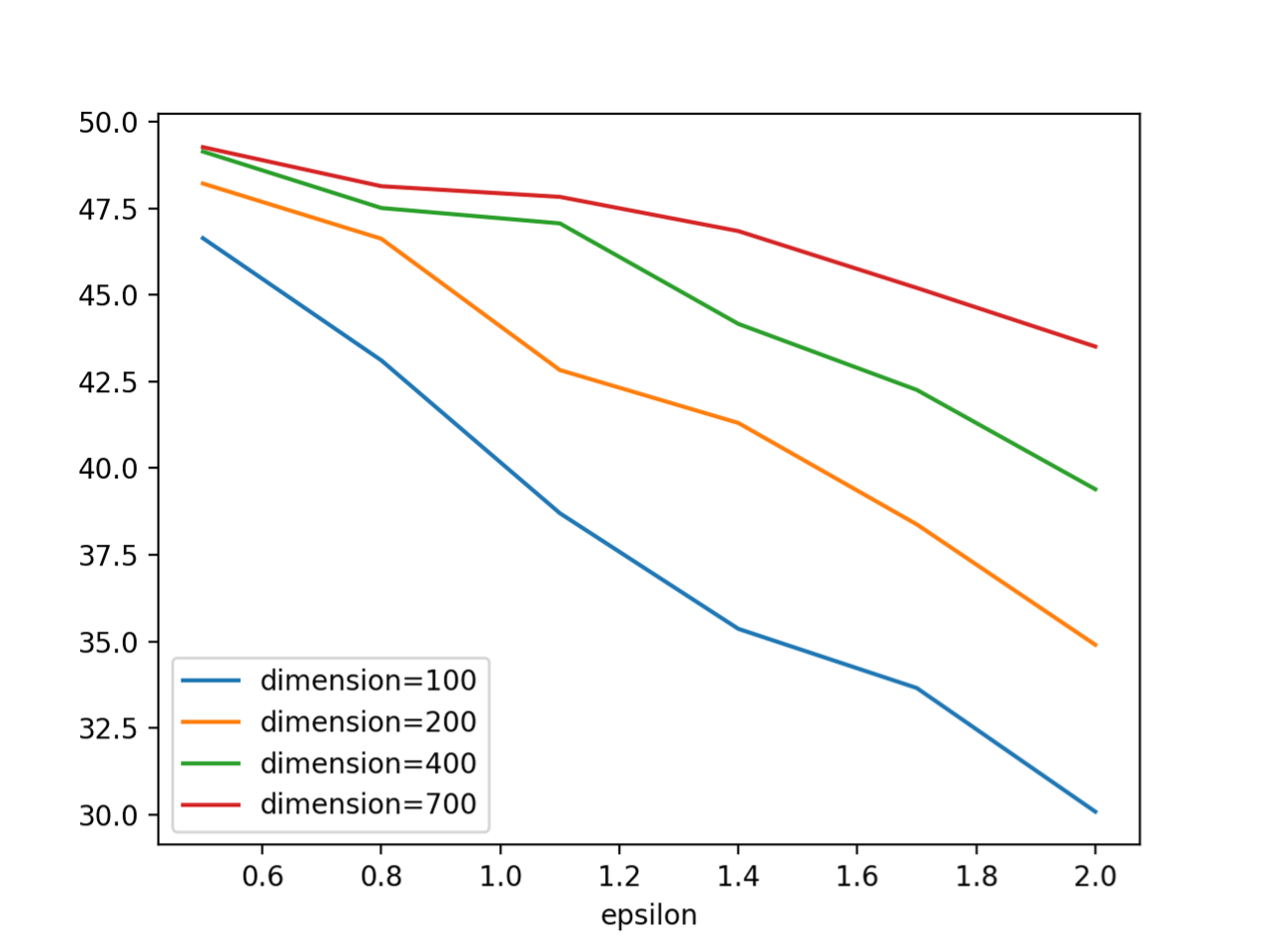} 
}
\subfigure[$\epsilon=1$ and $s^*=20$ \label{fig8b}]{
\includegraphics[width=0.31\textwidth]{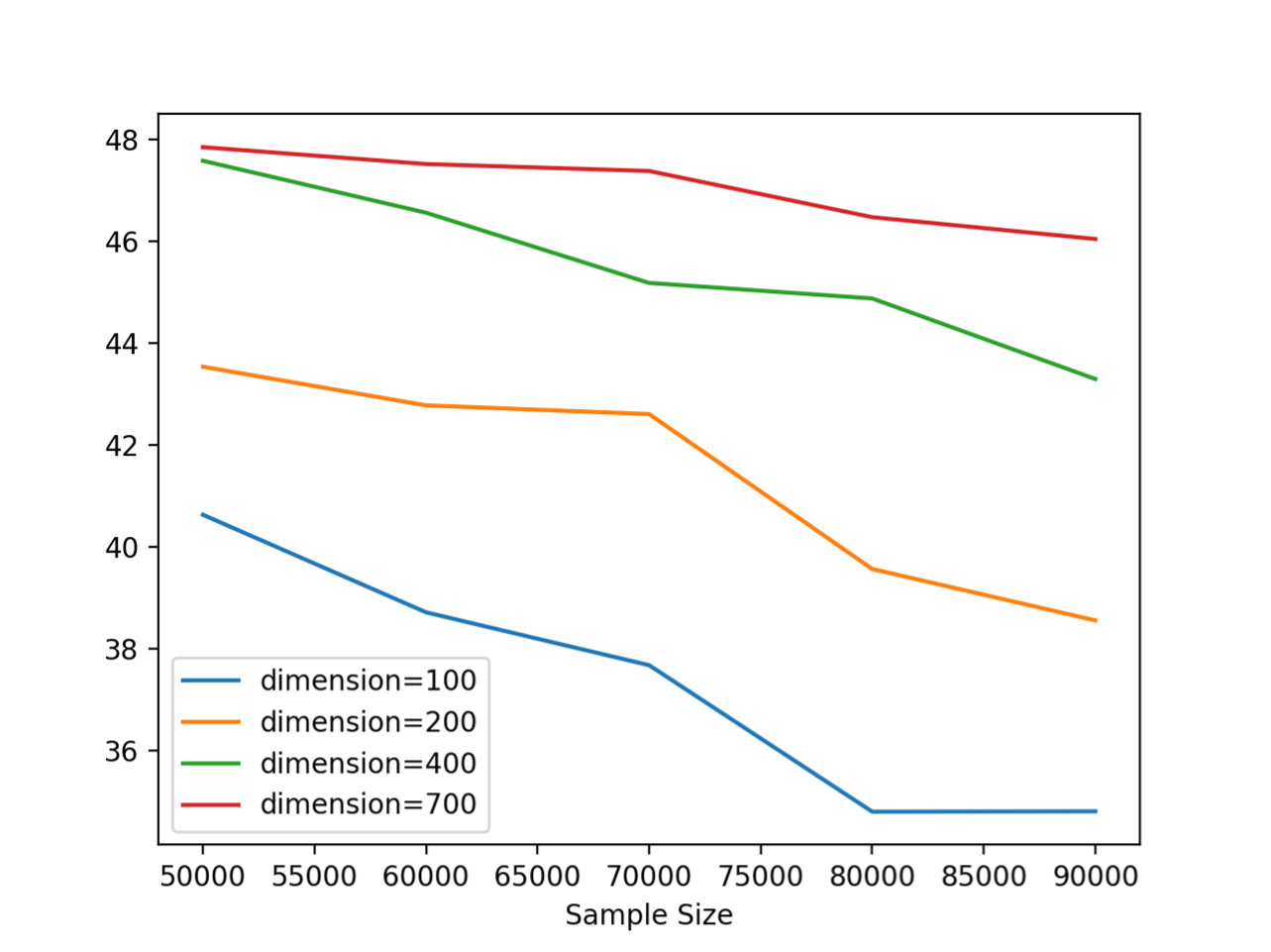} 
}
\subfigure[$\epsilon=1$ and $n=5\times 10^4$ \label{fig8c}]{
\includegraphics[width=0.31\textwidth]{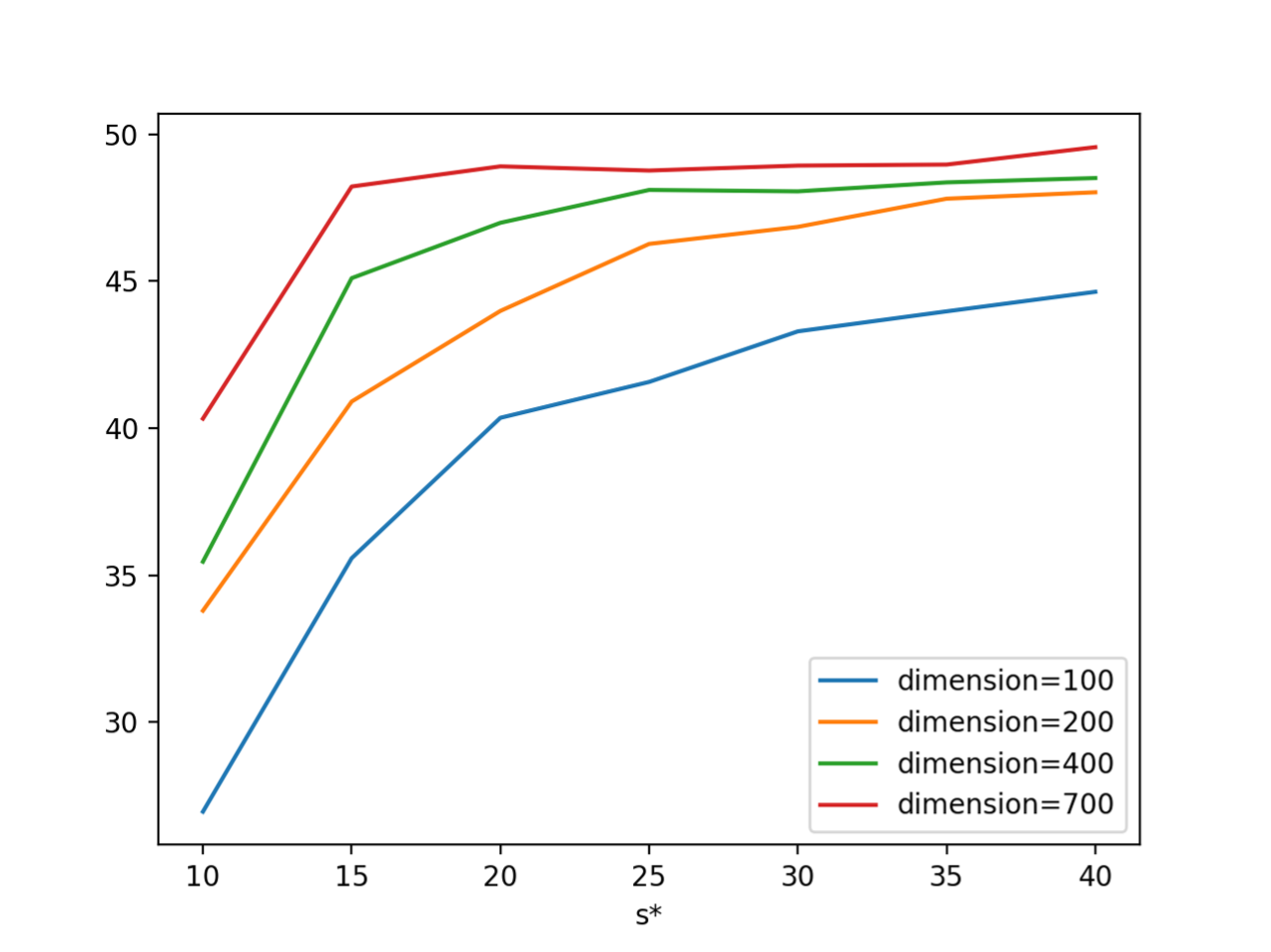} 
}
\caption{ Results of Algorithm \ref{alg:3} for linear regression with $x$ sampled from Gaussian distribution and noise sampled from log-gamma distribution.}
\label{fig:8}
\end{figure*}

From Figure \ref{fig:3} and \ref{fig:4} we can see that we have almost the same phenomenons as in Algorithm \ref{alg:1} in the previous subsection, which support our previous theoretical analysis. However, there are still some differences: First, compared with the results in Figure \ref{fig:1} we can see the trends in Figure \ref{fig:3} and \ref{fig:4} are unstable and non-smooth. Secondly, we can see the errors of Figure \ref{fig:3} are greater than the errors in Figure \ref{fig:1} under the same setting, which contradict to our previous theoretical results. We conjecture the main reason is that although the Algorithm \ref{alg:2} is better than Algorithm \ref{alg:1} theoretically, the hidden constant might be quite large and the sample size is not large enough. We leave it as an open problem to design more practical algorithms.

\subsection{Results of Algorithm \ref{alg:3} and  Algorithm \ref{alg:5}} 
We then study the practical behaviors of Algorithm \ref{alg:3}. Unlike the previous subsections, in Figure \ref{fig:5} we focus on the sparse linear model where $x$ is sampled from $\mathcal{N}(0, 5)$ and the noise is sampled from $\text{Lognormal}(0, 0.5)$. Given fixed $n=5\times 10^4$ and sparsity $s^*=20$, Figure \ref{fig5a} reveals the relation between the error and the privacy $\epsilon$ with different dimensions. Figure \ref{fig5b} shows the results of the error w.r.t various sample sizes for different dimensions when $\epsilon=1$ and $s^*=20$. Furthermore, in Figure \ref{fig5c} we also investigate the error for different sparsity $s^*$ with various dimensions under $\epsilon=1$ and $n=5\times 10^4$. In Figure \ref{fig:7} we studied the setting where the feature vector $x\sim \mathcal{N}(0, 5)$ and the noise sampled from the log-logistic distribution with $c=0.1$ (note that the PDF of  log-logistic distribution with $c$  is $p(w)=cw^{-c-1}(1+w^{-c})^{-2}$). And Figure \ref{fig:8} considers the case where $x\sim \mathcal{N}(0, 5)$ and the noise sampled from the log-gamma distribution with $c=0.5$ (note that the PDF of  log-gamma distribution with $c$  is $p(w)=\frac{\exp(cw)-\exp(w)}{\Gamma(c)}$, where $\Gamma(\cdot)$ is the Gamma function).

In Figure \ref{fig:6} we study Algorithm \ref{alg:5} with the feature vector $x\sim \mathcal{N}(0, 5)$ and the noise sampled from the logistic distribution with 
$u=0$ and $s=0.5$ (note that the PDF of  logistic distribution with $(u, s)$  is $p(w)=\frac{\exp(-(w-u)/s)}{s(1+\exp(-(w-u)/s))^2}$). And in Figure \ref{fig:9} we consider the case where $x\sim \text{Laplace}(5)$ and the noise is sampled from the log-gamma distribution with $c=0.5$.

From the above results we can see that since the upper bound of the error depends on the dimension logarithmically, increasing the dimension will just lightly affect the error. Moreover, when the sample size or the privacy parameter increases, the error will decrease, which is the same as in the previous parts. Finally, when the underlying sparsity $s^*$ is larger, the error will  become larger, and unlike the dimension, the dependence with the sparsity is not logarithmic. Thus, the sparsity will heavily affect the error. All of our results support the previous theoretical analysis. 

 \begin{figure*}[!htbp]
\centering
\subfigure[Sample size $n=8000$ and $s^*=20$ \label{fig6a}]{
\includegraphics[width=0.31\textwidth]{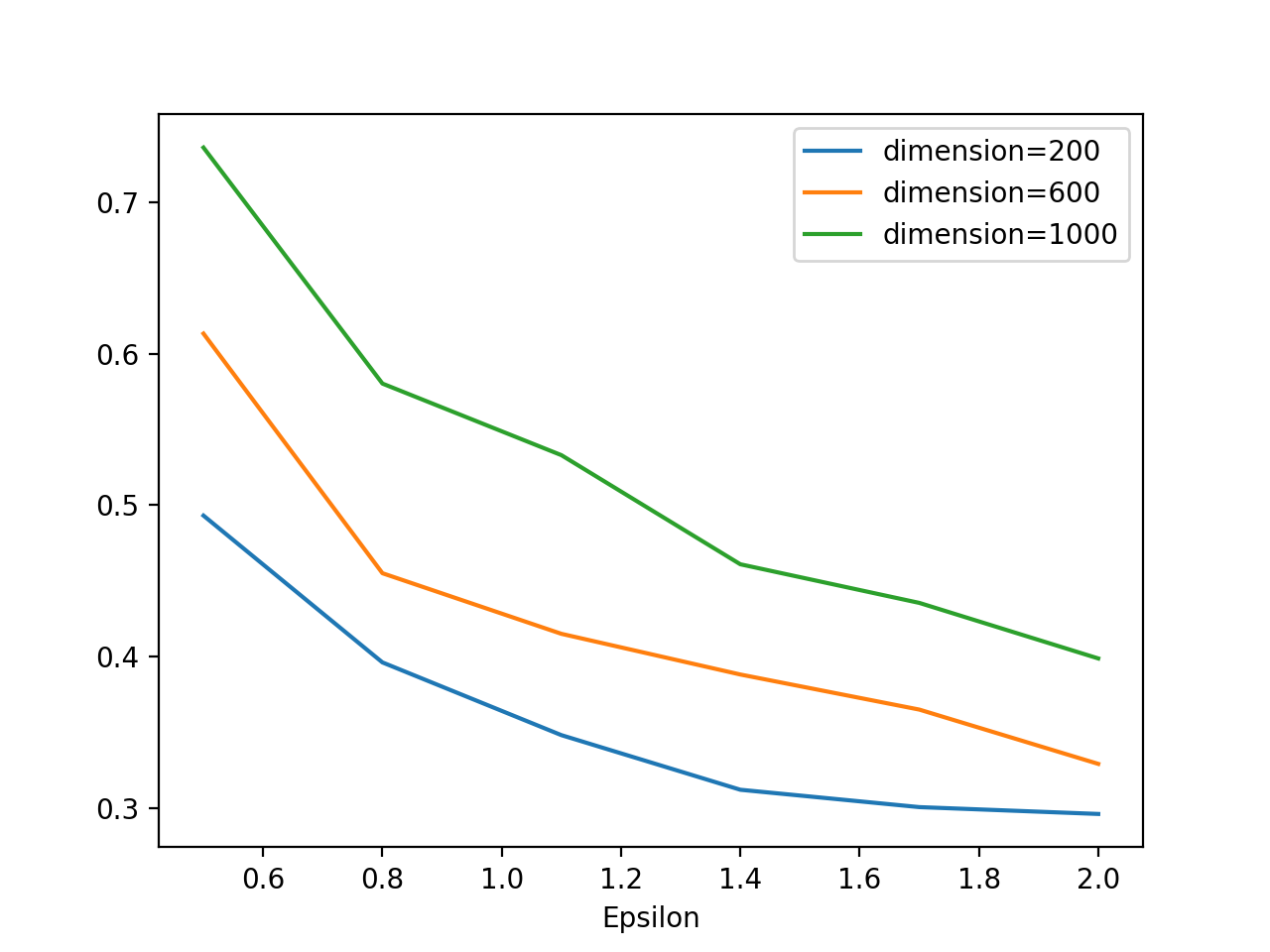} 
}
\subfigure[$\epsilon=1$ and $s^*=20$ \label{fig6b}]{
\includegraphics[width=0.31\textwidth]{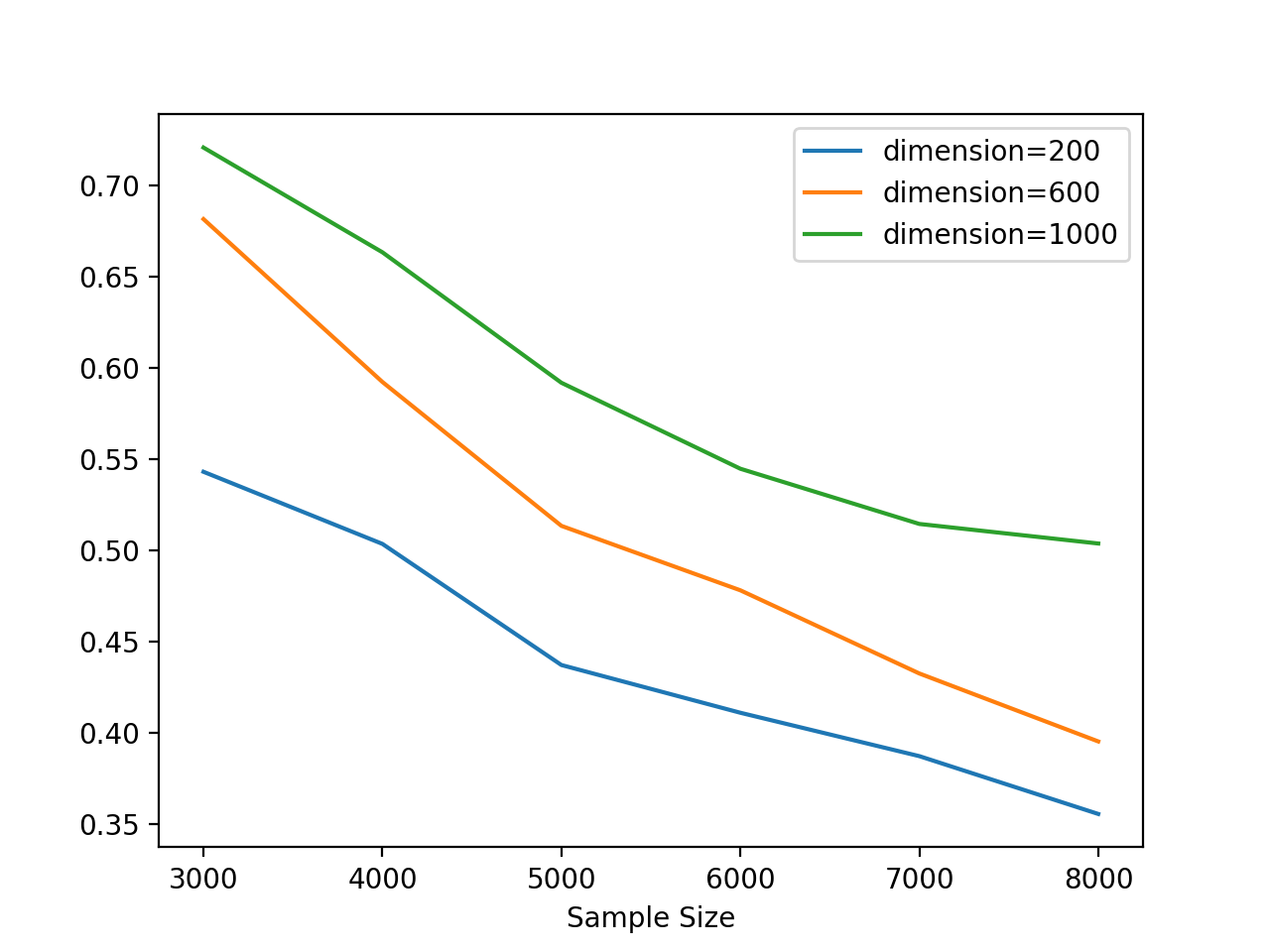} 
}
\subfigure[$\epsilon=1$ and $n=8000$ \label{fig6c}]{
\includegraphics[width=0.31\textwidth]{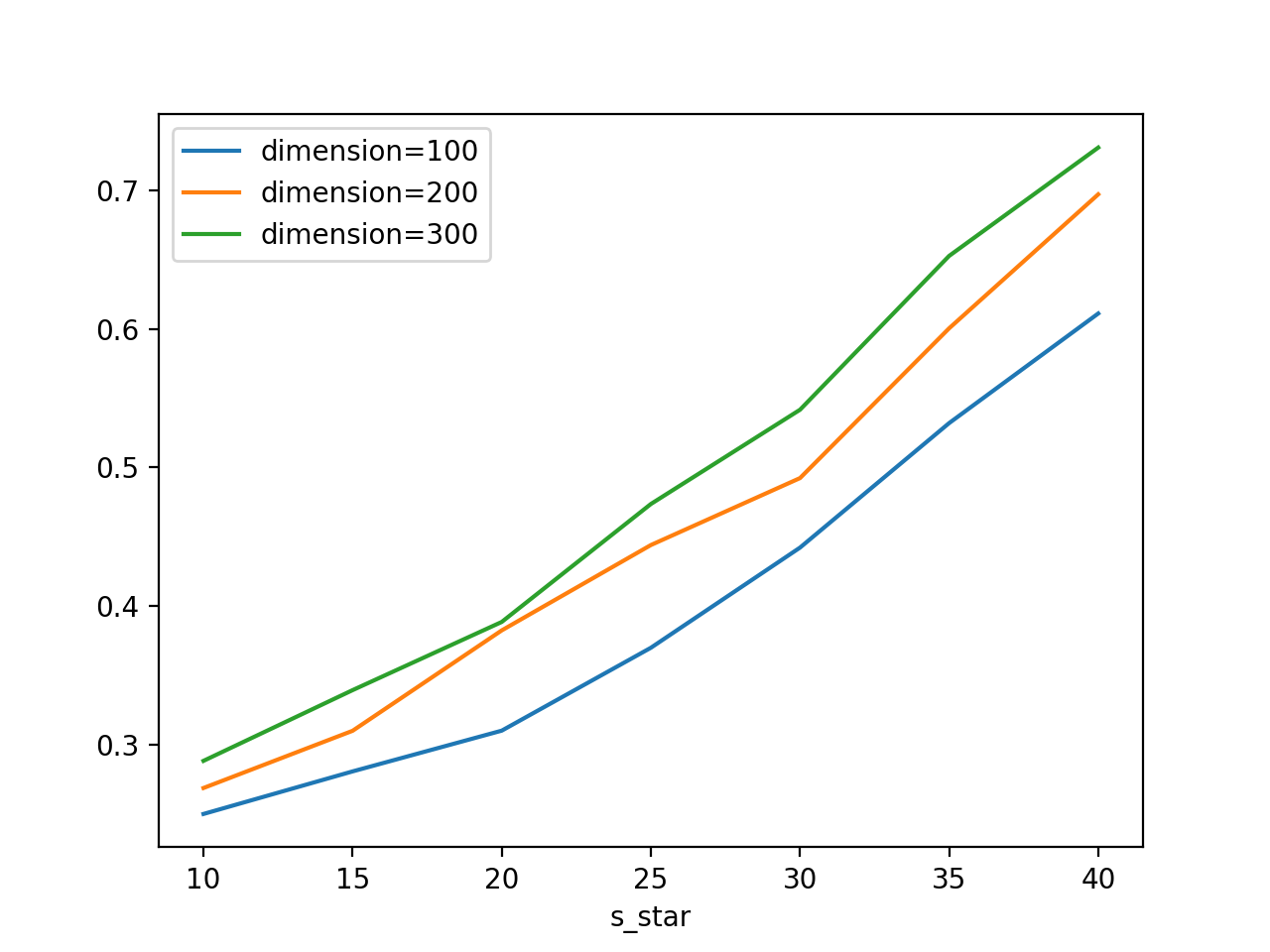} 
}
\caption{ Results of Algorithm \ref{alg:5} for regularized logistic regression with $x$ sampled from Gaussian distribution and noise sampled from log-normal distribution.}
\label{fig:6}
\end{figure*}
 \begin{figure*}[!htbp]
\centering
\subfigure[Sample size $n=8000$ and $s^*=20$ \label{fig9a}]{
\includegraphics[width=0.31\textwidth, height=0.14\textheight]{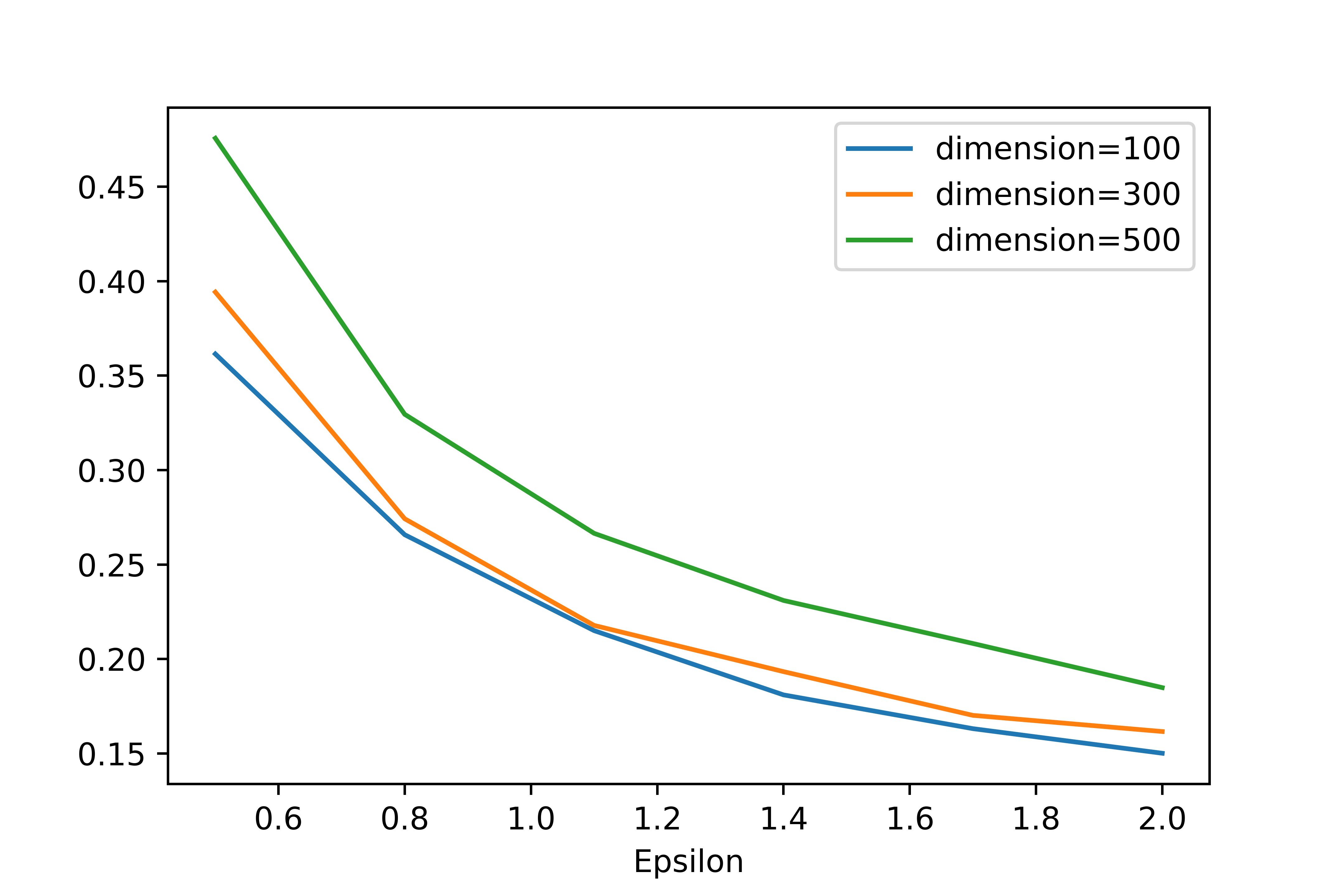} 
}
\subfigure[$\epsilon=1$ and $s^*=20$ \label{fig9b}]{
\includegraphics[width=0.31\textwidth]{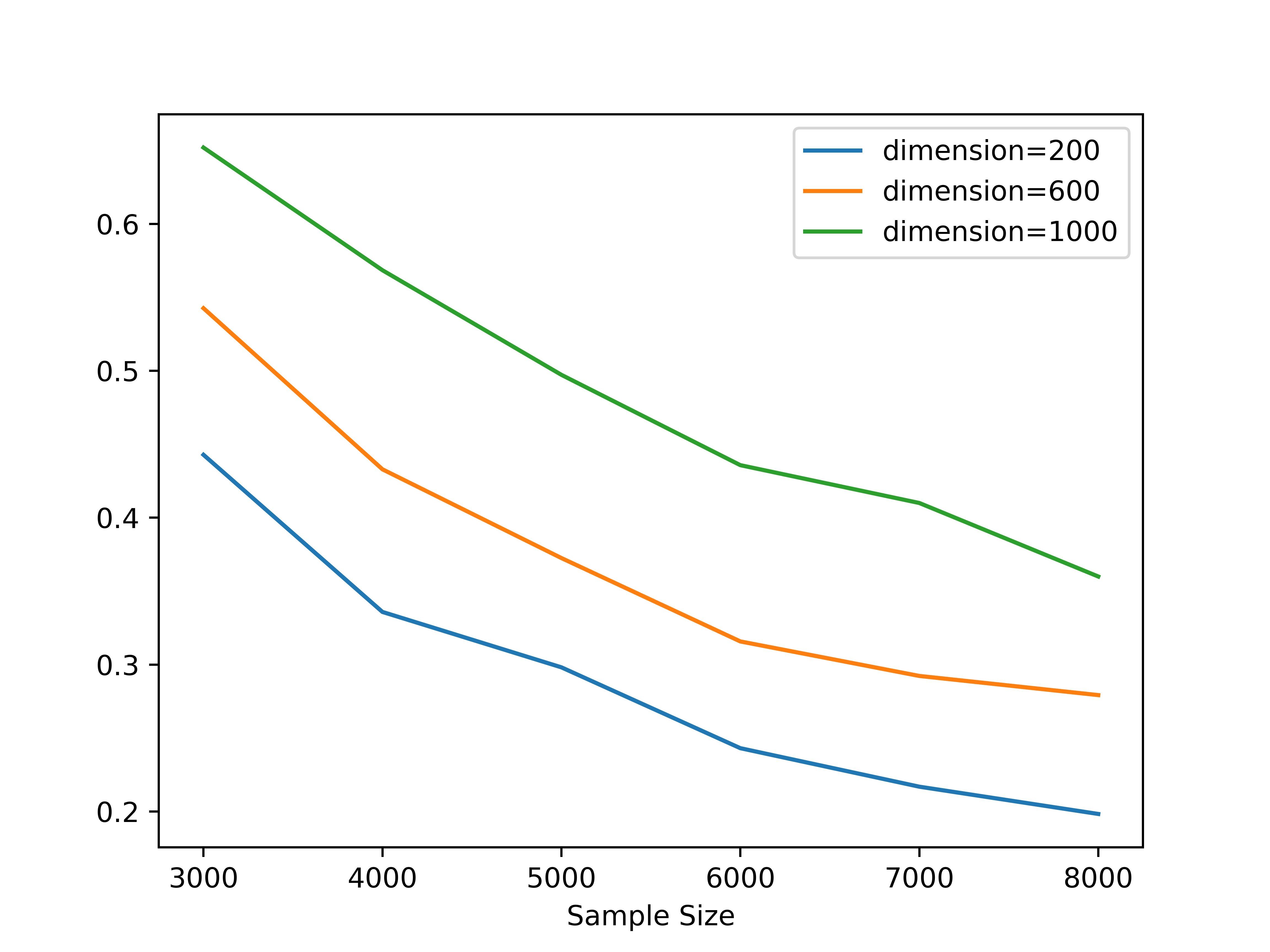} 
}
\subfigure[$\epsilon=1$ and $n=8000$ \label{fig9c}]{
\includegraphics[width=0.31\textwidth, height=0.14\textheight]{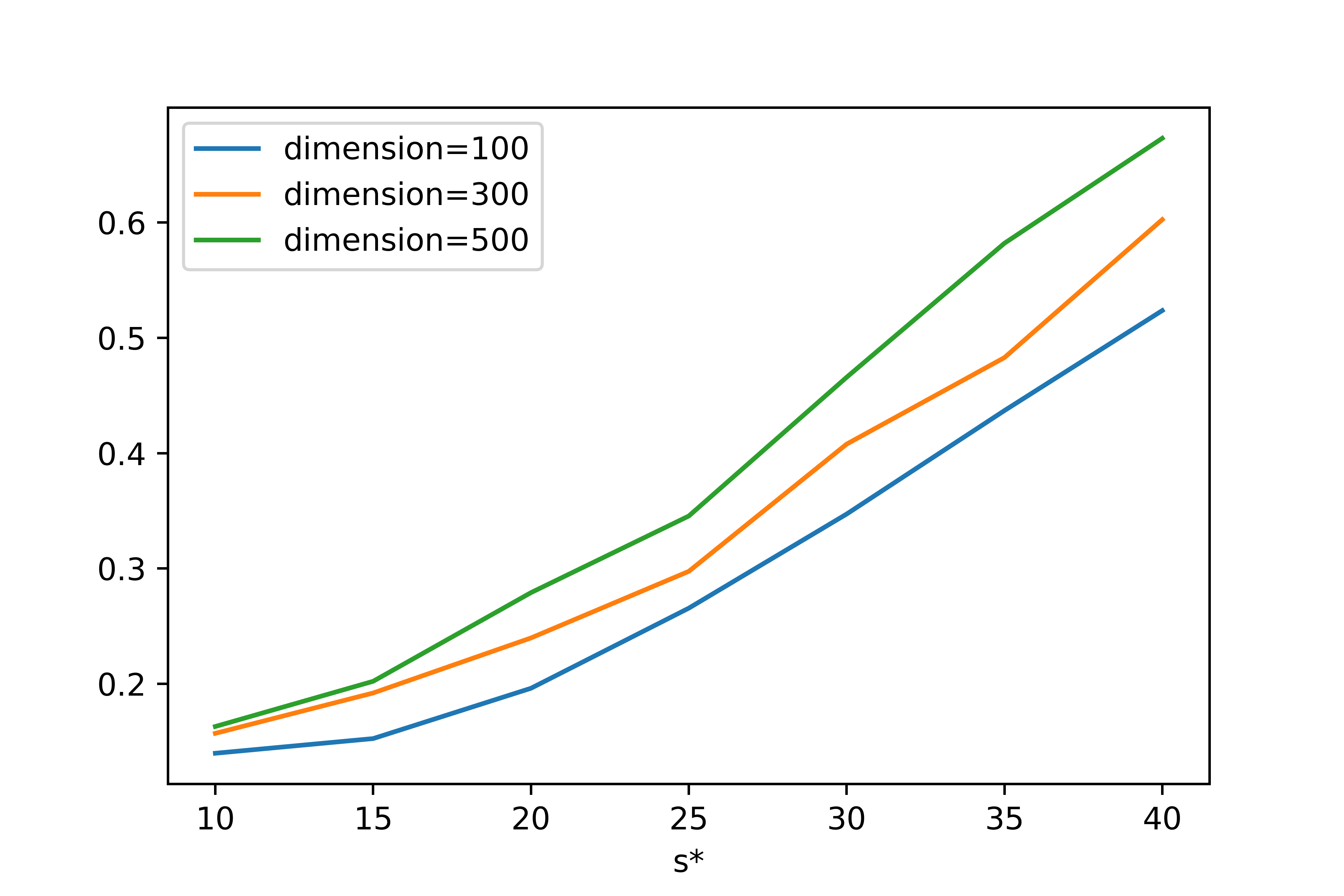} 
}
\caption{ Results of Algorithm \ref{alg:5} for regularized logistic regression with $x$ sampled from Laplacian distribution and noise sampled from log-gamma distribution.}
\label{fig:9}
\end{figure*}
\section{Conclusion}
In this paper, we studied the problem of Differentially Private Stochastic Convex Optimization (DP-SCO) in the high dimensional (sparse) setting, where the sample size $n$ is far less than the dimension of the space $d$ and the underlying data distribution may be heavy-tailed. We first considered the problem of DP-SCO where the constraint set is some polytope. We showed that if the gradient of loss function has bounded second order moment, then it is possible to achieve an excess population risk of $\tilde{O}(\frac{\log d}{(n\epsilon)^\frac{1}{3}})$ (with high probability) in the $\epsilon$-DP model, if we omit other terms. Moreover, for the LASSO problem, we showed that it is possible to achieve an error of $\tilde{O}(\frac{\log d}{(n\epsilon)^\frac{2}{5}})$ in the $(\epsilon, \delta)$-DP model. Next we studied DP-SCO for sparse learning with heavy-tailed data. We first investigated the sparse linear model and proposed a method whose output could achieve an estimation error of $\tilde{O}(\frac{s^{*2}\log^2 d}{n\epsilon})$, where $s^*$ is the sparsity of the underlying parameter. Then we studied a more general problem over the sparsity ({\em i.e.,} $\ell_0$-norm) constraint, and show that it is possible to achieve an error of $\tilde{O}(\frac{s^{*\frac{3}{2}}\log d}{n\epsilon})$ if the loss function is smooth and strongly convex. Finally, we showed a lower bound of $\tilde{O}(\frac{s^{*}\log d}{n\epsilon})$  for the high dimensional heavy-tailed sparse mean estimation in the $(\epsilon, \delta)$-DP model.

Besides the open problems we mentioned in the previous sections, there are still many other future work. First, in this paper, we studied the problem under various settings and assumptions and provided some bounds of the excess population risk. While we showed a lower bound for the high dimensional heavy-tailed sparse mean  problem, we still do not know the lower bounds of other problems. Previous results on the lower bounds need to assume the data is regular, thus we need new techniques or hard instances to get those lower bounds in the heavy-tailed setting. Secondly, in the heavy-tailed and  low dimensional case, we know that the bounds of excess population risk may be different in the high probability form and expectation form \cite{kamath2021improved,wang2020differentially}. Thus, our question is, in the high dimensional case, if we relax to the expectation form, can we further improve these upper bounds? Thirdly, we need to assume the gradient of the loss has bounded second order moment throughout the paper. However, sometimes this will not be held and the data may only has the $1+v$-th moment with some $v\in (0, 1)$ \cite{tao2021optimal}. Due to this weaker assumption, all the previous methods are failed. Thus, how to extend to this case in both low dimensional and high dimensional cases?

\section*{Acknowledgements}
 Di Wang and Lijie Hu were support in part by the baseline funding BAS/1/1689-01-01 and funding from the AI Initiative REI/1/4811-10-01 of King Abdullah University of Science and Technology (KAUST).
\bibliographystyle{plain}
\bibliography{nips}

\begin{thebibliography}{10}

\bibitem{abadi2016deep}
Martin Abadi, Andy Chu, Ian Goodfellow, H~Brendan McMahan, Ilya Mironov, Kunal
  Talwar, and Li~Zhang.
\newblock Deep learning with differential privacy.
\newblock In {\em Proceedings of the 2016 ACM SIGSAC Conference on Computer and
  Communications Security}, pages 308--318. ACM, 2016.

\bibitem{asi2021private}
Hilal Asi, Vitaly Feldman, Tomer Koren, and Kunal Talwar.
\newblock Private stochastic convex optimization: Optimal rates in $\ell_1$
  geometry.
\newblock {\em arXiv preprint arXiv:2103.01516}, 2021.

\bibitem{bahmani2013greedy}
Sohail Bahmani, Bhiksha Raj, and Petros~T Boufounos.
\newblock Greedy sparsity-constrained optimization.
\newblock {\em Journal of Machine Learning Research}, 14(Mar):807--841, 2013.

\bibitem{barber2014privacy}
Rina~Foygel Barber and John~C Duchi.
\newblock Privacy and statistical risk: Formalisms and minimax bounds.
\newblock {\em arXiv preprint arXiv:1412.4451}, 2014.

\bibitem{bassily2020stability}
Raef Bassily, Vitaly Feldman, Crist{\'o}bal Guzm{\'a}n, and Kunal Talwar.
\newblock Stability of stochastic gradient descent on nonsmooth convex losses.
\newblock {\em Advances in Neural Information Processing Systems}, 33, 2020.

\bibitem{bassily2020}
Raef Bassily, Vitaly Feldman, Kunal Talwar, and Abhradee Thakurta.
\newblock Private stochastic convex optimization with optimal rates.
\newblock In {\em NeurIPS}, 2019.

\bibitem{bassily2014private}
Raef Bassily, Adam Smith, and Abhradeep Thakurta.
\newblock Private empirical risk minimization: Efficient algorithms and tight
  error bounds.
\newblock In {\em Foundations of Computer Science (FOCS), 2014 IEEE 55th Annual
  Symposium on}, pages 464--473. IEEE, 2014.

\bibitem{biswas2007statistical}
Atanu Biswas, Sujay Datta, Jason~P Fine, and Mark~R Segal.
\newblock {\em Statistical advances in the biomedical science}.
\newblock Wiley Online Library, 2007.

\bibitem{brownlees2015empirical}
Christian Brownlees, Emilien Joly, G{\'a}bor Lugosi, et~al.
\newblock Empirical risk minimization for heavy-tailed losses.
\newblock {\em The Annals of Statistics}, 43(6):2507--2536, 2015.

\bibitem{brunel2020propose}
Victor-Emmanuel Brunel and Marco Avella-Medina.
\newblock Propose, test, release: Differentially private estimation with high
  probability.
\newblock {\em arXiv preprint arXiv:2002.08774}, 2020.

\bibitem{bubeck2013bandits}
S{\'e}bastien Bubeck, Nicolo Cesa-Bianchi, and G{\'a}bor Lugosi.
\newblock Bandits with heavy tail.
\newblock {\em IEEE Transactions on Information Theory}, 59(11):7711--7717,
  2013.

\bibitem{bun2019average}
Mark Bun and Thomas Steinke.
\newblock Average-case averages: Private algorithms for smooth sensitivity and
  mean estimation.
\newblock {\em arXiv preprint arXiv:1906.02830}, 2019.

\bibitem{cai2019cost}
T~Tony Cai, Yichen Wang, and Linjun Zhang.
\newblock The cost of privacy: Optimal rates of convergence for parameter
  estimation with differential privacy.
\newblock {\em arXiv preprint arXiv:1902.04495}, 2019.

\bibitem{cai2020cost}
T~Tony Cai, Yichen Wang, and Linjun Zhang.
\newblock The cost of privacy in generalized linear models: Algorithms and
  minimax lower bounds.
\newblock {\em arXiv preprint arXiv:2011.03900}, 2020.

\bibitem{catoni2004statistical}
Olivier Catoni.
\newblock {\em Statistical learning theory and stochastic optimization: Ecole
  d'Et{\'e} de Probabilit{\'e}s de Saint-Flour, XXXI-2001}, volume 1851.
\newblock Springer Science \& Business Media, 2004.

\bibitem{catoni2012challenging}
Olivier Catoni.
\newblock Challenging the empirical mean and empirical variance: a deviation
  study.
\newblock In {\em Annales de l'IHP Probabilit{\'e}s et statistiques},
  volume~48, pages 1148--1185, 2012.

\bibitem{catoni2017dimension}
Olivier Catoni and Ilaria Giulini.
\newblock Dimension-free pac-bayesian bounds for matrices, vectors, and linear
  least squares regression.
\newblock {\em arXiv preprint arXiv:1712.02747}, 2017.

\bibitem{chaudhuri2009privacy}
Kamalika Chaudhuri and Claire Monteleoni.
\newblock Privacy-preserving logistic regression.
\newblock In {\em Advances in neural information processing systems}, pages
  289--296, 2009.

\bibitem{chaudhuri2011differentially}
Kamalika Chaudhuri, Claire Monteleoni, and Anand~D Sarwate.
\newblock Differentially private empirical risk minimization.
\newblock {\em Journal of Machine Learning Research}, 12(Mar):1069--1109, 2011.

\bibitem{ding2017collecting}
Bolin Ding, Janardhan Kulkarni, and Sergey Yekhanin.
\newblock Collecting telemetry data privately.
\newblock In {\em Advances in Neural Information Processing Systems}, pages
  3571--3580, 2017.

\bibitem{Dua:2019}
Dheeru Dua and Casey Graff.
\newblock {UCI} machine learning repository, 2017.

\bibitem{duchi2013local}
John~C Duchi, Michael~I Jordan, and Martin~J Wainwright.
\newblock Local privacy and statistical minimax rates.
\newblock In {\em 2013 IEEE 54th Annual Symposium on Foundations of Computer
  Science}, pages 429--438. IEEE, 2013.

\bibitem{duchi2018minimax}
John~C Duchi, Michael~I Jordan, and Martin~J Wainwright.
\newblock Minimax optimal procedures for locally private estimation.
\newblock {\em Journal of the American Statistical Association},
  113(521):182--201, 2018.

\bibitem{dwork2009differential}
Cynthia Dwork and Jing Lei.
\newblock Differential privacy and robust statistics.
\newblock In {\em Proceedings of the forty-first annual ACM symposium on Theory
  of computing}, pages 371--380. ACM, 2009.

\bibitem{dwork2006calibrating}
Cynthia Dwork, Frank McSherry, Kobbi Nissim, and Adam Smith.
\newblock Calibrating noise to sensitivity in private data analysis.
\newblock In {\em Theory of cryptography conference}, pages 265--284. Springer,
  2006.

\bibitem{dwork2014algorithmic}
Cynthia Dwork, Aaron Roth, et~al.
\newblock The algorithmic foundations of differential privacy.
\newblock {\em Foundations and Trends in Theoretical Computer Science},
  9(3-4):211--407, 2014.

\bibitem{fan2016shrinkage}
Jianqing Fan, Weichen Wang, and Ziwei Zhu.
\newblock A shrinkage principle for heavy-tailed data: High-dimensional robust
  low-rank matrix recovery.
\newblock {\em arXiv preprint arXiv:1603.08315}, 2016.

\bibitem{feldman2020private}
Vitaly Feldman, Tomer Koren, and Kunal Talwar.
\newblock Private stochastic convex optimization: optimal rates in linear time.
\newblock In {\em Proceedings of the 52nd Annual ACM SIGACT Symposium on Theory
  of Computing}, pages 439--449, 2020.

\bibitem{holland2019better}
Matthew Holland and Kazushi Ikeda.
\newblock Better generalization with less data using robust gradient descent.
\newblock In {\em International Conference on Machine Learning}, pages
  2761--2770, 2019.

\bibitem{holland2019a}
Matthew~J Holland.
\newblock Robust descent using smoothed multiplicative noise.
\newblock In {\em 22nd International Conference on Artificial Intelligence and
  Statistics (AISTATS)}, volume~89 of {\em Proceedings of Machine Learning
  Research}, pages 703--711, 2019.

\bibitem{hsu2016loss}
Daniel Hsu and Sivan Sabato.
\newblock Loss minimization and parameter estimation with heavy tails.
\newblock {\em The Journal of Machine Learning Research}, 17(1):543--582, 2016.

\bibitem{ibragimov2015heavy}
Marat Ibragimov, Rustam Ibragimov, and Johan Walden.
\newblock {\em Heavy-tailed distributions and robustness in economics and
  finance}, volume 214.
\newblock Springer, 2015.

\bibitem{iyengar2019towards}
Roger Iyengar, Joseph~P Near, Dawn Song, Om~Thakkar, Abhradeep Thakurta, and
  Lun Wang.
\newblock Towards practical differentially private convex optimization.
\newblock In {\em 2019 IEEE Symposium on Security and Privacy (SP)}, pages
  299--316. IEEE, 2019.

\bibitem{jaggi2013revisiting}
Martin Jaggi.
\newblock Revisiting frank-wolfe: Projection-free sparse convex optimization.
\newblock In {\em International Conference on Machine Learning}, pages
  427--435. PMLR, 2013.

\bibitem{jain2014iterative}
Prateek Jain, Ambuj Tewari, and Purushottam Kar.
\newblock On iterative hard thresholding methods for high-dimensional
  m-estimation.
\newblock {\em arXiv preprint arXiv:1410.5137}, 2014.

\bibitem{kamath2021improved}
Gautam Kamath, Xingtu Liu, and Huanyu Zhang.
\newblock Improved rates for differentially private stochastic convex
  optimization with heavy-tailed data.
\newblock {\em arXiv preprint arXiv:2106.01336}, 2021.

\bibitem{kamath2020private}
Gautam Kamath, Vikrant Singhal, and Jonathan Ullman.
\newblock Private mean estimation of heavy-tailed distributions.
\newblock In {\em Conference on Learning Theory}, pages 2204--2235. PMLR, 2020.

\bibitem{kasiviswanathan2016efficient}
Shiva~Prasad Kasiviswanathan and Hongxia Jin.
\newblock Efficient private empirical risk minimization for high-dimensional
  learning.
\newblock In {\em International Conference on Machine Learning}, pages
  488--497, 2016.

\bibitem{kifer2012private}
Daniel Kifer, Adam Smith, and Abhradeep Thakurta.
\newblock Private convex empirical risk minimization and high-dimensional
  regression.
\newblock In {\em Conference on Learning Theory}, pages 25--1, 2012.

\bibitem{lecue2018robust}
Guillaume Lecu{\'e}, Matthieu Lerasle, and Timoth{\'e}e Mathieu.
\newblock Robust classification via mom minimization.
\newblock {\em arXiv preprint arXiv:1808.03106}, 2018.

\bibitem{liu2021robust}
Xiyang Liu, Weihao Kong, Sham Kakade, and Sewoong Oh.
\newblock Robust and differentially private mean estimation.
\newblock {\em arXiv preprint arXiv:2102.09159}, 2021.

\bibitem{loh2013regularized}
Po-Ling Loh and Martin~J Wainwright.
\newblock Regularized m-estimators with nonconvexity: Statistical and
  algorithmic theory for local optima.
\newblock In {\em Advances in Neural Information Processing Systems}, pages
  476--484, 2013.

\bibitem{lugosi2019risk}
G{\'a}bor Lugosi and Shahar Mendelson.
\newblock Risk minimization by median-of-means tournaments.
\newblock {\em Journal of the European Mathematical Society}, 2019.

\bibitem{minsker2015geometric}
Stanislav Minsker et~al.
\newblock Geometric median and robust estimation in banach spaces.
\newblock {\em Bernoulli}, 21(4):2308--2335, 2015.

\bibitem{nissim2007smooth}
Kobbi Nissim, Sofya Raskhodnikova, and Adam Smith.
\newblock Smooth sensitivity and sampling in private data analysis.
\newblock In {\em Proceedings of the thirty-ninth annual ACM symposium on
  Theory of computing}, pages 75--84. ACM, 2007.

\bibitem{prasad2018robust}
Adarsh Prasad, Arun~Sai Suggala, Sivaraman Balakrishnan, and Pradeep Ravikumar.
\newblock Robust estimation via robust gradient estimation.
\newblock {\em arXiv preprint arXiv:1802.06485}, 2018.

\bibitem{raskutti2011minimax}
Garvesh Raskutti, Martin~J Wainwright, and Bin Yu.
\newblock Minimax rates of estimation for high-dimensional linear regression
  over $\ell_q$-balls.
\newblock {\em IEEE transactions on information theory}, 57(10):6976--6994,
  2011.

\bibitem{smith2017interaction}
Adam Smith, Abhradeep Thakurta, and Jalaj Upadhyay.
\newblock Is interaction necessary for distributed private learning?
\newblock In {\em 2017 IEEE Symposium on Security and Privacy (SP)}, pages
  58--77. IEEE, 2017.

\bibitem{song2020characterizing}
Shuang Song, Om~Thakkar, and Abhradeep Thakurta.
\newblock Characterizing private clipped gradient descent on convex generalized
  linear problems.
\newblock {\em arXiv preprint arXiv:2006.06783}, 2020.

\bibitem{talwar2015nearly}
Kunal Talwar, Abhradeep Thakurta, and Li~Zhang.
\newblock Nearly-optimal private lasso.
\newblock In {\em Proceedings of the 28th International Conference on Neural
  Information Processing Systems-Volume 2}, pages 3025--3033, 2015.

\bibitem{apple}
Jun Tang, Aleksandra Korolova, Xiaolong Bai, Xueqiang Wang, and XiaoFeng Wang.
\newblock Privacy loss in apple's implementation of differential privacy on
  macos 10.12.
\newblock {\em CoRR}, abs/1709.02753, 2017.

\bibitem{tao2021optimal}
Youming Tao, Yulian Wu, Peng Zhao, and Di~Wang.
\newblock Optimal rates of (locally) differentially private heavy-tailed
  multi-armed bandits.
\newblock {\em arXiv preprint arXiv:2106.02575}, 2021.

\bibitem{vapnik2013nature}
Vladimir Vapnik.
\newblock {\em The nature of statistical learning theory}.
\newblock Springer science \& business media, 2013.

\bibitem{vershynin2018high}
Roman Vershynin.
\newblock {\em High-dimensional probability: An introduction with applications
  in data science}, volume~47.
\newblock Cambridge university press, 2018.

\bibitem{DBLP:journals/corr/abs-2010-13520}
Di~Wang, Jiahao Ding, Zejun Xie, Miao Pan, and Jinhui Xu.
\newblock Differentially private (gradient) expectation maximization algorithm
  with statistical guarantees.
\newblock {\em CoRR}, abs/2010.13520, 2020.

\bibitem{JMLR:v21:19-253}
Di~Wang, Marco Gaboardi, Adam Smith, and Jinhui Xu.
\newblock Empirical risk minimization in the non-interactive local model of
  differential privacy.
\newblock {\em Journal of Machine Learning Research}, 21(200):1--39, 2020.

\bibitem{wang2020differentially}
Di~Wang, Hanshen Xiao, Srini Devadas, and Jinhui Xu.
\newblock On differentially private stochastic convex optimization with
  heavy-tailed data.
\newblock {\em arXiv preprint arXiv:2010.11082}, 2020.

\bibitem{Wang019a}
Di~Wang and Jinhui Xu.
\newblock On sparse linear regression in the local differential privacy model.
\newblock In {\em {ICML}}, volume~97 of {\em Proceedings of Machine Learning
  Research}, pages 6628--6637. {PMLR}, 2019.

\bibitem{WangX21}
Di~Wang and Jinhui Xu.
\newblock On sparse linear regression in the local differential privacy model.
\newblock {\em {IEEE} Trans. Inf. Theory}, 67(2):1182--1200, 2021.

\bibitem{wang2019differentially12}
Lingxiao Wang and Quanquan Gu.
\newblock Differentially private iterative gradient hard thresholding for
  sparse learning.
\newblock In {\em 28th International Joint Conference on Artificial
  Intelligence}, 2019.

\bibitem{wang2020knowledge}
Lingxiao Wang and Quanquan Gu.
\newblock A knowledge transfer framework for differentially private sparse
  learning.
\newblock In {\em AAAI}, pages 6235--6242, 2020.

\bibitem{woolson2011statistical}
Robert~F Woolson and William~R Clarke.
\newblock {\em Statistical methods for the analysis of biomedical data}, volume
  371.
\newblock John Wiley \& Sons, 2011.

\bibitem{zhang2018ell_1}
Lijun Zhang and Zhi-Hua Zhou.
\newblock $\ell\_1$-regression with heavy-tailed distributions.
\newblock In {\em Advances in Neural Information Processing Systems}, pages
  1076--1086, 2018.

\bibitem{zhou2020bypassing}
Yingxue Zhou, Zhiwei~Steven Wu, and Arindam Banerjee.
\newblock Bypassing the ambient dimension: Private sgd with gradient subspace
  identification.
\newblock {\em arXiv preprint arXiv:2007.03813}, 2020.

\end{thebibliography}
\newpage 

\appendix

\section{Omitted Proofs}
\begin{proof}[{\bf Proof of Theorem \ref{thm:1}}]
In each iteration, by the definition of the robust estimator in (\ref{eq:5}) we can see that the after changing a data record ({\em i.e.,} $D_t$ to $D_t'$) we have $\|\tilde{g}(w^{t-1}, D_t)- \tilde{g}(w^{t-1}, D'_t)\|_\infty\leq \frac{4\sqrt{2}s}{3m}$. Thus for a fixed $v\in V$, the sensitivity of the score function satisfies $|u(D_{t}, v)-u(D'_{t}, v)|=|\langle v, \tilde{g}(w^{t-1}, D_t)-\tilde{g}(w^{t-1}, D_t') \rangle|\leq \|v\|_1 \|\tilde{g}(w^{t-1}, D_t)-\tilde{g}(w^{t-1}, D'_t)\|_\infty\leq \|\mathcal{W}\|_1\frac{4\sqrt{2}s}{3m}$. Thus, step 6 in Algorithm \ref{alg:1} is $\epsilon$-DP. Since in each iteration we use a new subset of the data, the whole algorithm will be $\epsilon$-DP. 
\end{proof}

\begin{proof}[{\bf Proof of Theorem \ref{thm:2}}]
Before analyzing the utility, we first provide a general upper bound of the error of $\hat{x}(s, \beta)$ in (\ref{eq:4}). 
\begin{lemma}\label{lemma:3}
Let  $x_1, x_2, \cdots, x_n$ be i.i.d. samples from  distribution $x\sim \mu$. Assume that there is some known upper bound on the second-order moment, {\em i.e.,} $\mathbb{E}_\mu x^2\leq \tau $. For a given failure probability $\zeta$,  then with probability at least $1-\zeta$ we have
\begin{equation}\label{aeq:10}
    | \hat{x}(s, \beta) -\mathbb{E}x| \leq  \frac{\tau }{2s}(\frac{1}{\beta}+1)+\frac{s}{n}(\frac{\beta}{2}+\log \frac{2}{\zeta}). 
\end{equation}
\end{lemma}
\begin{proof}[{\bf Proof of Lemma \ref{lemma:3}}] 
Let $\mathcal{P}(\mathbb{R})$ denote all the probability measures on $\mathbb{R}$, with an appropriate $\sigma$-field tacitly assumed. Consider any two measures $v, v_0\in \mathcal{P}(\mathbb{R})$, and $h:\mathbb{R}\mapsto \mathbb{R}$ a $v_0$-measurable function. By \cite{catoni2004statistical}, it is proved that a Legendre transform of the mapping $v\mapsto K(v, v_0)$ takes the form of a cumulant generating function, namely 
\begin{equation}\label{aeq:11}
    \sup_v (\int h(u)dv(u)-K(v, v_0))=\log \int \exp(h(u))dv_0(u), 
\end{equation}
where the supremum is taken over $v\in \mathcal{P}(\mathbb{R})$. Following \cite{catoni2017dimension} here we use the Kullback
divergence for the Legendre transform of the mapping, so we define 
\begin{equation}\label{aeq:12}
     K(v, v_0)=\int \log (\frac{d v}{d v_0} )dv
\end{equation}
if $v_0\ll v$, and  $ K(v, v_0)=+\infty$ otherwise.

This identity is a technical tool and the choice of $h$ and $v_0$ are parameters that can be adjusted to fit the application. 
In actually setting these parameters, we will follow the technique given by \cite{catoni2017dimension}, which is later adapted by \cite{holland2019a}. Note the term 
\begin{equation*}
    \phi(\frac{x_i+\eta_i x_i}{s})
\end{equation*}
depends on two terms, namely the data $x_i$ and the artificial noise $\eta_i$ (if we fix $s$). Thus, for convenience we denote 
\begin{equation*}
    f(\eta, x):=\phi(\frac{x+\eta x}{s}).
\end{equation*}
By the definition of $\phi$ we can see that $f:\mathbb{R}^2\mapsto \mathbb{R}$ is measurable and bounded. Next, we denote that 
\begin{equation*}
    h(\eta)=\sum_{i=1}^n f(\eta, x_i)-c(\eta), 
\end{equation*}
where $c(\epsilon)$ is a term to be determined shortly. Take $h(\eta)$ into (\ref{aeq:11}) we have 
\begin{align*}
    B:&=  \sup_v (\int h(u)dv(u)-K(v, v_0))\\
    &= \log \int \exp(\sum_{i=1}^n f(\eta, x_i)-c(\eta))dv(\eta), 
\end{align*}
Taking the exponential of this $B$ and then taking expectation with respect to the sample, we
have 
\begin{align*}
    \mathbb{E}\exp(B)&=\mathbb{E}\int (\frac{\exp(\sum_{i=1}^n f(\eta, x_i))}{\exp(c(\eta))})v(\eta) \\
    &= \int \frac{\Pi_{i=1}^n \mathbb{E}\exp( f(\eta, x_i))}{\exp(c(\eta))})v(\eta)
\end{align*}
The first equality comes from simple log/exp manipulations, and the second equality from
taking the integration over the sample inside the integration with respect to $v$, valid via Fubini’s theorem. By setting 
\begin{equation*}
    c(\eta)=n \log \mathbb{E} \exp (f(\eta, x)). 
\end{equation*}
 With this preparation done, we can start on the high-probability upper bound of interest:
\begin{align*}
    P(B\geq \log \frac{1}{\zeta})&= P(\exp(B)\geq \frac{1}{\zeta}) \\
    &= \mathbb{E}\mathbb{I}(\exp(B){\zeta} \geq 1) \\
    &\leq \mathbb{E} \exp(B){\zeta}=\zeta, 
\end{align*}
where the last equality is due to $\mathbb{E} \exp(B)=1$ by setting $ c(\eta)=n \log \mathbb{E} f(\eta, x)$. Note that since our setting of $c(\eta)$ is such that $c(\cdot)$ is $v$-measurable (via the measurability of $f$), the resulting $h$ is indeed measurable w.r.t $v$. Thus by the definition of $B$ we have with probability at least $1-\zeta$
\begin{equation*}
    \sup_v (\int h(u)dv(u)-K(v, v_0))\leq \log \frac{1}{\zeta}. 
\end{equation*}
    Take the implicit form of $B$ via $h(\eta)$ and $c(\eta)$ and divide by $n$ form both side we have 
    \begin{align}\label{aeq:13}
        \frac{1}{n}\int \sum_{i=1}^n f(\eta, x_i)dv(\eta)\leq \int \log \mathbb{E}\exp (f(\eta,x))d v(\eta) + \frac{K(v, v_0)+\log \frac{1}{\zeta}}{n}. 
    \end{align}
It is notable that by definition $\hat{x}(s, \beta)$ in (\ref{eq:4}) satisfies 
\begin{align*}
    \hat{x}(s, \beta)&=  \frac{s}{n}\sum_{i=1}^n \int \phi(\frac{x_i+\eta_i x_i}{s})d \chi(\eta_i)=\frac{s}{n}\int \sum_{i=1}^n f(\eta, x_i)dv(\eta) \\
    &\leq s\int \log \mathbb{E}\exp(f(\eta,x))d v(\eta)+ \frac{sK(v, v_0)+s\log \frac{1}{\zeta}}{n}.
\end{align*}
In the following we will bound the term of $\int \log \mathbb{E}\exp(f(\eta,x))d v(\eta) $ and $K(v, v_0)$. Starting with the first term, recall the definition of the truncation function $\phi(\cdot)$ in (\ref{eq:2}) we can it satisfies that for all $u\in \mathbb{R}$
\begin{equation}\label{aeq:14}
    -\log (1-u + \frac{u^2}{2})\leq \phi(u)\leq \log(1+u+\frac{u^2}{2}). 
\end{equation}
Thus we have 
\begin{align}
    &\int \log \mathbb{E}f(\eta,x)d v(\eta) \nonumber \\
    &= \int \log \mathbb{E} \exp(\phi(\frac{(1+\eta)x}{s}) d v(\eta) \nonumber \\
    & \leq \int \log (1+\frac{(1+\eta)\mathbb{E}x }{s}+ \frac{(1+\eta)^2\mathbb{E}x^2 }{2s^2}) dv(\eta)  \nonumber \\
    &\leq \int (\frac{(1+\eta)\mathbb{E}x }{s}+ \frac{(1+\eta)^2\mathbb{E}x^2 }{2s^2}) dv(\eta)  \nonumber \\
    &= \frac{\mathbb{E}x}{s} (1+\mathbb{E}_\eta \eta)+ \frac{\mathbb{E} x^2}{2s^2} \mathbb{E}_\eta (1+\eta)^2 \nonumber \\
    &= \frac{\mathbb{E}x}{s}+ \frac{\mathbb{E} x^2}{2s^2} (1+\frac{1}{\beta}). \label{aeq:15}
\end{align}
Where the last equality is due to $\eta \sim \mathcal{N}(0, \frac{1}{\beta})$.

For term $K(v, v_0)$, it is critical to select an appropriate measure $v_0$ to easily calculate the KL-divergence. Here we set $v_0\sim \mathcal{N}(1, \frac{1}{\beta})$. In this case we have 
\begin{align}
    K(v, v_0) &= \int_{-\infty}^{+\infty} \log (\exp(\frac{\beta(u-1)^2}{2}-\frac{\beta u^2}{2}))\sqrt{\frac{\beta}{2\pi}}\exp(-\frac{\beta u^2}{2})du \nonumber \\
    &= \int_{-\infty}^{+\infty} \frac{(1-2u)\beta }{2}\sqrt{\frac{\beta}{2\pi}}\exp(-\frac{\beta u^2}{2})du \nonumber \\
    &=\frac{\beta}{2}. \label{aeq:16}
\end{align}
Thus, combining with (\ref{aeq:15}) and (\ref{aeq:16}) we have with probability at least $1-\zeta$ 
\begin{equation}\label{aeq:17}
     \hat{x} (s, \beta)\leq \mathbb{E}x + \frac{\mathbb{E}x^2}{2s}(\frac{1}{\beta}+1)+\frac{s}{n}(\frac{\beta}{2}+\log \frac{1}{\zeta}). 
\end{equation}
Next, we will get a lower bound of $ \hat{x} (s, \beta)-\mathbb{E}x $. The proof is quite similar as in the above proof. The main difference is here we set 
\begin{equation*}
    f(\eta, x):=-\phi(\frac{x+\eta x}{s}). 
\end{equation*}
Thus we have 
\begin{align*}
    -\hat{x}&=  \frac{s}{n}\sum_{i=1}^n \int -\phi(\frac{x_i+\eta_i x_i}{s})d \chi(\eta_i)=\frac{s}{n}\int \sum_{i=1}^n f(\eta, x_i)dv(\eta) \\
    &\leq s\int \log \mathbb{E}\exp(f(\eta,x))d v(\eta)+ \frac{sK(v, v_0)+s\log \frac{1}{\zeta}}{n} \\ 
    &\leq s\int \log \mathbb{E}\exp(-\phi(\frac{x+\eta x}{s}))d v(\eta)+ \frac{sK(v, v_0)+s\log \frac{1}{\zeta}}{n}
\end{align*}
By (\ref{aeq:14}) we have 
\begin{align*}
     -\hat{x} &\leq s\int \log \mathbb{E}\exp(-\phi(\frac{x+\eta x}{s}))d v(\eta)+ \frac{sK(v, v_0)+s\log \frac{1}{\zeta}}{n} \\
     &\leq s[ (-1)\int \frac{(1+\eta)\mathbb{E}x}{s}dv(\eta)+ \int \frac{(1+\eta)^2\mathbb{E}x^2}{s^2}dv(\eta)] \\
     &\qquad + \frac{sK(v, v_0)+s\log \frac{1}{\zeta}}{n}\\
     &\leq -\mathbb{E} x+ \frac{\mathbb{E}x^2}{2s}(\frac{1}{\beta}+1)+\frac{s}{n}(\frac{\beta}{2}+\log \frac{1}{\zeta}). 
\end{align*}
     Thus we have 
     \begin{equation}\label{aeq:18}
         \mathbb{E} x-\hat{x}(s, \beta) \leq \frac{\tau}{2s}(\frac{1}{\beta}+1)+\frac{s}{n}(\frac{\beta}{2}+\log \frac{1}{\zeta}) 
     \end{equation}
     In total we have with probability at least $1-2\zeta$, 
     \begin{equation}\label{aeq:19}
        |\mathbb{E} x-\hat{x}(s, \beta)| \leq \frac{\tau}{2s}(\frac{1}{\beta}+1)+\frac{s}{n}(\frac{\beta}{2}+\log \frac{1}{\zeta}).
     \end{equation}
\end{proof}
Since in each iteration, we will use the estimator (\ref{eq:4}) to each coordinate of the gradient. Taking the union bound and setting the failure probability $\zeta=\frac{\zeta}{dT}$ in Lemma \ref{lemma:3}, and since in each iteration we use independent samples, we have with probability at least $1-\zeta$, for each $t\in [T]$
\begin{equation}\label{aeq:20}
    \|\tilde{g}(w^{t-1}, D_t)-\nabla L_\mathcal{D}(w^{t-1})\|_\infty \leq \frac{\tau }{2s}(\frac{1}{\beta}+1)+\frac{s}{m}(\frac{\beta}{2}+\log \frac{2dT}{\zeta}).
\end{equation}
In the following we will assume (\ref{aeq:20}) holds. We will show the following lemma: 
\begin{lemma}\label{lemma:4}
Assume (\ref{aeq:20}) holds, then with probability at least $1-\zeta$, in each iteration $t\in [T]$, for $\tilde{w}^{t-1}$ we have 
\begin{multline}
       \langle \tilde{w}^{t-1}, \nabla L_\mathcal{D}(w^{t-1})\leq \min_{v\in V} \langle v, \nabla L_\mathcal{D}(w^{t-1})\rangle+ O(\frac{\|\mathcal{W}\|_1\tau }{s}(\frac{1}{\beta}+1) \\+\frac{\|\mathcal{W}\|_1sT}{n}(\beta+\log \frac{dT}{\zeta})+\frac{\|\mathcal{W}\|_1 sT\log |V| \log \frac{T}{\zeta}}{n\epsilon} ). 
\end{multline}
\end{lemma}
\begin{proof}[ {\bf Proof of Lemma \ref{lemma:4}}]
By the utility of the exponential mechanism (Lemma \ref{lemma:exp}) we have with probability at least $1-\zeta$, in each iteration $t\in [T]$,
\begin{equation}\label{aeq:22}
    \langle \tilde{w}^{t-1}, \tilde{g}(w^{t-1}, D_t)\rangle\leq \min_{v\in V} \langle v, \tilde{g}(w^{t-1}, D_t) \rangle + O(\frac{sT\|\mathcal{W}\|_1\log |V| \log \frac{T}{\zeta}}{n\epsilon}). 
\end{equation}
For the left side of (\ref{aeq:22}) we have 
\begin{align}
    & \langle \tilde{w}^{t-1}, \nabla L_\mathcal{D}(w^{t-1})- \tilde{g}(w^{t-1}, D_t)\rangle \nm \\
    &\leq \|\tilde{w}^{t-1}\|_1 \|\nabla L_\mathcal{D}(w^{t-1})- \tilde{g}(w^{t-1}, D_t)\|_\infty  \nonumber \\
     &\leq \|\mathcal{W}\|_1 \big(\frac{\tau }{s}(\frac{1}{\beta}+1)+\frac{s}{m}({\beta}+\log \frac{2dT}{\zeta})\big). \label{aeq:23} 
\end{align}
For the right side of (\ref{aeq:22}), we denote $v_1=\arg\min_{v\in V} \langle v, \tilde{g}(w^{t-1}, D_t) \rangle $ and $v_2= \arg\min_{v\in V} \langle v, \nabla L_\mathcal{D}(w^{t-1})\rangle$, then we have 
\begin{align}
    &\langle v_1, \tilde{g}(w^{t-1}, D_t) \rangle \leq \langle v_2, \tilde{g}(w^{t-1}, D_t) \rangle \nonumber \\
    &\leq  \langle v_2, \nabla L_\mathcal{D}(w^{t-1}) \rangle+ \|\mathcal{W}\|_1 \big(\frac{\tau }{2s}(\frac{1}{\beta}+1)+\frac{s}{m}(\frac{\beta}{2}+\log \frac{2dT}{\zeta})\big). \label{aeq:24} 
\end{align}
We can finish the proof by combing with (\ref{aeq:23}), (\ref{aeq:24}) and (\ref{aeq:22}).
\end{proof}
Next, we recall the converge rate of Frank-Wolfe method proposed in \cite{jaggi2013revisiting}. 
\begin{lemma}[Theorem 1 in \cite{jaggi2013revisiting}]\label{lemma:5}
In the Frank-Wolfe algorithm, suppose in each iteration we get $w^t=(1-\eta_{t-1})w^{t-1}+\eta_{t-1}s_t$, where $s_t \in \mathcal{W}$ such that 
\begin{equation}\label{aeq:25}
    \langle s_t, \nabla L_\mathcal{D}(w^{t-1})\rangle \leq \min_{s\in \mathcal{W}} \langle s, \nabla L_\mathcal{D}(w^{t-1})\rangle +\frac{1}{2}\chi \eta_{t-1} \Gamma_L. 
\end{equation}
Where $\chi>0$ is fixed,  $\Gamma_L$ is the curvature constant of the function $L_\mathcal{D}(w)$ which is bound by $\Gamma_L\leq \alpha \|\mathcal{W}\|_1$. if $L_\mathcal{D}(\cdot)$ is $\alpha$-smooth. 
If we take $\eta_{t-1}=\frac{2}{t+2}$ then 
\begin{equation*}
    L_\mathcal{D}(w^t)-\min_{w\in \mathcal{W}} L_\mathcal{D}(w)\leq \frac{2\Gamma_L}{t+2}(1+\chi). 
\end{equation*}
\end{lemma}
Now we will begin to proof Theorem \ref{thm:2}.  Under the condition of Lemma \ref{lemma:4} we can see if $\chi$ in Lemma \ref{lemma:5} satisfies 
\begin{multline*}
    \frac{1}{2}\chi \eta_{t-1}\alpha \|\mathcal{W}\|_1 \geq  O(\frac{\|\mathcal{W}\|_1\tau }{s}(\frac{1}{\beta}+1) +\frac{\|\mathcal{W}\|_1sT}{n}(\beta+\log \frac{dT}{\zeta})+\frac{\|\mathcal{W}\|_1 sT\log |V| \log \frac{T}{\zeta}}{n\epsilon} )
\end{multline*}
That is 
\begin{equation*}
   \chi  =O(\frac{T}{\alpha} \big(  \frac{\tau }{s}(\frac{1}{\beta}+1)+\frac{sT}{n}(\beta+\log \frac{dT}{\zeta})+\frac{ sT\log |V| \log \frac{T}{\zeta}}{n\epsilon} \big) ),
\end{equation*}

Then by Lemma \ref{lemma:5} we have if $\eta=\frac{2}{t+2}$ then (note that this hold with probability at least $1-2\zeta$), 
\begin{align}
     & L_\mathcal{D}(w^T)-\min_{w\in \mathcal{W}} L_\mathcal{D}(w)\nonumber\\
     &\leq O\big( \frac{\alpha\|\mathcal{W}\|_1}{T}+ \frac{\|\mathcal{W}\|_1\tau }{s}(\frac{1}{\beta}+1) +\frac{\|\mathcal{W}\|_1sT}{n}(\beta+\log \frac{dT}{\zeta}) +\frac{\|\mathcal{W}\|_1 sT\log |V| \log \frac{T}{\zeta}}{n\epsilon} \big)\nonumber \\
     &\leq O\big( \frac{\alpha\|\mathcal{W}\|_1}{T}+ \frac{\|\mathcal{W}\|_1\tau }{s\beta} +\frac{\|\mathcal{W}\|_1 sT\beta \log |V| \log \frac{T}{\zeta}}{n\epsilon} \big) . \label{aeq:26}
\end{align}
Thus, take $\beta=O(1), s=O(\sqrt{\frac{n\epsilon\tau }{T \log \frac{|V|dT}{\zeta}}})$ we have 
   \begin{equation*}
          L_\mathcal{D}(w^T)-\min_{w\in \mathcal{W}} L_\mathcal{D}(w)\leq O \big( \frac{\alpha\|\mathcal{W}\|_1}{T}+\|\mathcal{W}\|_1 \frac{\sqrt{\tau T\log \frac{|V|dT}{\zeta}}}{\sqrt{n\epsilon}}\big). 
   \end{equation*}
   Taking $T=\tilde{O}\big( (\frac{n\epsilon \alpha^2}{\tau\log \frac{|V|d}{\zeta} })^\frac{1}{3}  \big)$. We can get 
   \begin{equation}\label{aeq:27}
          L_\mathcal{D}(w^T)-\min_{w\in \mathcal{W}} L_\mathcal{D}(w)\leq O \big( \frac{\|\mathcal{W}\|_1 (\alpha\tau\log \frac{n |V|d}{\zeta} )^\frac{1}{3} }{(n\epsilon)^\frac{1}{3}}\big).
   \end{equation}

\end{proof}
\begin{proof}[{\bf Proof of Theorem \ref{thm:3}}]
We first show the following lemma.
\begin{lemma}\label{lemma:6}
Under Assumption \ref{ass:2}, for any $w\in \mathcal{W}$ we have 
\begin{equation*}
     L_\mathcal{D}(w)- L_\mathcal{D}(w^*)\leq \frac{C_\psi}{2c_\psi}\langle \nabla L_\mathcal{D}(w), w-w^* \rangle.
\end{equation*}
\end{lemma}
\begin{proof}[{\bf Proof of Lemma \ref{lemma:6}} ]
First, the smoothness of $\psi$ implies that for any $s, s^*$, we have 
\begin{equation*}
    \psi(s)-\psi(s^*)\leq \psi'(s^*)(s-s^*)+\frac{C_\psi}{2}(s-s^*)^2.
\end{equation*}
Taking $s=\langle w, x \rangle$ and $s^*=\langle w^*, x\rangle$, and then taking expectation w.r.t. $(x,y)$, we get
\begin{align}
   & L_\mathcal{D}(w)-L_\mathcal{D}(w^*)\nm \\
   &\leq \mathbb{E}_{x,y}[\psi'(\langle w^*, x\rangle -y)\langle w-w^*, x \rangle ]+\frac{C_\psi}{2}\mathbb{E}\langle w-w^*, x\rangle ^2 \nonumber\\
    &=\langle \nabla L_\mathcal{P}(w^*), w-w^*\rangle +\frac{C_\psi}{2}\mathbb{E}\langle w-w^*, x\rangle ^2. \label{aeq:28}
\end{align}
By Assumption \ref{ass:2}, we have
\begin{equation*}
    \nabla L_\mathcal{P}(w^*)=\mathbb{E}_{x,\xi}[\psi'(-\xi)x]=0.
\end{equation*}
Thus, we get $ L_\mathcal{P}(w)-L_\mathcal{P}(w^*)\leq\frac{C_\psi}{2}\mathbb{E}\langle w-w^*, x\rangle ^2 $. On the other hand, using gradient we have 
\begin{align*}
    \langle \nabla L_\mathcal{P}(w), w-w^* \rangle &= \mathbb{E}_x[\mathbb{E}_\xi \psi'(\langle w-w^*, x\rangle -\xi)\langle w-w^*, x\rangle ]\\
    &=\mathbb{E}_x[h(\langle w-w^*, x\rangle )\langle w-w^*, x\rangle ].
\end{align*}
By the assumption on function $h(\cdot)$, we  get
\begin{align*}
   h(\langle w-w^*, x\rangle )\langle w-w^*, x\rangle &=\frac{h(\langle w-w^*, x\rangle )}{\langle w-w^*, x\rangle} \langle w-w^*, x\rangle^2 \\
   &\geq c_\psi \langle w-w^*, x\rangle^2,
\end{align*}
where the inequality is due to the fact that $h(0)=0$ and $h'(0)\geq c_\psi$.

Taking the expectation, we have 
\begin{equation*}
     \langle \nabla L_\mathcal{P}(w), w-w^* \rangle\geq c_\psi \mathbb{E}_x \langle w-w^*, x\rangle^2.
\end{equation*}
Combing previous inequalities we can finish the proof. 
\end{proof}
Next, we will verify the loss function satisfies is smooth and has bounded second order moment for each coordinate of the gradient. For the smooth, we can see that 
\begin{equation*}
    \|\nabla^2 L_\mathcal{D}(w) \|_2 = \|\mathbb{E} x^Tx \psi''(\langle w, x\rangle -y)\|_2 \leq C_\psi \lambda_{\max} (\mathbb{E}(xx^T)). 
\end{equation*}
For each coordinate $j\in [d]$, 
\begin{equation*}
    \mathbb{E}(\nabla_j \ell(w; z))^2=\mathbb{E} [(x_j\psi'(\langle w, x\rangle -y))^2 ]\leq C^2_\psi \mathbb{E} x^2_j=O(C^2_\psi). 
\end{equation*}
Similar to Lemma \ref{lemma:4}, we have  with probability at least $1-2\zeta$, in each iteration $t\in [T]$, for $\tilde{w}^{t-1}$ we have 
\begin{multline}\label{aeq:29}
       \langle \tilde{w}^{t-1}, \nabla L_\mathcal{D}(w^{t-1})\leq \min_{v\in V} \langle v, \nabla L_\mathcal{D}(w^{t-1})\rangle \\ + O(\frac{ C^2_\psi }{s}(\frac{1}{\beta}+1)+\frac{sT}{n}(\beta+\log \frac{dT}{\zeta})+\frac{ sT \log \frac{dT}{\zeta}}{n\epsilon} ). 
\end{multline}
Denote $v^{t-1}=\arg\min_{v\in V}\langle v,\nabla L_\mathcal{D}(w^{t-1})\rangle$ and $\mu=O(\frac{ C^2_\psi }{s}(\frac{1}{\beta}+1)+\frac{sT}{n}(\beta+\log \frac{dT}{\zeta})+\frac{ sT \log \frac{dT}{\zeta}}{n\epsilon} )$. 
	By the smooth property, we have 
	\begin{align*}
	&	\frac{ C_\psi \lambda_{\max} (\mathbb{E}(xx^T))}{2}\|w^{t}-w^{t-1}\|_2^2\\ &\geq L_\mathcal{D} (w^{t})-L_\mathcal{D}(w^{t-1})-\langle L_\mathcal{D}(w^{t-1}),w^{t}-w^{t-1}\rangle\\
		&=L_\mathcal{D}(w^{t})-L_\mathcal{D}(w^{t-1})-\eta_{t-1}\langle \nabla L_\mathcal{D}(w^{t-1}), \tilde{w}^{t-1}-w^{t-1}\rangle\\
		&\geq L_\mathcal{D}(w^{t})-L_\mathcal{D}(w^{t-1})-\eta_{t-1}(\langle  \nabla L_\mathcal{D}(w^{t-1}),v^{t-1}-w^{t-1}\rangle +\mu).
	\end{align*}
Thus, we have 
	\begin{align*}
	   & L_\mathcal{D}(w^{t})-L_\mathcal{D}(w^{t-1})+\eta_{t-1} \langle  \nabla L_\mathcal{D}(w^{t-1}), w^{t-1}-w^*\rangle\\
	   &\leq    L_\mathcal{D}(w^{t})-L_\mathcal{D}(w^{t-1})+\eta_{t-1} \langle  \nabla L_\mathcal{D}(w^{t-1}), w^{t-1}-v^{t-1}\rangle 
	   \\
	   & \leq 	\frac{ C_\psi \eta^2_{t-1} \lambda_{\max} (\mathbb{E}(xx^T))}{2}\|\mathcal{W}\|_1^2+\eta_{t-1}\mu 
	\end{align*}
	By Lemma \ref{lemma:6} we can get 
	\begin{multline}\label{aeq:30}
	    L_\mathcal{D}(w^{t})-L_\mathcal{D}(w^{t-1})+\frac{2c_\psi}{C_\psi}\eta_{t-1} (L_\mathcal{D}(w^{t-1})-L_\mathcal{D}(w^*)) \\
	    \leq \frac{ C_\psi \eta^2_{t-1} \lambda_{\max} (\mathbb{E}(xx^T))}{2}\|\mathcal{W}\|_1^2+\eta_{t-1}\mu. 
	\end{multline}
	Denote $\Delta_{t}=  L_\mathcal{D}(w^{t})-L_\mathcal{D}(w^*)$, we have \begin{equation*}
	    \Delta_t\leq (1-\frac{2c_\psi}{C_\psi}\eta_{t-1})\Delta_{t-1}+\frac{ C_\psi\eta^2_{t-1} \lambda_{\max} (\mathbb{E}(xx^T))}{2}\|\mathcal{W}\|_1^2+\eta_{t-1}\mu.  
	\end{equation*}
	Let $\eta_{t-1}=\eta$ for all $t$. Sum from $t=1\cdots T$ we have 
	\begin{align*}
	      \Delta_T &\leq (1-\frac{2c_\psi}{C_\psi}\eta)^T\Delta_{0}+\frac{C_\psi}{2c_\psi\eta} \big(\frac{ C_\psi\eta^2 \lambda_{\max} (\mathbb{E}(xx^T))}{2}+\eta \mu \big) \\ 
	      &\leq \frac{C_\psi}{2c_\psi}\frac{1}{\eta T}\Delta_{0}+ \big(\frac{ C^2_\psi\eta \lambda_{\max} (\mathbb{E}(xx^T))}{2c_\psi}+ \frac{C_\psi}{2c_\psi} \mu \\
	      &= O\big(\frac{C_\psi}{2c_\psi}\frac{1}{\eta T}+\frac{ C^2_\psi\eta \lambda_{\max} (\mathbb{E}(xx^T))}{2c_\psi} +\frac{ C^3_\psi }{c_\psi s\beta}+\frac{C_\psi sT \beta \log \frac{dT}{\zeta}}{c_\psi n\epsilon} \big).
	\end{align*}
	  Note that the second inequality is due to $(1-\eta)^T\leq \frac{1}{\eta T}$, since it is equivalent to $\eta T\leq \frac{\eta T}{1-\eta}\leq  (1+\frac{\eta}{1-\eta})^T$

	Taking $\beta=O(1)$, $s=O(\frac{\sqrt{n\epsilon}}{ \sqrt{T\log \frac{dT}{\zeta}}})$, $\eta=\frac{1}{\sqrt{T}}$, and $T=\tilde{O}(\sqrt{\frac{n\epsilon}{\log \frac{d}{\zeta}}})$ we get the result.
\end{proof}
\begin{proof}[{\bf Proof of Theorem \ref{thm:4}}]
We can see that after truncating, each entry in $\tilde{x}_i$ and $\tilde{y}_i$ satisfies $|\tilde{x}_{i,j}|\leq K, |\tilde{y}_{i}|\leq K$. Consider a neighboring data of $D$, $D'$ and its correspond truncated data $\tilde{D}'$. Then for each fixed $w$ and $v\in V$ we have 
\begin{align*}
&|u(\tilde{D}, v)-u(\tilde{D}', v)|\leq \|\mathcal{W}\|_1 \|\tilde{g}(w^{t-1}, \tilde{D})- \tilde{g}(w^{t-1}, \tilde{D}')\|_\infty \\ &\leq  \frac{4\|\mathcal{W}\|_1}{n}\|\tilde{x}_n(\langle \tilde{x}_n,  w^{t-1} \rangle-\tilde{y}_n)\|_\infty
\leq \frac{8\|\mathcal{W}\|_1 K^2}{n}. 
\end{align*}
Thus the sensitivity of $u(\tilde{D}, v)$ is $\frac{8\|\mathcal{W}\|_1 K^2}{n}$. And since in each iteration it is $\frac{\epsilon}{2\sqrt{2T\log \frac{1}{\delta}}}$-DP. By the advanced composition theorem (Lemma \ref{lemma:adv}) we can see the while algorithm is $(\epsilon, \delta)$-DP. 
\end{proof}
\begin{proof}[{\bf Proof of Theorem \ref{thm:5}}]
The key lemma of the whole proof is the following: 
\begin{lemma}\label{lemma:7}
Denote $g(w)=\nabla L_\mathcal{D}(w)= \mathbb{E}[2x(\langle w, x\rangle-y)]$. We have the following inequality with probability at least $1-\zeta$:
\begin{equation}\label{aeq:33}
    \sup_{v\in V}\sup_{w\in \mathcal{W}} |\langle v, \tilde{g}(w, \tilde{D})-g(w) \rangle | \leq O(\|\mathcal{W}\|^2_1(\sqrt{\frac{M\log \frac{d}{\zeta}}{n}}+\frac{K^2\log \frac{d}{\zeta}}{n}+ \frac{M}{K^2})). 
\end{equation}
\end{lemma}
\begin{proof}[{\bf Proof of Lemma \ref{lemma:7}}]
We can easily get 
\begin{align*}
     & \sup_{v\in V}\sup_{w\in \mathcal{W}} |\langle v, \tilde{g}(w, \tilde{D})-g(w) \rangle | \\
     &\leq \|\mathcal{W}\|_1 \sup_{w\in \mathcal{W}} \| \tilde{g}(w, \tilde{D})-g(w) \|_\infty \\
      &\leq  \|\mathcal{W}\|_1 \sup_{w\in \mathcal{W}}(\|2(\frac{1}{n}\sum_{i=1}^n \tilde{x}_i\tilde{x}_i^T-\mathbb{E}[xx^T])w\|_\infty) + \|2(\frac{1}{n}\sum_{i=1}^n \tilde{x}_i\tilde{y}-\mathbb{E}[xy])|_\infty\\
      &\leq \|\mathcal{W}\|_1  \sup_{w\in \mathcal{W}}\max_{j\in [d]}(2(\frac{1}{n}\sum_{i=1}^n \tilde{x}_i\tilde{x}_i^T-\mathbb{E}[xx^T])^T_j w)  + \|2(\frac{1}{n}\sum_{i=1}^n \tilde{x}_i\tilde{y}-\mathbb{E}[xy])|_\infty\\
      &\leq 2\|\mathcal{W}\|^2_1 \big(\underbrace{\|(\frac{1}{n}\sum_{i=1}^n \tilde{x}_i\tilde{x}_i^T-\mathbb{E}[xx^T])\|_{\infty, \infty}}_{A}+\underbrace{ \|(\frac{1}{n}\sum_{i=1}^n \tilde{x}_i\tilde{y}-\mathbb{E}[xy])|_\infty}_{B }\big),
\end{align*}
where $\|S\|_{\infty, \infty}$ for a $d\times d$ matrix $S$ is $\max_{i\in[d],j\in [d]}|S_{ij}|$. In the following we will bound term $A$ and $B$.

First we consider the term $A$, for simplicity for each $j, k\in [d]$ denote $\hat{\sigma}_{jk}=(\frac{1}{n}\sum_{i=1}^n \tilde{x}_i\tilde{x}_i^T)_{jk}=\frac{1}{n}\sum_{i=1}^n\tilde{x}_{i,j}\tilde{x}_{i, k}$, $\tilde{\sigma}_{jk}= (\mathbb{E}[\tilde{x}\tilde{x}^T])_{jk}= \mathbb{E}[\tilde{x}_j\tilde{x}_k]$ and $\sigma_{jk}=(\mathbb{E}[{x}{x}^T])_{jk}= \mathbb{E}[{x}_j{x}_k]$. We have 
\begin{equation*}
    |\hat{\sigma}_{jk}- \sigma_{jk}| \leq |\hat{\sigma}_{jk}-\tilde{\sigma}_{jk}|+|\tilde{\sigma}_{jk}-\sigma_{jk}|. 
\end{equation*}
We know that $|\tilde{x}_j\tilde{x}_k|\leq K^2$ and $\text{Var}(\tilde{x}_j\tilde{x}_k)\leq  \text{Var}(x_jx_k)\leq \mathbb{E}(x_jx_k)^2\leq  M$. 

By using Bernstein's inequality (Lemma \ref{bern}) we have for any $t>0$
\begin{equation}\label{aeq:34}
    \text{Pr}(|\hat{\sigma}_{jk}-\sigma_{jk}|\leq  C(\sqrt{\frac{Mt}{n}}+\frac{K^2t}{n}))\geq 1-\exp(-t). 
\end{equation}
Taking the union bound for all $k,j\in [d]$ we have 
\begin{equation}\label{aeq:35}
    \text{Pr}(\max_{j,k}|\hat{\sigma}_{jk}-\sigma_{jk}|\leq  C(\sqrt{\frac{Mt}{n}}+\frac{K^2t}{n}))\geq 1-d^2\exp(-t). 
\end{equation}
For the term $|\tilde{\sigma}_{ij}-\sigma_{ij}|$ by definition we have there is a university constant $C_2>0$,
\begin{align}\label{aeq:36}
   & |\tilde{\sigma}_{jk}-\sigma_{jk}|\nm \\
   &= |\mathbb{E}[\tilde{x}_j (\tilde{x}_k-x_k)\mathbb{I}(|x_k|\geq K)] |+| \mathbb{E}[(\tilde{x}_j-x_j)\mathbb{I}(|x_j|\geq K) x_k]|\nm \\
   & \leq C_2\frac{M}{K^2}. 
\end{align}
To get (\ref{aeq:36}), we have
\begin{align*}
    &   |\tilde{\sigma}_{jk}-\sigma_{jk}|\\
    &= |\mathbb{E}[\tilde{x}_j (\tilde{x}_k-x_k)\mathbb{I}(|x_k|\geq K)] |+| \mathbb{E}[(\tilde{x}_j-x_j)\mathbb{I}(|x_j|\geq K) x_k]| \\
    &\leq \sqrt{\mathbb{E}(\tilde{x}_j (\tilde{x}_k-x_k))^2 \text{Pr}( |x_k|^4\geq K^4) } + \sqrt{\mathbb{E}  ((\tilde{x}_j-x_j) x_k)^2 \text{Pr} (|x_j|^4\geq K^4)}\\
    &\leq O(\frac{M}{K^2}), 
\end{align*}
where the last inequality is due to that $\mathbb{E}(\tilde{x}_j (\tilde{x}_k-x_k))^2\leq  4\mathbb{E}(x_j x_k))^2 \leq 4M$ and $\mathbb{E}  ((\tilde{x}_j-x_j) x_k)^2\leq  4\mathbb{E}(x_j x_k))^2\leq 4M $.

Combing with (\ref{aeq:36}) and (\ref{aeq:35}) we have with probability at least $1-d^2\exp(-t)$, 
\begin{equation}\label{aeq:37}
    A\leq O(\sqrt{\frac{Mt}{n}}+\frac{K^2t}{n}+ \frac{M}{K^2}). 
\end{equation}
We can use the same technique to term $B$, for simplicity for each $j \in [d]$ denote $\hat{\sigma}_{j}=\frac{1}{n}\sum_{i=1}^n \tilde{y}_i\tilde{x}_j$, $\tilde{\sigma}_{j}= \mathbb{E}[\tilde{y}\tilde{x}_j]$ and $\sigma_{j}= \mathbb{E}[y{x}_j]$. We have 
\begin{equation*}
    |\hat{\sigma}_{j}- \sigma_{j}| \leq |\hat{\sigma}_{j}-\tilde{\sigma}_{j}|+|\tilde{\sigma}_{j}-\sigma_{j}|. 
\end{equation*}
Since $|\tilde{x}_{j}\tilde{y}|\leq K^2 $ and $\text{Var}(\tilde{x}_{j}\tilde{y})\leq \text{Var}(x_jy)\leq M$ we have  for all $j\in [d]$
\begin{equation*}
  \text{Pr}(  |\hat{\sigma}_{j}-\tilde{\sigma}_{j}|\leq C(\sqrt{\frac{Mt}{n}}+\frac{K^2t}{n}))\geq 1-d\exp(-t). 
\end{equation*}
Moreover 

\begin{align*}
    &|\tilde{\sigma}_{j}-\sigma_{j}|\leq |\mathbb{E}[\tilde{y}(\tilde{x}_j-x_j)\mathbb{I}(|x_j|) \geq K ] |+|\mathbb{E}[x_j(\tilde{y}-y) \mathbb{I}(|y|\geq K)]|\\
     &\leq \sqrt{\mathbb{E}((\tilde{y}(\tilde{x}_j-x_j))^2 \text{Pr}( |x_j|^4\geq K^4) }  + \sqrt{\mathbb{E}  (x_j(\tilde{y}-y))^2 \text{Pr} (|y|^4\geq K^4)}\\
    &\leq O(\frac{\sqrt{M}}{K^2}), 
\end{align*}

we can easily see that with probability at most $1-d\exp(-t)$, 
\begin{equation}\label{aeq:38}
    B\leq O(\sqrt{\frac{Mt}{n}}+\frac{K^2t}{n}+ \frac{\sqrt{M}}{K^2}).
\end{equation}
Take $t=\frac{\zeta}{d^3}$, we can finish the proof. That is with probability at least $1-\zeta$
\begin{equation}\label{aeq:39}
   A+ B\leq O(\sqrt{\frac{M\log \frac{d}{\zeta}}{n}}+\frac{K^2\log \frac{d}{\zeta}}{n}+ \frac{M}{K^2}).
\end{equation}
\begin{lemma}\label{bern}[\cite{vershynin2018high}]
Let $X_1, \cdots, X_n$ be $n$ be independent zero-mean random variables. Suppose each $|X_i|\leq s$, and $\mathbb{E}[X_i^2]\leq r$. Then there is a universal constant $C$ such that for all $t>0$, 
\begin{equation*}
  \text{Pr}(\frac{1}{n}\sum_{i=1}^n X_i >C(\frac{st}{n}+\sqrt{\frac{r t}{n}}))\leq \exp(-t). 
\end{equation*}
\end{lemma}
\end{proof}
Now lets back to the proof of Theorem \ref{thm:5}. By using the utility of the exponential mechanism and take the union for all iterations we have with probability at least $1-\zeta$, 
    \begin{equation}\label{aeq:40}
    \langle \tilde{w}^{t-1}, \tilde{g}(w^{t-1}, D)\rangle\leq \min_{v\in V} \langle v, \tilde{g}(w^{t-1}, D) \rangle + O(\frac{\|\mathcal{W}\|_1K^2\log |V|\sqrt{T\log \frac{1}{\delta}}  \log \frac{T}{\zeta}}{n\epsilon}). 
\end{equation}
Using Lemma \ref{lemma:4} and Lemma \ref{lemma:7} we have with probability at least $1-\zeta$ for all $t\in [T]$, 
    \begin{multline}\label{aeq:41}
    \langle \tilde{w}^{t-1}, g(w^{t-1})\rangle\leq \min_{v\in V} \langle v, g(w^{t-1}) \rangle  + O(\frac{\|\mathcal{W}\|_1K^2\log |V|\sqrt{T\log \frac{1}{\delta}}  \log \frac{T}{\zeta}}{n\epsilon}\\ +\|\mathcal{W}\|^2_1(\sqrt{\frac{M\log \frac{d}{\zeta}}{n}}+\frac{K^2\log \frac{d}{\zeta}}{n}+ \frac{M}{K^2})). 
\end{multline}
Since $L_\mathcal{D}(\cdot)$ is $\lambda_{\max}(\mathbb{E}(xx^T))$-smooth, we can get its curvature constant is bound by $\Gamma_L\leq \lambda_{\max}(\mathbb{E}(xx^T)) \|\mathcal{W}\|_1$. Under the condition of Lemma \ref{lemma:5} we can see if $\chi$ in Lemma \ref{lemma:5} satisfies 
\begin{multline*}
       \frac{1}{2}\chi \eta_{t-1}\lambda_{\max}(\mathbb{E}(xx^T)) \geq  O(\frac{K^2\log |V|\sqrt{T\log \frac{1}{\delta}}  \log \frac{T}{\zeta}}{n\epsilon}+\|\mathcal{W}\|_1(\sqrt{\frac{M\log \frac{d}{\zeta}}{n}}+\frac{K^2\log \frac{d}{\zeta}}{n}+ \frac{M}{K^2}) ). 
\end{multline*}
That is 
\begin{multline*}
   \chi  =O(\frac{T}{\lambda_{\max}(\mathbb{E}(xx^T))} \big(  \frac{K^2\log |V|\sqrt{T\log \frac{1}{\delta}}  \log \frac{T}{\zeta}}{n\epsilon} +\|\mathcal{W}\|_1(\sqrt{\frac{M\log \frac{d}{\zeta}}{n}}+\frac{K^2\log \frac{d}{\zeta}}{n}+ \frac{M}{K^2})  \big) ),
\end{multline*}

Then by Lemma \ref{lemma:3} we have if $\eta=\frac{2}{t+2}$ then (note that this hold with probability at least $1-2\zeta$), 
\begin{align}
   & L_\mathcal{D}(w^T)-\min_{w\in \mathcal{W}}L_\mathcal{D} (w) \nm\\
   &\leq O\big(\frac{\lambda_{\max}(\mathbb{E}(xx^T))}{T}+ \frac{\|\mathcal{W}\|_1 K^2\log |V|\sqrt{T\log \frac{1}{\delta}}  \log \frac{T}{\zeta}}{n\epsilon} \nonumber +\|\mathcal{W}\|^2_1(\sqrt{\frac{M\log \frac{d}{\zeta}}{n}}+\frac{K^2\log \frac{d}{\zeta}}{n}+ \frac{M}{K^2})\big)\nonumber \\
    &\leq O(\frac{\lambda_{\max}(\mathbb{E}(xx^T))}{T}+\sqrt{\frac{M\log \frac{d}{\zeta}}{n}}+ \frac{ K^2\sqrt{T\log \frac{1}{\delta}}  \log \frac{dT}{\zeta}}{n\epsilon}+\frac{1}{K^2}). \label{aeq:42}
\end{align}
Take $K=\frac{{(n\epsilon)^{\frac{1}{4}}}}{T^\frac{1}{8}}$, and $T=\tilde{O}((\frac{\sqrt{n\epsilon}\lambda_{\max}(\mathbb{E}(xx^T))}{\sqrt{\log \frac{1}{\delta}}  \log \frac{dT}{\zeta}})^\frac{4}{5})$. 
We have 
\begin{equation}\label{aeq:43}
      L_\mathcal{D}(w^T)-\min_{w\in \mathcal{W}}L_\mathcal{D} (w)\leq O(\frac{\lambda_{\max}^\frac{1}{5}(\mathbb{E}(xx^T)){(\sqrt{\log \frac{1}{\delta}}  \log \frac{dn}{\zeta})^\frac{4}{5}}}{(n\epsilon)^\frac{2}{5}}). 
\end{equation}
\end{proof}
\begin{proof}[\textbf{Proof of Theorem \ref{thm:6}}]
Since in each iteration of Algorithm \ref{alg:3} we use a new data. Thus, it is sufficient to show it is $(\epsilon, \delta)$-DP in each iteration. The proof is based on a lemma which shows that Algorithm \ref{alg:4} is $(\epsilon, \delta)$-DP: 
\begin{lemma}[Lemma 3.3 in \cite{cai2019cost}]\label{lemma:9}
If for every pair of neighboring datasets $D, D'$ we have $\|v(D)-v(D')\|_\infty \leq \lambda$, then Algorithm \ref{alg:4} is $(\epsilon, \delta)$-DP. 
\end{lemma}
Thus by Lemma \ref{lemma:9}, it is sufficient for us to bound the $\ell_\infty$-norm sensitivity of $w^{t+0.75}$. We have 
\begin{align*}
   & \|w^{t+0.75}-w'^{t+0.75}\|_\infty \leq \frac{2\eta_0}{m}\|\tilde{x}(\langle \tilde{x}, w^t\rangle -\tilde{y})\|_\infty \\
   &\leq \frac{2\eta_0}{m}(\|\tilde{x}\langle \tilde{x}, w^t\rangle\|_\infty+  \|\tilde{x}\tilde{y}\|_\infty)\leq \frac{2K^2\eta_0}{m}(\sqrt{s}+1), 
\end{align*}
where the last inequality is due to that $\langle \tilde{x}, w^t\rangle \leq \sqrt{s}\|\tilde{x}\|_\infty$ since $\|w^t\|_2\leq 1$ and $w^t$ is $s$-sparse.  
\end{proof}
\begin{proof}[{\bf Proof of Theorem \ref{thm:8}}]
For the guarantee of DP. Since in each iteration of Algorithm \ref{alg:5} we use a new data. Thus, it is sufficient to show it is $(\epsilon, \delta)$-DP in each iteration. Since we have 
\begin{equation*}
    \|w^{t+0.5}-w'^{t+0.5}\|_\infty \leq  \|\eta\tilde{g}(w^{t-1}, D_t)-\eta\tilde{g}(w^{t-1}, D'_t)\|_\infty\leq \frac{4\sqrt{2}\eta s}{3m}.
\end{equation*}
Thus, by Lemma \ref{lemma:9} we can see it is $(\epsilon, \delta)$-DP.

In the following we will proof the utility. For simplicity we will omit the subscript $r$ in $u_r, \lambda_r$.

We denote $\tilde{g}^t=\tilde{g}(w^{t-1}, D_t)$ and $g^t=\nabla L_\mathcal{D}(w^{t-1})$. The utility is almost the same as in the proof of Theorem \ref{thm:7} where $w^{t+0.75}=w^{t+1}$ in Algorithm \ref{alg:5} (since there is no projection step). We can follow its proof and finally get (\ref{aeq:69}) if  $\eta=\frac{2}{3}$, $s=72\frac{\gamma^2}{\mu^2}s^*$
\begin{equation}
    L_\mathcal{D}(w^{t+1})-\LT\leq (\frac{2 s^*}{s+s^*}+\frac{\mu}{24\gamma})(L_\mathcal{D}(w^*)-\LT)+O(N^t+N^t_2), 
\end{equation}
where 
$N^t+N^t_2\leq O(\sum_{i\in [s]}\|w_i\|_\infty^2+(2s+s^*)\|\tilde{g}^t-g^t\|_\infty^2+s\|\tilde{w}\|_\infty^2)$.
Take the union we have with probability at least $1-\zeta$ for all the iterations, $$\sum_{i\in [s]}\|w_i\|_\infty^2+s\|\tilde{w}\|_\infty^2\leq O(\frac{\eta_0^2s^2k^2T^2\log^2 \frac{Ts}{\zeta}\log\frac{1}{\delta}}{n^2\epsilon^2}).$$
For the term $\|\tilde{g}^t-g^t\|_\infty^2$, by Lemma \ref{lemma:3}, we have with probability at least $1-\zeta$, for all $t\in [T]$, 
\begin{equation}
    \|\tilde{g}^t-g^t\|_\infty^2\leq O(\frac{\tau^2}{k^2\beta^2}+\frac{\beta^2T^2k^2\log^2 \frac{dT}{\zeta}}{n^2} ). 
\end{equation}
Thus, 
$$N^t+N^t_2\leq O(\frac{s^2k^2T^2\log^2 \frac{Ts}{\zeta}\log\frac{1}{\delta}}{\gamma^2 n^2\epsilon^2}+\frac{s\tau^2}{k^2\beta^2}+\frac{s\beta^2T^2k^2\log^2 \frac{dT}{\zeta}}{n^2}).$$
Take $\beta=O(1)$ and $k=\sqrt[4]{\frac{n^2\epsilon^2\tau^2}{(sT)^2\log \frac{Ts}{\zeta}}}$, we have 

$$N^t+N^t_2\leq O(\frac{\tau s^{\frac{3}{2}}T\log \frac{Ts}{\zeta}\sqrt{\log\frac{1}{\delta}}}{\gamma^2 n\epsilon}).$$
In total we have for all $t\in [T]$ with probability at least $1-\zeta$, 
\begin{align}
      L_\mathcal{D}(w^{t+1})-L_\mathcal{D}(w^*)
      &\leq (1-\frac{\mu}{12\gamma})(\LT-L_\mathcal{D}(w^*)) \nm  +O(\frac{s^{\frac{3}{2}}T\tau \log \frac{ Ts}{\zeta}\sqrt{\log\frac{1}{\delta}}}{\gamma n\epsilon})
\end{align}

\begin{align*}
  L_\mathcal{D}(w^{T+1})-L_\mathcal{D}(w^*) &\leq (1-\frac{\mu}{12\gamma})^{T+1}(L_\mathcal{D}(w^1)-L_\mathcal{D}(w^*))+O(\frac{\tau s^{\frac{3}{2}}T\log \frac{Td}{\zeta}\sqrt{\log\frac{1}{\delta}}}{\mu n\epsilon}).
\end{align*}
Take $T=\tilde{O}(\frac{\gamma}{\mu}\log n)$ and $s=O((\frac{\gamma}{\mu})^2s^*)$ we have the result.

\end{proof}

\subsection{Proof of the Lower Bound in Theorem \ref{thm:8}}

\begin{proof}[{\bf Proof of Theorem \ref{thm:9}}]
Our following proof is inspired by the proof of the lower bound for non-sparse mean estimation in the $(\epsilon, \delta)$-DP model (Proposition 4 in \cite{barber2014privacy}), where here we extend to the sparse case.

Let $P_0$ be a point mass distribution supported on $X=0$, and for fixed $p\in [0, 1]$ let $P_v$ be a point mass supported on $X=\frac{\sqrt{\tau}}{\sqrt{p}}v$ where $\|v\|_2\leq 1$ and is $s^*$-sparse. We will show for any $p\in [0, 1]$, the probability $P_{\theta_v}:=(1-p)P_0+pP_v\in \mathcal{P}$. That is due to that $\mu(P_{\theta_v})= \sqrt{p\tau }v$ and $\mathbb{E}_{X\sim P_{\theta_v}} X_j^2=\tau v_j^2\leq \tau$. Next we recall the following lemma
\begin{lemma}\label{dlemma:2}[\cite{raskutti2011minimax}]
For any $s\in [d]$, define the set 
\begin{equation*}
\mathcal{H}(s):=\{z\in\{-1, 0, +1\}^d\mid \|z\|_0=s\} 
\end{equation*}
with Hamming distance $\rho_{H}(z,z')=\sum_{i=1}^d1[z_j\neq z'_{j}]$ between the vectors $z$ and $z'$. Then, there exists a subset $\tilde{\mathcal{H}}\subset \mathcal{H}$ with  cardinality $|\tilde{\mathcal{H}}|\geq \exp(\frac{s}{2}\log \frac{d-s}{s/2})$ such that $\rho_H(z, z')\geq \frac{s}{2}$ for all $z, z'\in \tilde{\mathcal{H}}$.
\end{lemma}
We denote the index set $\mathcal{V}=\frac{1}{\sqrt{2s}} \tilde{H}$ where $\tilde{H}$ is in Lemma \ref{dlemma:2}. We can see that for any $v, v'\in \mathcal{V}$ we have $\|v-v'\|_2\geq \sqrt{2}$ and each $\|v\|_2\leq 1$. Thus, 
\begin{equation*}
    \rho^*(\mathcal{V}):=\min\{\rho(\theta_v, \theta_{v'})| v, v'\in \mathcal{V}, v\neq v'\}\geq \sqrt{2}\sqrt{p\tau}. 
\end{equation*}
Thus by Lemma \ref{thm:lower} we have 
\begin{align*}
       &\mathcal{M}_n(\theta(\mathcal{P}), Q, \Phi\circ \rho) \\
       & \geq \Phi(  \rho^*(\mathcal{V})) \frac{(|\mathcal{V}|-1)(\frac{1}{2}e^{-\epsilon \lc np \rc}-\delta \frac{1-e^{-\epsilon \lc n p \rc} }{1-e^{-\epsilon}}  )}{1+(|\mathcal{V}|-1)e^{-\epsilon \lc n p \rc} } \\
       &\geq 2p\tau \frac{(|\mathcal{V}|-1)(\frac{1}{2}e^{-\epsilon \lc np \rc}-\delta \frac{1-e^{-\epsilon \lc n p \rc}}{1-e^{-\epsilon}} )}{1+(|\mathcal{V}|-1)e^{-\epsilon \lc n p \rc} } \\
      & \geq 2p\tau \frac{( \exp(\frac{s}{2}\log \frac{d-s}{s/2})-1)(\frac{1}{2}e^{-\epsilon \lc np \rc}-\delta \frac{1-e^{-\epsilon \lc n p \rc}}{1-e^{-\epsilon}} )}{1+( \exp(\frac{s}{2}\log \frac{d-s}{s/2})-1)e^{-\epsilon \lc n p \rc} }
\end{align*}
where $|\mathcal{V}|\geq \exp(\frac{s}{2}\log \frac{d-s}{s/2})$ by Lemma \ref{dlemma:2}. Now we set 
\begin{equation*}
    p=\frac{1}{n\epsilon}\min \{\frac{s}{2}\log \frac{d-s}{s/2}-\epsilon, \log(\frac{1-e^{-\epsilon}}{4\delta e^\epsilon})\}. 
\end{equation*}
Thus we have 
\begin{equation*}
    \frac{1}{2}e^{-\epsilon \lc np \rc}-\delta \frac{1-e^{-\epsilon \lc n p \rc}}{1-e^{-\epsilon}}\geq \frac{1}{4}\exp(-\epsilon (np+1)), 
\end{equation*}
\begin{align*}
 &\frac{( \exp(\frac{s}{2}\log \frac{d-s}{s/2})-1)(\frac{1}{2}e^{-\epsilon \lc np \rc}-\delta \frac{1-e^{-\epsilon \lc n p \rc}}{1-e^{-\epsilon}} )}{1+( \exp(\frac{s}{2}\log \frac{d-s}{s/2})-1)e^{-\epsilon \lc n p \rc} }\\&\geq 
  \frac{( \exp(\frac{s}{2}\log \frac{d-s}{s/2})-1)(\frac{1}{2}e^{-\epsilon \lc np \rc}-\delta \frac{1}{1-e^{-\epsilon}} )}{1+( \exp(\frac{s}{2}\log \frac{d-s}{s/2})-1)e^{-\epsilon \lc n p \rc} }\\
  &\geq \frac{\frac{1}{4}\exp(-\epsilon (np+1))( \exp(\frac{s}{2}\log \frac{d-s}{s/2})-1)}{1+( \exp(\frac{s}{2}\log \frac{d-s}{s/2})-1)e^{-\epsilon (n p+1)}} \\
  &\geq \frac{1}{8}.
\end{align*}
Thus in total we have 
\begin{align*}
      &\mathcal{M}_n(\theta(\mathcal{P}), Q, \Phi\circ \rho)\\
      &\geq \frac{\tau}{4}\frac{1}{n\epsilon}\min \{\frac{s}{2}\log \frac{d-s}{s/2}-\epsilon, \log(\frac{1-e^{-\epsilon}}{4\delta e^\epsilon})\} \\
      &=\Omega (\frac{\tau \min \{s\log d, \log \frac{1}{\delta}\} }{n\epsilon}). 
\end{align*}
\end{proof}
\subsection*{Explicit Form of $\hat{C}(a, b)$ in (\ref{eq:4}) }
We first define the following notations:
\begin{align*}
   &V_- := \frac{\sqrt{2}-a}{b}, V_{+}=\frac{\sqrt{2}+a}{b} \\
   &F_{-}:= \Phi(-V_-), F_{+}:=\Phi(-V_+) \\
   &E_{-}:= \exp(-\frac{V^2_-}{2}), E_{+}:=\exp(-\frac{V^2_{+}}{2}),
\end{align*}
where $\Phi$ denotes the CDF of the standard Gaussian distribution. Then \begin{equation*}
    \hat{C}(a,b)=T_1+T_2+\cdots+T_5, 
\end{equation*}
where 
\begin{align*}
    &T_1:= \frac{2\sqrt{2}}{3}(F_{-}-F_{+}) \\
    & T_2:= -(a-\frac{a^3}{6})(F_{-}+F_{+}) \\ 
    & T_3:=\frac{b}{\sqrt{2\pi}}(1-\frac{a^2}{2})(E_{+}-E_{-})\\ 
    &T_4 : = \frac{ab^2}{2}\left(F_{+}+F_{-}+\frac{1}{\sqrt{2\pi}}(V_{+}E_{+}+V_{-}E_{-})\right) \\
    & T_5:= \frac{b^3}{6\sqrt{2\pi}}\left((2+V_{-}^2)E_{-}-(2+V_{+}^2)E_{+}\right). 
\end{align*}
\begin{proof}[{\bf Proof of Theorem \ref{thm:7}}]
Before the proof, let us first recall two lemmas related to the output of Algorithm \ref{alg:4}. 
\begin{lemma}[Lemma 3.4 in \cite{cai2019cost} ] \label{lemma:10}
Let $S$ and $\{w_i\}_{i=1}^s$ be defined is Algorithm \ref{alg:4}. For every $R_1\subseteq S$ and $R_2 \subseteq S^c$ such that $|R_1|=|R_2|$ and every $c>0$,  we have
\begin{equation*}
    \|v_{R_2}\|_2^2\leq (1+c)\|v_{R_1}\|_2^2+4(1+\frac{1}{c})\sum_{i\in [s]}\|w_i\|_\infty^2, 
\end{equation*}
where $v$ is the input vector of Algorithm \ref{alg:4}. 
\end{lemma}
\begin{lemma}[Lemma A.3 in \cite{cai2019cost} ]\label{lemma:11}
Consider in Algorithm \ref{alg:4} with input vector $\tilde{v}$ and the index set $S$. For any index set $I$, any $v \in \mathbb{R}^{|I|}$ which is a subvector of $\tilde{v}$  and $\hat{v}$ such that $\|\hat{v}\|_0\leq \hat{s}\leq s$, we have that for every $c>0$, 
\begin{equation*}
    \|v_S-v\|_2^2\leq (1+\frac{1}{c})\frac{|I|-s}{|I|-\hat{s}}\|\hat{v}-v\|_2^2+4(1+c)\sum_{i\in[s]}\|w_i\|_\infty^2. 
\end{equation*}
\end{lemma}

For simplicity we denote $\tilde{g}^t=\frac{1 }{m}\sum_{x\in \tilde{D}_t} \tilde{x}(\langle \tilde{x},  w^{t} \rangle-\tilde{y}) $,  $g^t=\mathbb{E}[x^T(\langle x, w^t \rangle -y)]=\nabla L_\mathcal{D}(w^t)$, $S^t=\text{supp}(w^t)$, $S^{t+1}=\text{supp}(w^{t+1})$, $S^*=\text{supp}(w^*)$ and $I^t=S^{t+1}\bigcup S^t \bigcup S^*$. We can see that $|S^t|\leq s$, $|S^{t+1}|\leq 2$ and $|I^t|\leq 2s+s^*$. We also denote $W^t=4\sum_{i\in [s]}\|w_i\|_\infty^2$, where $\{w_i\}$ are the vectors  in Algorithm \ref{alg:4} in the $t$-th iteration. We let  $\gamma=\lambda_{\max} (\mathbb{E}[xx^T])$, $\mu=\lambda_{\min} (\mathbb{E}[xx^T])$ and $\eta_0=\frac{\eta}{\gamma}$ for some $\eta$.

Then the smooth Lipschitz property we have 
\begin{align}
  & \LD-\LT \nm \\
  &\leq \lge \wD-\wT, g^t \rge + \frac{\gamma}{2}\|\wD-\wT\|_2^2 \nm \\
   &=\lge \wD_{I^t}-\wT_{I^t}, g^t_{I^t}\rge + \frac{\gamma}{2 }\| \wD_{I^t}-\wT_{I^t}\|_2^2 \nm \\
   &\leq \frac{\gamma}{2}\|\wD_{I^t}-\wT_{I^t}+\frac{\eta}{\gamma} g^t_{I^t}\|_2^2-\frac{\eta^2}{2\gamma}\|g^t_{I^t}\|_2^2+(1-\eta)\lge \wD-\wT, g^t \rge \label{aeq:46}
\end{align}
First, let us focus on the third term of (\ref{aeq:46})
\begin{align}
    \lge \wD-\wT, g^t \rge &= \lge \wD_{\sunion}-\wT_{\sunion}, g^t_{\sunion} \rge \nm \\
    &=  \lge \wD_{S^{t+1}}-\wT_{S^{t+1}}, g^t_{S^{t+1}}  \rge+ \lge \wD_\sdiffb-\wT_\sdiffb, g^t_\sdiffb \rge \nm \\
    &=  \lge \wD_{S^{t+1}}-\wT_{S^{t+1}}, g^t_{S^{t+1}}  \rge- \lge \wT_\sdiffb, g^t_\sdiffb \rge. \label{aeq:47}
\end{align}
From the definition we know that $\wD=\hatwD+\tilde{w}_{S^{t+1}}$, where $\hatwD=(w^t-\eta_0\tilde{g}^t)_{S^{t+1}}$. Thus, 
\begin{align}\label{aeq:48}
\lge \wD-\wT, g^t \rge= \lge \hatwD_{S^{t+1}}-\wT_{S^{t+1}}, g^t_{S^{t+1}} \rge+\lge \tilde{w}_{S^{t+1}},g^t_{S^{t+1}}\rge -\lge \wT_\sdiffb-g^t_\sdiffb \rge.
\end{align}
For the first term in (\ref{aeq:48}) we have 
\begin{align}
    \lge \hatwD_{S^{t+1}}-\wT_{S^{t+1}}, g^t_{S^{t+1}} \rge &= \lge -\eta_0\tilde{g}^t_{S^{t+1}}, g^t_{S^{t+1}} \rge =-\frac{\eta}{\gamma} \lge \tilde{g}^t_{S^{t+1}}, g^t_{S^{t+1}} \rge \nm \\
    &= -\frac{\eta}{\gamma} \|g^t_{S^{t+1}}\|_2^2-\frac{\eta}{\gamma} \langle \tilde{g}^t_{S^{t+1}}-g^t_{S^{t+1}}, g^t_{S^{t+1}} \rge \nm \\
    &\leq -\frac{\eta}{\gamma} \|g^t_{S^{t+1}}\|_2^2+ \frac{\eta}{2\gamma} \|g^t_{S^{t+1}}\|_2^2+ \frac{\eta}{2\gamma}\|\tilde{g}^t_{S^{t+1}}-g^t_{S^{t+1}}\|_2^2 \nm \\
    &= -\frac{\eta}{2\gamma} \|g^t_{S^{t+1}}\|_2^2+\frac{\eta}{2\gamma}\|\tilde{g}^t_{S^{t+1}}-g^t_{S^{t+1}}\|_2^2 \label{aeq:49}. 
\end{align}
Take (\ref{aeq:49}) into (\ref{aeq:48}) we have for $c>1$
\begin{multline}\label{aeq:50} 
   \lge \wD-\wT, g^t \rge\leq  -\frac{\eta}{2\gamma} \|g^t_{S^{t+1}}\|_2^2+\frac{\eta}{2\gamma}\|\tilde{g}^t_{S^{t+1}}-g^t_{S^{t+1}}\|_2^2+c\|\tilde{w}_{S^{t+1}}\|_2^2
   +\frac{1}{4c}\|g^t_{S^{t+1}}\|_2^2-\lge \wT_\sdiffb-g^t_\sdiffb \rge. 
\end{multline}
For the last term of (\ref{aeq:50}) we have 
\begin{align}
    -\lge \wT_\sdiffb-g^t_\sdiffb \rge &\leq \frac{\gamma}{2\eta}(\|\wT_\sdiffb-\frac{\eta}{\gamma}g^t_\sdiffb\|_2^2-(\frac{\eta}{\gamma})^2 \|g^t_\sdiffb\|_2^2) \nm \\
    &=\frac{\gamma}{2\eta} \|\wT_\sdiffb-\frac{\eta}{\gamma}g^t_\sdiffb\|_2^2-\frac{\eta}{2\gamma} \|g^t_\sdiffb\|_2^2.\label{aeq:51} 
\end{align}
In Lemma \ref{lemma:10}, let $v=\wT-\frac{\eta}{\gamma}\tilde{g}^t$, $R_2=S^t\backslash S^{t+1}$ and $R_1=S^{t+1}\backslash S^t$. We have for $c>1$
\begin{equation*}
    \|\wT_\sdiffb-\frac{\eta}{\gamma}\tilde{g}^t_\sdiffb\|_2^2\leq (1+\frac{1}{c}) \|\wT_\sdiffa-\frac{\eta}{\gamma}\tilde{g}^t_\sdiffa\|_2^2+(1+c)W^t.
\end{equation*}
Since for every $c>1$, $(1-\frac{1}{c}){\|a\|^2}-(c-1)\|b\|_2^2\leq \|a+b\|^2\leq  (1+\frac{1}{c})\|a\|_2^2+ (1+c) \|b\|_2^2$ we have 
\begin{align}
&(1-\frac{1}{c}){ \|\wT_\sdiffb-\frac{\eta}{\gamma}{g}^t_\sdiffb\|_2^2}- (c-1)\frac{\eta^2}{\gamma^2}\|{g}^t_\sdiffb-\tilde{g}^t_\sdiffb\|_2^2 \nm \\
&\leq  \|\wT_\sdiffb-\frac{\eta}{\gamma}\tilde{g}^t_\sdiffb\|_2^2
\leq (1+\frac{1}{c}) \|\wT_\sdiffa-\frac{\eta}{\gamma}\tilde{g}^t_\sdiffa\|_2^2+(1+c)W^t \nm \\
&\leq  (1+\frac{1}{c})[ (1+1/c)\|\wT_\sdiffa-\frac{\eta}{\gamma}{g}^t_\sdiffa\|_2^2+(1+c) \frac{\eta^2}{\gamma^2}\|{g}^t_\sdiffb-\tilde{g}^t_\sdiffb\|_2^2]+2(1+c)W^t. \label{newaeq:1}
\end{align}
That is 
\begin{multline*}
     \|\wT_\sdiffb-\frac{\eta}{\gamma}{g}^t_\sdiffb\|_2^2\leq \frac{(c+1)^2}{c(c-1)}\|\wT_\sdiffa-\frac{\eta}{\gamma}{g}^t_\sdiffa\|_2^2 \\+ (c+\frac{(c+1)^2}{c})\frac{\eta}{2\gamma}(\|{g}^t_\sdiffb-\tilde{g}^t_\sdiffb\|_2^2+\|{g}^t_\sdiffa-\tilde{g}^t_\sdiffa\|_2^2 )+ \frac{c(1+c)}{c-1}W^t. 
\end{multline*}

Thus
\begin{multline}
    -\lge \wT_\sdiffb-g^t_\sdiffb \rge \leq \frac{(c+1)^2}{2c(c-1)} \frac{\eta }{\gamma}\|g^t_{S^{t+1}\backslash S^t}\|_2^2\\+\frac{\gamma c(1+c)}{2\eta(c-1) }W^t-\frac{\eta}{2\gamma} \|g^t_\sdiffb\|_2^2 +\frac{(2c+3)\eta}{2\gamma}(\|{g}^t_\sdiffb-\tilde{g}^t_\sdiffb\|_2^2+\|{g}^t_\sdiffa-\tilde{g}^t_\sdiffa\|_2^2 ). \nm
\end{multline}
Thus in (\ref{aeq:50}) we have 
\begin{align}
      & \lge \wD-\wT, g^t \rge\leq -\frac{\eta}{2\gamma} \|g^t_{S^{t+1}}\|_2^2+\frac{\eta}{2\gamma}\|\tilde{g}^t_{S^{t+1}}-g^t_{S^{t+1}}\|_2^2+c\|\tilde{w}_{S^{t+1}}\|_2^2 \nm \\ 
   & \qquad +\frac{1}{4c}\|g^t_{S^{t+1}}\|_2^2-\lge \wT_\sdiffb-g^t_\sdiffb \rge \nm \\ 
   &\leq -\frac{\eta}{2\gamma} \|g^t_{S^{t+1}}\|_2^2+\frac{\eta}{2\gamma}\|\tilde{g}^t_{S^{t+1}}-g^t_{S^{t+1}}\|_2^2+c\|\tilde{w}_{S^{t+1}}\|_2^2 +\frac{1}{4c}\|g^t_{S^{t+1}}\|_2^2-\frac{\eta}{2\gamma} \|g^t_{S^{t+1}}\|_2^2  \nm \\
   &+ \frac{(c+1)^2}{2c(c-1)} \frac{\eta }{\gamma}\|g^t_{S^{t+1}\backslash S^t}\|_2^2+\frac{(2c+3)\eta}{2\gamma}(\|{g}^t_\sdiffb-\tilde{g}^t_\sdiffb\|_2^2+\|{g}^t_\sdiffa-\tilde{g}^t_\sdiffa\|_2^2 ) \nm  \\
   &=\frac{\eta }{2\gamma}\|g^t_{S^{t+1}\backslash S^t}\|_2^2+\frac{\eta}{2\gamma}\frac{3c+1}{c(c-1)}\|g^t_{S^{t+1}\backslash S^t}\|_2^2-\frac{\eta}{2\gamma} \|g^t_\sdiffb\|_2^2 -\frac{\eta}{2\gamma} \|g^t_{S^{t+1}}\|_2^2 \nm \\
   &\qquad+ \frac{1}{4c}\|g^t_{S^{t+1}}\|_2^2+ \frac{\gamma c(1+c)}{2\eta(c-1) }W^t+ \frac{\eta}{2\gamma}\|\tilde{g}^t_{S^{t+1}}-g^t_{S^{t+1}}\|_2^2+c\|\tilde{w}_{S^{t+1}}\|_2^2\nm \\
   &\qquad +\frac{(2c+3)\eta}{\gamma}(\|{g}^t_\sdiffb-\tilde{g}^t_\sdiffb\|_2^2+\|{g}^t_\sdiffa-\tilde{g}^t_\sdiffa\|_2^2 ) \nm \\ 
     &=\frac{\eta }{2\gamma}\|g^t_{S^{t+1}\backslash S^t}\|_2^2-\frac{\eta}{2\gamma} \|g^t_\sdiffb\|_2^2 -\frac{\eta}{2\gamma} \|g^t_{S^{t+1}}\|_2^2 +\frac{1}{c}(\frac{1}{4}+\frac{\eta}{2\gamma}+\frac{\eta}{2\gamma}\frac{3c+1}{(c-1)} )\|g^t_{S^{t+1}}\|_2^2\nm \\
   & +\underbrace{\frac{\gamma c(1+c)}{2\eta(c-1) }W^t+ \frac{\eta}{2\gamma}\|\tilde{g}^t_{S^{t+1}}-g^t_{S^{t+1}}\|_2^2+c\|\tilde{w}_{S^{t+1}}\|_2^2+\frac{(2c+3)\eta}{\gamma}(\|{g}^t_\sdiffb-\tilde{g}^t_\sdiffb\|_2^2+\|{g}^t_\sdiffa-\tilde{g}^t_\sdiffa\|_2^2 )}_{N^t}\nm \\
   &\leq \frac{\eta }{2\gamma}\|g^t_{S^{t+1}\backslash S^t}\|_2^2-\frac{\eta}{2\gamma} \|g^t_\sdiffb\|_2^2 -\frac{\eta}{2\gamma} \|g^t_{S^{t+1}}\|_2^2+C_1\frac{\eta}{\gamma c}\|g^t_{S^{t+1}}\|_2^2+N^t,  \label{aeq:52}
\end{align}
where $C_1>0$ is some constant. We can easily see that 
\begin{align*}
    \frac{\eta }{2\gamma}\|g^t_{S^{t+1}\backslash S^t}\|_2^2-\frac{\eta}{2\gamma} \|g^t_\sdiffb\|_2^2 -\frac{\eta}{2\gamma} \|g^t_{S^{t+1}}\|_2^2&= -\frac{\eta}{2\gamma} \|g^t_\sdiffb\|_2^2-\frac{\eta}{2\gamma} \|g^t_{S^{t+1}\bigcap S^t}\|_2^2 \\
    &=-\frac{\eta}{2\gamma} \|g^t_{S^{t+1}\bigcup S^t}\|_2^2 
\end{align*}
In total 
\begin{equation}\label{aeq:53}
        \lge \wD-\wT, g^t \rge\leq -\frac{\eta}{2\gamma} \|g^t_{S^{t+1}\bigcup S^t}\|_2^2 +C_1\frac{\eta}{\gamma c}\|g^t_{S^{t+1}}\|_2^2+N_1. 
\end{equation}

Take (\ref{aeq:53}) into (\ref{aeq:46}) we have 
\begin{align}
&\LD-\LT
   \leq \frac{\gamma}{2}\|\wD_{I^t}-\wT_{I^t}+\frac{\eta}{\gamma} g^t_{I^t}\|_2^2-\frac{\eta^2}{2\gamma}\|g^t_{I^t}\|_2^2+(1-\eta)\lge \wD-\wT, g^t \rge  \nm \\
   &\leq \frac{\gamma}{2}\|\wD_{I^t}-\wT_{I^t}+\frac{\eta}{\gamma} g^t_{I^t}\|_2^2-\frac{\eta^2}{2\gamma}\|g^t_{I^t}\|_2^2-\frac{(1-\eta)\eta}{2\gamma} \|g^t_{S^{t+1}\bigcup S^t}\|_2^2  +C_1\frac{(1-\eta)}{c}\frac{\eta}{\gamma}\|g^t_{S^{t+1}}\|_2^2 \nm \\& \qquad +(1-\eta)N^t \nm \\
   & \leq \frac{\gamma}{2}\|\wD_{I^t}-\wT_{I^t}+\frac{\eta}{\gamma} g^t_{I^t }\|_2^2-\frac{\eta^2}{2\gamma}\|g^t_{I^t\backslash (S^{t}\bigcup S^*)}\|_2^2- \frac{\eta^2}{2\gamma}\|g^t_{(S^{t}\bigcup S^*)}\|_2^2 -\frac{(1-\eta)\eta}{2\gamma} \|g^t_{S^{t+1}\bigcup S^t}\|_2^2 \nm \\& +C_1\frac{(1-\eta)}{c}\frac{\eta}{\gamma}\|g^t_{S^{t+1}}\|_2^2 \nm  +(1-\eta)N^t \nm  \\
   & \leq \frac{\gamma}{2}\|\wD_{I^t}-\wT_{I^t}+\frac{\eta}{\gamma} g^t_{I^t }\|_2^2-\frac{\eta^2}{2\gamma}\|g^t_{I^t\backslash (S^{t}\bigcup S^*)}\|_2^2- \frac{\eta^2}{2\gamma}\|g^t_{(S^{t}\bigcup S^*)}\|_2^2  \nm \\& -\frac{(1-\eta)\eta}{2\gamma} \|g^t_{S^{t+1}\backslash(S^* \bigcup S^t)}\|_2^2+C_1\frac{(1-\eta)}{c}\frac{\eta}{\gamma}\|g^t_{S^{t+1}}\|_2^2   +(1-\eta)N^t, \label{aeq:54}
\end{align}
where the last inequality is due to $S^{t+1}\backslash(S^* \bigcup S^t)\subseteq S^{t+1}\bigcup S^t$. Next we will analyze the term $ \frac{\gamma}{2}\|\wD_{I^t}-\wT_{I^t}+\frac{\eta}{\gamma} g^t_{I^t }\|_2^2-\frac{\eta^2}{2\gamma}\|g^t_{I^t\backslash (S^{t}\bigcup S^*)}\|_2^2$ in (\ref{aeq:54}). Let $R$ be a subset of $\sdiffb$ such that $|R|=|I^t\backslash (S^*\bigcup S^t)|=|S^{t+1}\backslash (S^t\bigcup S^*)|$. In Lemma \ref{lemma:10}, we take  $v=\wT-\frac{\eta}{\gamma}\tilde{g}^t$, $R_2=R$ and $R_1=I^t\backslash (S^*\bigcup S^t)$ we have for $c>1$, 
\begin{equation}\label{aeq:55}
    \|w^t_R-\frac{\eta}{\gamma}\tilde{g}^t_R\|_2^2\leq (1+c)\|(w^t-\frac{\eta}{\gamma}\tilde{g}^t)_{I^t\backslash (S^*\bigcup S^t)}\|_2^2+(1+\frac{1}{c})W^t. 
\end{equation}
Just as in (\ref{newaeq:1}) we have is for $c>1$, 
\begin{equation}\label{aeq:56}
   (\frac{\eta}{\gamma})^2 \|g^t_{I^t\backslash (S^*\bigcup S^t)}\|_2^2\geq (1-\frac{1}{c})\|w^t_R-\frac{\eta}{\gamma}g^t_R\|_2^2-cW^t-c\frac{\eta^2}{\gamma^2}(\|\tilde{g}^t_R-g^t_R\|_2^2+\|g^t_{I^t\backslash (S^*\bigcup S^t)}-\tilde{g}^t_{I^t\backslash (S^*\bigcup S^t)}\|_2^2). 
\end{equation}
Then we have 
\begin{align}
    &\frac{\gamma}{2}\|\wD_{I^t}-\wT_{I^t}+\frac{\eta}{\gamma} g^t_{I^t }\|_2^2-\frac{\eta^2}{2\gamma}\|g^t_{I^t\backslash (S^{t}\bigcup S^*)}\|_2^2\nm \\
    &\leq \frac{\gamma}{2}\|\tilde{w}_{S^{t+1}}\|_2^2+\frac{\gamma}{2}\|\hatwD_{I^t}-\wT_{I^t}+\frac{\eta}{\gamma} g^t_{I^t }\|_2^2-\frac{\gamma}{2}(1-\frac{1}{c})\|w^t_R-\frac{\eta}{\gamma}g^t_R\|_2^2+\frac{\gamma c}{2}W^t \nm \\
    &+c\frac{\eta^2}{2\gamma}(\|\tilde{g}^t_R-g^t_R\|_2^2+\|g^t_{I^t\backslash (S^*\bigcup S^t)}-\tilde{g}^t_{I^t\backslash (S^*\bigcup S^t)}\|_2^2)+c\frac{\eta^2}{\gamma^2}(\|\tilde{g}^t_R-g^t_R\|_2^2 \label{aeq:57} \\
    &= \frac{\gamma}{2}\|\hatwD_{I^t}-\wT_{I^t}+\frac{\eta}{\gamma} g^t_{I^t }\|_2^2- \frac{\gamma}{2}\|\hatwD_R- w^t_R+\frac{\eta}{\gamma}g^t_R\|_2^2+\frac{\gamma}{2}\|\tilde{w}_{S^{t+1}}\|_2^2+ \frac{\gamma c}{2}W^t  \nm \\
    & + \frac{\gamma}{2c}\|w^t_R-\frac{\eta}{\gamma}g^t_R\|_2^2+c\frac{\eta^2}{2\gamma}(\|\tilde{g}^t_R-g^t_R\|_2^2+\|g^t_{I^t\backslash (S^*\bigcup S^t)}-\tilde{g}^t_{I^t\backslash (S^*\bigcup S^t)}\|_2^2) \label{aeq:58} \\
    &\leq \frac{\gamma}{2}\|\hatwD_{I^t\backslash R}-\wT_{I^t\backslash R}+\frac{\eta}{\gamma} g^t_{I^t\backslash R}\|_2^2  + \frac{\gamma}{2c}(1+\frac{1}{c})\|\frac{\eta}{\gamma}g^t_{I^t\backslash (S^*\bigcup S^t)}\|_2^2  +\frac{\gamma}{2}\|\tilde{w}_{S^{t+1}}\|_2^2 \nm \\
    &+ \underbrace{\frac{\gamma c}{2}W^t+\frac{\gamma}{2c}(1+c)W^t +C_2c\frac{\eta^2}{2\gamma}(\|\tilde{g}^t_R-g^t_R\|_2^2+\|g^t_{I^t\backslash (S^*\bigcup S^t)}-\tilde{g}^t_{I^t\backslash (S^*\bigcup S^t)}\|_2^2)}_{N_1^t}.\label{aeq:59}
\end{align}
(\ref{aeq:57}) is due to that $[\hatwD_{I^t}- (\wT_{I^t}-\frac{\eta}{\gamma} g^t_{I^t })]_{S^{t+1}}=0$, thus $\lge \tilde{w}_{S^{t+1}}, \hatwD_{I^t}- (\wT_{I^t}-\frac{\eta}{\gamma} g^t_{I^t }) \rge=0$ and (\ref{aeq:56}). (\ref{aeq:58}) is due to  $\hatwD_R=0$, (\ref{aeq:59}) is due to (\ref{aeq:55}) by the same technique as in (\ref{newaeq:1}) and $w^t_{I^t\backslash (S^*\bigcup S^t)}=0$. In the following we will consider the first term in (\ref{aeq:59}). 

In Lemma \ref{lemma:11}, take $v=\wT_{I^t\backslash R}-\frac{\eta}{\gamma} \tilde{g}^t_{I^t\backslash R}$, $\hat{v}=w^*$, $S=S^{t+1}$ we have for all $c>1$ 
\begin{equation*}
    \|\hatwD_{I^t\backslash R}-\wT_{I^t\backslash R}+\frac{\eta}{\gamma} \tilde{g}^t_{I^t\backslash R}\|_2^2\leq (1+\frac{1}{c})\frac{|I^t\backslash R|-s}{|I^t\backslash R|-s^*}\|w^*-\wT_{I^t\backslash R}+\frac{\eta}{\gamma} \tilde{g}^t_{I^t\backslash R}\|_2^2+(1+c)W^t. 
\end{equation*}
Then we have 
\begin{align*}
&(1-\frac{1}{c})\|\hatwD_{I^t\backslash R}-\wT_{I^t\backslash R}+\frac{\eta}{\gamma} {g}^t_{I^t\backslash R}\|_2^2-(c-1)\frac{\eta^2}{\gamma^2}\|g^t_{I^t\backslash R}-\tilde{g}^t_{I^t\backslash R}\|_2^2 \nm \\&\leq 
\|\hatwD_{I^t\backslash R}-\wT_{I^t\backslash R}+\frac{\eta}{\gamma} \tilde{g}^t_{I^t\backslash R}\|_2^2 \\
&\leq (1+\frac{1}{c})\frac{|I^t\backslash R|-s}{|I^t\backslash R|-s^*}\|w^*-\wT_{I^t\backslash R}+\frac{\eta}{\gamma} \tilde{g}^t_{I^t\backslash R}\|_2^2+(1+c)W^t \\
&\leq (1+\frac{1}{c})\frac{|I^t\backslash R|-s}{|I^t\backslash R|-s^*}[(1+\frac{1}{c})\|w^*-\wT_{I^t\backslash R}+\frac{\eta}{\gamma} {g}^t_{I^t\backslash R}\|_2^2 +(1+c)\frac{\eta^2}{\gamma^2}\|{g}^t_{I^t\backslash R}-\tilde{g}^t_{I^t\backslash R} \|_2^2 ]+(1+c)W^t 
\end{align*}
That is 
\begin{multline*}
    \|\hatwD_{I^t\backslash R}-\wT_{I^t\backslash R}+\frac{\eta}{\gamma} {g}^t_{I^t\backslash R}\|_2^2\leq \frac{(c+1)^2}{c(c-1)} \frac{2s^*}{s+s^*}\|w^*-\wT_{I^t\backslash R}+\frac{\eta}{\gamma} g^t_{I^t\backslash R}\|_2^2\\+ \frac{(c+1)^2}{c-1}  \frac{\eta^2}{\gamma^2}\|{g}^t_{I^t\backslash R}-\tilde{g}^t_{I^t\backslash R} \|_2^2 + c \frac{\eta^2}{\gamma^2}\|g^t_{I^t\backslash R}-\tilde{g}^t_{I^t\backslash R}\|_2^2+\frac{(1+c)c}{c-1}W^t
\end{multline*}

Take $c\geq \frac{\sqrt{3}}{\sqrt{3}-\sqrt{2}}$, and since $|I^t\backslash R|\leq 2s^*+s$, we have 
\begin{align}\label{aeq:60}
      &\|\hatwD_{I^t\backslash R}-\wT_{I^t\backslash R}+\frac{\eta}{\gamma} g^t_{I^t\backslash R}\|_2^2\leq \frac{3}{2}\frac{2s^*}{s+s^*}\|w^*-\wT_{I^t\backslash R}+\frac{\eta}{\gamma} g^t_{I^t\backslash R}\|_2^2\nm \\
      &+\underbrace{C_3 c(\|{g}^t_{I^t\backslash R}-\tilde{g}^t_{I^t\backslash R} \|_2^2+\|g^t_{I^t\backslash R}-\tilde{g}^t_{I^t\backslash R}\|_2^2+W^t)}_{N_3^t}.  
\end{align}
 
Take (\ref{aeq:60}) into (\ref{aeq:59}) we have
\begin{align}
      &\frac{\gamma}{2}\|\wD_{I^t}-\wT_{I^t}+\frac{\eta}{\gamma} g^t_{I^t }\|_2^2-\frac{\eta^2}{2\gamma}\|g^t_{I^t\backslash (S^{t}\bigcup S^*)}\|_2^2\nm \\
      &\leq \frac{3\gamma s^*}{2(s+s^*)}\|w^*-\wT_{I^t\backslash R}+\frac{\eta}{\gamma} g^t_{I^t\backslash R}\|_2^2+ \frac{\gamma}{2c}(1+\frac{1}{c})\|\frac{\eta}{\gamma}g^t_{I^t\backslash (S^*\bigcup S^t)}\|_2^2+\frac{\gamma}{2}\|\tilde{w}_{S^{t+1}}\|_2^2  \nm \\
    &  + N_1^t+N_3^t\\
    &=  \frac{3\gamma s^*}{2(s+s^*)}\|w^*-\wT_{I^t\backslash R}+\frac{\eta}{\gamma} g^t_{I^t\backslash R}\|_2^2+ \frac{\gamma}{2c}(1+\frac{1}{c})\|\frac{\eta}{\gamma}g^t_{S^{t+1}}\|_2^2  \nm \\
    &\qquad  +\frac{\gamma}{2}\|\tilde{w}_{S^{t+1}}\|_2^2+ N_1^t+N_3^t \\
    &= \frac{3 s^*}{s+s^*}(
    \eta \lge  w^*-\wT, g^t \rge +\frac{\gamma}{2}\|w^*-\wT\|_2^2+\frac{\eta^2}{2c \gamma }\|g^t_{I^t}\|_2^2)+ \frac{\eta^2}{2c\gamma}(1+\frac{1}{c})\|g^t_{S^{t+1}}\|_2^2  \nm \\
    &\qquad  +\frac{\gamma}{2}\|\tilde{w}_{S^{t+1}}\|_2^2+  N_1^t+N_3^t  \\
    &\leq \frac{3 s^*}{s+s^*}(
    \eta (L_\mathcal{D}(w^*)-\LT
    )+\frac{\gamma-\eta \mu}{2}\|w^*-\wT\|_2^2+\frac{\eta^2}{2c \gamma }\|g^t_{I^t}\|_2^2)+ \frac{\eta^2}{2c\gamma}(1+\frac{1}{c})\|g^t_{S^{t+1}}\|_2^2  \nm \\
    &\qquad  +\underbrace{\frac{\gamma}{2}\|\tilde{w}_{S^{t+1}}\|_2^2+  N_1^t+N_3^t }_{N_2^t}.  \label{aeq:64}
\end{align}
Take (\ref{aeq:64}) into (\ref{aeq:54}) we have
\begin{align}
    &\LD-\LT\leq  \frac{\gamma}{2}\|\wD_{I^t}-\wT_{I^t}+\frac{\eta}{\gamma} g^t_{I^t }\|_2^2-\frac{\eta^2}{2\gamma}\|g^t_{I^t\backslash (S^{t}\bigcup S^*)}\|_2^2- \frac{\eta^2}{2\gamma}\|g^t_{(S^{t}\bigcup S^*)}\|_2^2  \nm \\& -\frac{(1-\eta)\eta}{2\gamma} \|g^t_{S^{t+1}\backslash(S^* \bigcup S^t)}\|_2^2+C_1\frac{(1-\eta)}{c}\frac{\eta}{\gamma}\|g^t_{S^{t+1}}\|_2^2   +(1-\eta)N^t \nm \\
    &\leq \frac{3 s^*}{s+s^*}(
    \eta (L_\mathcal{D}(w^*)-\LT)+\frac{\gamma-\eta \mu}{2}\|w^*-\wT\|_2^2+\frac{\eta^2}{2c \gamma }\|g^t_{I^t}\|_2^2)+ \frac{\eta^2}{2c\gamma}(1+\frac{1}{c})\|g^t_{S^{t+1}}\|_2^2   \nm \\& - \frac{\eta^2}{2\gamma}\|g^t_{(S^{t}\bigcup S^*)}\|_2^2-\frac{(1-\eta)\eta}{2\gamma} \|g^t_{S^{t+1}\backslash(S^* \bigcup S^t)}\|_2^2+C_1\frac{(1-\eta)}{c}\frac{\eta}{\gamma}\|g^t_{S^{t+1}}\|_2^2 \|g^t_{S^{t+1}}\|_2^2   +(1-\eta)N^t+N^t_2. \label{aeq:65}
\end{align}
We have when $c\rightarrow \infty$
\begin{align}
    \frac{\eta^2}{2c\gamma}(1+\frac{1}{c})\|g^t_{S^{t+1}}\|_2^2 +C_1\frac{(1-\eta)}{c}\frac{\eta}{\gamma}\|g^t_{S^{t+1}}\|_2^2  \rightarrow 0. \nm
\end{align}
Thus, if $\eta \geq \frac{1}{2}$ there must exits a sufficient large $c$ such that 
\begin{align}
    &\frac{\eta^2}{2c\gamma}(1+\frac{1}{c})\|g^t_{S^{t+1}}\|_2^2 +\frac{(1-\eta)}{c}(\frac{1}{4}+\frac{\eta}{2\gamma} )\|g^t_{S^{t+1}}\|_2^2 \leq \frac{\eta(1-\eta)}{4\gamma}\|g^t_{S^{t+1}}\|_2^2 \nm \\
    &\leq  \frac{\eta^2}{4\gamma}\|g^t_{(S^{t}\bigcup S^*)}\|_2^2+ \frac{(1-\eta)\eta}{4\gamma} \|g^t_{S^{t+1}\backslash(S^* \bigcup S^t)}\|_2^2
\end{align}
Thus, 
\begin{align}
    &\LD-\LT\leq  \frac{\gamma}{2}\|\wD_{I^t}-\wT_{I^t}+\frac{\eta}{\gamma} g^t_{I^t }\|_2^2-\frac{\eta^2}{2\gamma}\|g^t_{I^t\backslash (S^{t}\bigcup S^*)}\|_2^2- \frac{\eta^2}{2\gamma}\|g^t_{(S^{t}\bigcup S^*)}\|_2^2  \nm \\& -\frac{(1-\eta)\eta}{2\gamma} \|g^t_{S^{t+1}\backslash(S^* \bigcup S^t)}\|_2^2+C_1\frac{(1-\eta)}{c}\frac{\eta}{\gamma}\|g^t_{S^{t+1}}\|_2^2    +(1-\eta)N^t \nm \\
    &\leq \frac{3 s^*}{s+s^*}(
    \eta (L_\mathcal{D}(w^*)-\LT)+\frac{\gamma-\eta \mu}{2}\|w^*-\wT\|_2^2+\frac{\eta^2}{2c \gamma }\|g^t_{I^t}\|_2^2)  \nm \\& - \frac{\eta^2}{4\gamma}\|g^t_{(S^{t}\bigcup S^*)}\|_2^2-\frac{(1-\eta)\eta}{4\gamma} \|g^t_{S^{t+1}\backslash(S^* \bigcup S^t)}\|_2^2   +(1-\eta)N^t+N^t_2.
\end{align}
Take $\eta=\frac{2}{3}$, $s=72\frac{\gamma^2}{\mu^2}s^*$ so that $\frac{3 s^*}{s+s^*}\leq \frac{\mu^2}{24\gamma(\gamma-\eta \mu)}\leq \frac{1}{8}$. We have 
\begin{align}
     &\LD-\LT\leq \frac{3 s^*}{s+s^*}(
    \eta (L_\mathcal{D}(w^*)-\LT)+\frac{\gamma-\eta \mu}{2}\|w^*-\wT\|_2^2+\frac{\eta^2}{2c \gamma }\|g^t_{I^t}\|_2^2)  \nm \\& - \frac{\eta^2}{4\gamma}\|g^t_{(S^{t}\bigcup S^*)}\|_2^2-\frac{(1-\eta)\eta}{4\gamma} \|g^t_{S^{t+1}\backslash(S^* \bigcup S^t)}\|_2^2   +(1-\eta)N^t+N^t_2 \nm \\
    &\leq \frac{2 s^*}{s+s^*}(L_\mathcal{D}(w^*)-\LT)+ \frac{\mu^2}{48\gamma}\|w^*-w^t\|_2^2+\frac{1}{36\gamma}\|g^t_{I^t}\|_2^2 \nm \\
    &\qquad -\frac{1}{9\gamma}\|g^t_{(S^{t}\bigcup S^*)}\|_2^2-\frac{1}{18\gamma}\|g^t_{S^{t+1}\backslash(S^* \bigcup S^t)}\|_2^2+O(N^t+N^t_2) \nm \\
    &\leq  \frac{2 s^*}{s+s^*}(L_\mathcal{D}(w^*)-\LT)-\frac{3}{36\gamma}(\|g^t_{(S^{t}\bigcup S^*)}\|_2^2-\frac{\mu^2}{4}\|w^*-w^t\|_2^2)+O(N^t+N^t_2) \label{aeq:68}\\
    &\leq (\frac{2 s^*}{s+s^*}+\frac{\mu}{24\gamma})(L_\mathcal{D}(w^*)-\LT)+O(N^t+N^t_2) \label{aeq:69}. 
\end{align}
Where (\ref{aeq:68}) is due to the following lemma: 
\begin{lemma}\label{lemma:12}[Lemma 6 in \cite{jain2014iterative}]
\begin{equation}
    |g^t_{(S^{t}\bigcup S^*)}\|_2^2-\frac{\mu^2}{4}\|w^*-w^t\|_2^2\geq \frac{\mu}{2}(\LT-L_\mathcal{D}(w^*)).
\end{equation}
\end{lemma}
Thus 
\begin{equation*}
   \LD- L_\mathcal{D}(w^*)\leq (1-\frac{5}{72}\frac{\mu}{\gamma})( \LT- L_\mathcal{D}(w^*))+O(N^t+N^t_2).
\end{equation*}
Where 
\begin{align}
           &N^t+N^t_2\leq O(\sum_{i\in [s]}\|w_i\|_\infty^2+(2s+s^*)\|\tilde{g}^t-g^t\|_\infty^2+s\|\tilde{w}\|_\infty^2) \label{aeq:71},
\end{align}
where  each coordinate of $w_i, \tilde{w}$ sampled from $ \sim \text{Lap}(O(\frac{K^2\eta_0sT\sqrt{\log\frac{1}{\delta}}}{n\epsilon}) )$. 
 We first bound the term of $\sum_{i\in [s]}\|w_i\|_\infty^2+s\|\tilde{w}\|_\infty^2$. We recall the following lemma: 
 \begin{lemma}[Lemma A.1 in \cite{cai2019cost}]\label{lemma:13}
 For a random vector $w\in \mathbb{R}^d$, where $w_i\sim \text{Lap}(\lambda)$, then for any $\zeta>0$, 
 \begin{equation*}
     \text{Pr}(\|w\|^2_\infty\geq 4\lambda^2\log^2 \frac{1}{\zeta}\log^2 k )\leq \zeta.
 \end{equation*}
 \end{lemma}
 Take the union we have with probability at least $1-\zeta$, $$\sum_{i\in [s]}\|w_i\|_\infty^2+s\|\tilde{w}\|_\infty^2\leq O(\frac{K^4\eta_0^2s^3T^2\log^2 \frac{s}{\zeta}\log\frac{1}{\delta}}{n^2\epsilon^2}).$$
 For $s\|\tilde{g}^t-g^t\|_\infty^2$, follow the proof of Lemma \ref{lemma:7}, we can see that when  $K=\frac{{(n\epsilon)^{\frac{1}{4}}}}{(Ts)^\frac{1}{4}}$ with probability at least $1-\zeta$, 
$$s\|\tilde{g}^t-g^t\|_\infty^2\leq O(\frac{sTM\log \frac{d}{\zeta}}{n}+\frac{sT^2K^4\log^2 \frac{d}{\zeta}}{n^2}+ \frac{sM^2}{K^4})=O(\frac{sTM\log \frac{d}{\zeta}}{n\epsilon}).$$ In total, with probability $1-\zeta$ for each $t\in [T]$,
\begin{equation}
    N^t+N^t_2=O(\frac{\eta^2s^2T\log^2 \frac{dT}{\zeta}\log\frac{1}{\delta}}{\gamma^2 n\epsilon}+\frac{sTM\log \frac{dT}{\zeta}}{n\epsilon})=O(\frac{\eta^2s^2TM\log^2 \frac{dT}{\zeta}\log\frac{1}{\delta}}{\gamma^2 n\epsilon}). 
\end{equation}
In the following we will assume the above event holds. We note that by our model for any $w$
\begin{equation*}
 \gamma \|w-w^*\|_2^2\geq   L_\mathcal{D}(w) - L_\mathcal{D}(w^*)\geq \mu \|w-w^*\|_2^2.
\end{equation*}
Thus we have
\begin{align}
   & \LD- L_\mathcal{D}(w^*)\leq(1-\frac{5}{72}\frac{\mu}{\gamma})( \LT- L_\mathcal{D}(w^*))+O(\frac{\eta^2s^2TM\log^2 \frac{dT}{\zeta}\log\frac{1}{\delta}}{\gamma^2 n\epsilon})  \label{aeq:73}
\end{align}
In the following we will show that $w^{t+1}=\wD$ for all $t$. We will use induction, assume $w^{i+1}=w^{i+0.75}$ holds for all $i\in [t-1]$, we will show that it will also true for $t$. Use (\ref{aeq:73}) for $i\in [t-1]$ we have 
\begin{align}
   &\mu \|\wD-w^*\|_2^2\leq  \LD- L_\mathcal{D}(w^*)\nm \\ &\leq(1-\frac{5}{72}\frac{\mu}{\gamma})^t( L_\mathcal{D}(w^1)- L_\mathcal{D}(w^*))+
   O(\frac{\gamma}{\mu}\frac{\eta^2s^2TM\log^2 \frac{dT}{\zeta}\log\frac{1}{\delta}}{\gamma^2 n\epsilon}) \nm \\
   &\leq \gamma(1-\frac{5}{72}\frac{\mu}{\gamma})^t \|w^1-w^*\|_2^2
   +O(\frac{\gamma}{\mu}\frac{\eta^2s^2TM\log^2 \frac{dT}{\zeta}\log\frac{1}{\delta}}{\gamma^2 n\epsilon}). \nm
\end{align}
When $\|w^1-w^*\|_2^2\leq \frac{1}{2}\frac{\mu}{\gamma}$, and $n$ is large enough such that 
\begin{align*}
    &n\geq O(\frac{1}{\mu^3} \frac{s^2TM\log^2 \frac{dT}{\zeta}\log\frac{1}{\delta}}{\gamma^2 \epsilon}) 
\end{align*}
Then $\|\wD\|_2\leq \|w^*\|_2+\frac{1}{2}\leq 1$. Thus $w^{t+1}=\wD$. So we have 
\begin{align}
   & L_\mathcal{D}(w^{T+1})- L_\mathcal{D}(w^*)\leq(1-\frac{5}{72}\frac{\mu}{\gamma})^T( L_\mathcal{D}(w^1)- L_\mathcal{D}(w^*))+O(\frac{s^2TM\log^2 \frac{dT}{\zeta}\log\frac{1}{\delta}}{\mu^2 \gamma n\epsilon})  \nm
\end{align}
Thus, take $T=\tilde{O}(\frac{\gamma}{\mu}\log n)$ and  $s= O( (\frac{\gamma}{\mu})^2 s^*)$ we have the result. 
\end{proof}

\end{document}